\documentclass[11pt]{article}

\usepackage[utf8]{inputenc}
\usepackage{comment}
\usepackage{microtype}
\usepackage[
]{graphicx}
\usepackage{subfigure}
\usepackage{booktabs} % for professional tables
\usepackage{scalerel}
\usepackage[margin=1in]{geometry}
\usepackage[round]{natbib}

\usepackage{amssymb}
\usepackage{amsthm} 
\usepackage{amsmath}
\usepackage{mathtools}

\usepackage{nicefrac}

\usepackage{algorithm}
\usepackage{algpseudocode}
\usepackage[shortlabels]{enumitem}

\usepackage{url}
\usepackage{float}
\allowdisplaybreaks
\usepackage[format=plain]{caption}

\theoremstyle{plain}
\newtheorem{theorem}{Theorem}[section]

\newtheorem{lemma}[theorem]{Lemma}
\newtheorem{corollary}[theorem]{Corollary}
\theoremstyle{definition}
\newtheorem{definition}[theorem]{Definition}
\newtheorem{assumption}[theorem]{Assumption}
\theoremstyle{remark}

\usepackage{nicefrac}
\usepackage{thmtools}
\usepackage{enumitem}
\usepackage{thm-restate}
\usepackage{verbatim, indentfirst}
\usepackage{bbm, xcolor}
\usepackage[colorlinks=true,linkcolor=blue]{hyperref}%

\hypersetup{citecolor=blue}

\newtheorem{claim}{Claim}[section]
\usepackage[title]{appendix}
\usepackage{amsfonts}       % blackboard math symbols

\usepackage{microtype}      % microtypography
\usepackage{subfiles}
\usepackage{setspace}
\DeclareMathOperator{\E}{\mathbb{E}}

\DeclareMathOperator{\uH}{\mathcal{H}}
\DeclareMathOperator{\U}{\mathcal{U}}
\DeclareMathOperator{\I}{\mathcal{I}}
\DeclareMathOperator{\F}{\mathcal{F}}
\DeclareMathOperator{\G}{\mathcal{G}}
\DeclareMathOperator{\uL}{\mathcal{L}}
\DeclareMathOperator{\J}{\mathcal{J}}

\DeclareMathOperator{\T}{\mathcal{T}}
\DeclareMathOperator{\A}{\mathcal{A}}
\DeclareMathOperator{\lc}{\textbf{c}}
\DeclareMathOperator{\ls}{\textbf{s}}

\let\Pr\relax\DeclareMathOperator{\Pr}{\mathbb{P}}

\DeclareMathOperator*{\argmin}{arg\,min}

\title{InfoNCE Loss Provably Learns Cluster-Preserving Representations}

\author{Advait Parulekar\thanks{Department of Electrical and Computer Engineering, 
The University of Texas at Austin, Austin, TX,  USA. \qquad\qquad\{advaitp@utexas.edu, liamc@utexas.edu, mokhtari@austin.utexas.edu, sanjay.shakkottai@utexas.edu\}.},\quad Liam Collins$^*$, \quad Karthikeyan Shanmugam\thanks{Google Research India, karthikeyanvs@google.com},\\  Aryan Mokhtari$^*$, \quad Sanjay Shakkottai$^*$}

\date{}

\begin{document}

\maketitle

\begin{abstract}
The goal of contrasting learning is to learn a representation that preserves underlying clusters by keeping  samples with similar content, e.g. the ``dogness'' of a dog, close to each other in the space generated by the representation. A common and successful approach for tackling this unsupervised learning problem is minimizing the InfoNCE loss associated with the training samples, where each sample is associated with their augmentations (positive samples such as rotation, crop) and a batch of negative samples (unrelated samples). To the best of our knowledge, it was unanswered if the representation learned by minimizing the InfoNCE loss preserves the underlying data clusters, as it only promotes learning a representation that is faithful to augmentations, i.e., an image and its augmentations have the same representation. Our main result is to show that the representation learned by InfoNCE with a finite number of negative samples is also consistent with respect to {\em clusters} in the data, under the condition that the augmentation sets within clusters may be non-overlapping but are close and intertwined,  relative to the complexity of the learning function class.
\end{abstract}

% % \begin{keywords}%
% %   Contrastive learning, Representation learning, Self-supervised learning
% % \end{keywords}
% % \newpage

\section{Introduction}

Representations pretrained on partially or completely unlabeled data are becoming ubiquitous in machine learning applications \citep{DBLP:journals/corr/abs-1802-05365,radford2021learning}, in large part due to the availability of large unlabeled datasets and significant computing power offline, and the effectiveness of self-supervised representation learning algorithms, especially contrastive learning (CL). CL aims to learn representations that treat natural images similarly to their augmentations, while maximizing the average distance between random pairs of images. In recent years CL has demonstrated numerous successes in pretraining representations with unlabeled data that learn meaningful relationships between data points that generalize well to  downstream tasks in computer vision \citep{hjelm2018learning,oord2018representation,bachman2019learning,caron2020unsupervised,chen2020simple,chen2020big,he2020momentum,henaff2020data,li2020prototypical,misra2020self,tian2020contrastive,tian2020makes} and natural language processing \citep{brown2020language,gao2021simcse,su2021tacl,radford2019language}. 

Despite its empirical success, it is not well-understood how CL learns meaningful relationships between data points. Since data are unlabeled, the only immediate structure in datasets leveraged by CL are the sets of images and their augmentations. Without further assumptions, this structure is insufficient to learn relationships between images across augmentation sets.
To circumvent this issue there are two approaches. The first is to assume that augmentation sets of semantically similar natural images overlap, i.e. for two images of cats, some of the augmentations of each image are equivalent \citep{arora2019theoretical,haochen2021provable,haochen2022beyond,shen2022connect,wang2022chaos}. However, this assumption is unlikely to hold in practice, as pointed out by recent work \citep{saunshi2022understanding}. The second approach is to consider inductive biases of the representation function class and/or optimization algorithm, and use these to argue that only certain types of representations (that capture semantic relationships) can be learned.

Prior studies have initiated the study of how inductive biases of the representation class can lead to meaningful representations in CL \citep{saunshi2022understanding,haochen2022theoretical}, but their analysis is for the spectral contrastive loss, which is not used in practice. Instead, the vast majority of CL approaches, including the widely popular SimCLR framework \citep{chen2020simple}, optimize a loss function based on the InfoNCE loss \citep{gutmann2010noise,oord2018representation}. A variety of works have studied properties of the InfoNCE loss, but due to its unwieldy log-sum structure have made restrictive assumptions, such as having infinite  \citep{wang2020understanding,robinson2020contrastive,von2021self} or only a single \citep{tosh2021contrastive,huang2021towards} negative sample(s). 

%\subsection{Main Contributions}
%\label{sec:contrib}
\vspace{0.05in}
\noindent \textbf{Main Contributions.} {\color{black} We are given a collection of clusters of natural images, with each image associated with  augmentations (positive samples such as `rotation') and a finite set of negative samples (unrelated images). Using the InfoNCE loss, our goal is to learn a $d-$dimensional representation $g \in \G$, where $g = (f_1, f_2, \ldots f_d)$ and $\{f_i\}$ are binary functions mapping images to $\{-1, 1\}$ (thus $g$ maps images on the {\color{black}  hypercube} $\mathcal{H}_d = \{-1, 1\}^d$).
% , such that when this is composed with a head, can solve all downstream binary classification tasks of interest.   
%
Our setting is one where the function class has bounded expressivity with respect to the augmentation sets, meaning that the augmentation sets within clusters are intertwined, and {hard} to separate from the rest of the cluster using functions in $\mathcal{F}$. }

% is a property both of the (Boolean) function class and the augmentations. It roughly says that 
% if a function splits two natural images in a cluster, then it also splits some image in that cluster from atleast one of its augmentations.

%% (map from images to a vertex on the Rademacher hypercube $\mathcal{H}_d = \{-1, 1\}^d$)

% that inherit this inductive bias, namely, that the function class clusters the images consistent with these downstream tasks. Formally, a representation is a mapping from images to a vertex on the Rademacher hypercube $\mathcal{H}_d = \{-1, 1\}^d$.

\vspace{0.05in}
\noindent \textbf{(Realizable Setting)} Suppose there exists a representation $g^* \in G$ that is: {\em (a)} {cluster preserving, and {\em (b)} different clusters of images are uniformly mapped over distinct vertices on the hypercube
% and furthermore the induced distribution is uniform 
(qualitatively, class-balance in the image dataset). We show that with {\em any finite number of negative samples}, the representation learned by the InfoNCE loss is cluster-preserving and uniform. Furthermore, this learned representation when composed with a two-layer ReLU head, achieves zero downstream error on any cluster-preserving binary classification task. Our proof hinges on a novel Markov Chain construction showing that the InfoNCE loss of any non-uniform representation can be improved by ``blurring'' the representation through the Markov Chain transitions. Conversely, we show that solutions to the InfoNCE loss optimized over an arbitrarily powerful representation class $\mathcal{G}_\star$  cannot have meaningful downstream performance guarantees on such tasks.

%\vspace{0.05in}
%\noindent \textbf{(Agnostic Setting)} In the non-realizable setting with arbitrary numbers and sizes of clusters, we prove  that representations in $\mathcal{G}$ that are close to uniform under a looser notion of closeness must also be clean if they are minimizers of the InfoNCE  loss on $\mathcal{G}$.  Our proof exploits the geometry induced by the inverse maps from the hypercube to the space of images by the pair of representations -- the non-clean map, and a perturbed (potentially cleaner) map. These together induce a partition over the space of images, leading to multiple cases. In each of these cases, we show that the improvement in the InfoNCE loss due to positive samples outweighs the effects of negative samples while moving from the non-clean to the cleaner representation.

%Our proof leverages a novel partitioning of the image space  that may be of independent interest for future analysis of the InfoNCE loss.

\vspace{0.05in}
\noindent \textbf{(Agnostic Setting)} In the agnostic (non-realizable) case, through sensitivity analysis, we show that for any close-to-uniform and non-cluster-preserving representation, there exists a  representation that preserves one additional cluster and thus improves the InfoNCE loss. Our proof uses a novel partitioning of the image space that is of independent interest for future analysis of the InfoNCE loss.

\subsection{Related Work}

% \textbf{Understanding contrastive learning.} 
% Contrastive learning encompasses a large class of self-supervised methods that aim to learn expressive representations by maximizing similarity between different views of the same data point and minimizing the average similarity between all pairs of points \cite{wei2021pretrained}. The vast majority of these works optimize InfoNCE-based losses \cite{gutmann2010noise,oord2018representation}, including the widely popular SimCLR framework \cite{chen2020simple}.
% , which differs from the InfoNCE only in its sampling method for negative samples, and in its lack of inclusion of the positive similarity in the uniformity loss.
Several works have aimed to explain the success of contrastive learning in recent years. 
\cite{wang2021understanding} and \cite{wang2020understanding} showed empirically that CL encourages aligned and uniform representations, and improving alignment and uniformity improves downstream performance. The work by \citet{chen2021intriguing} generalizes the InfoNCE loss to a larger family of losses with alignment and uniformity terms weighted according to a hyperparameter. %and shows downstream performancecan be improved by upweighting the positive loss representations and \cite{chen2021intriguing}. % and theoretical approaches. %, our understanding
%grasp of its behavior 
Early theoretical studies attributed the success of CL to its proclivity to maximize the mutual information between augmentations of the same image \citep{bachman2019learning}, but later work cast doubt on this viewpoint by showing that optimizing a tighter bound on the mutual information leads to worse performance \citep{mcallester2020formal,tschannen2019mutual}.
% Many works have studied CL from a theoretical perspective in recent years \cite{tosh2021contrastive,tosh2021contrastivee,wang2020understanding,zimmermann2021contrastive,wen2021toward,chen2021intriguing,haochen2021provable,saunshi2022understanding}, among which a popular approach has been 
% \cite{arora2019theoretical} tried to explain the success of CL by its ability to recover latent variables, but their analysis suggests an upper bound on the number of negative samples which does not hold in practice \cite{nozawa2021understanding}. 
% \citet{tosh2021contrastive} and \citet{tosh2021contrastivee} theorized that conditional independence and redundancy can explain the success of CL with one negative sample per batch.
\citet{wang2020understanding} further {showed} that solutions to the InfoNCE loss are aligned and uniform in the limit of infinite negative samples per batch.

A variety of works have studied CL's ability to recover meaningful clusters or latent variables in the data \citep{arora2019theoretical,tosh2021contrastive,zimmermann2021contrastive,ash2021investigating,nozawa2021understanding,haochen2021provable,shen2022connect,haochen2022beyond,haochen2022theoretical,wang2022chaos,awasthi2022more,bao2022surrogate}. However, the majority of these works consider arbitrary function classes, which requires  strong assumptions on the connectedness of augmentation sets within each cluster, such as assuming positive pairs are conditionally independent given their cluster identity, in order to give downstream guarantees \citep{saunshi2022understanding}.
The work by \citet{haochen2022theoretical} is the most related work to ours, as they study function classes that induce a similar bias towards preserving clusters as ours without any assumption on the connectedness of augmentation sets. 
However, their study is focused on 
minimizing a spectral contrastive loss which serves as a surrogate for the more practically used InfoNCE loss. While studying spectral contrastive loss is enlightening and provides some intuition, it cannot be extended to the InfoNCE loss because of two major reasons: % {\color{green} First, the }
{\color{black} First, the loss function fails to highlight the role of finite batches of negative samples, which is a well-studied and key component of the InfoNCE loss \citep{awasthi2022more,bao2022surrogate,ash2021investigating,nozawa2021understanding}.  Second, their analysis does not translate to our setting because the key difficulty in our proof is to show that negative samples promote uniformity; this aspect directly follows with the spectral loss due to the covariance regularizer.}

Additional theoretical works have studied the feature learning process of CL with (stochastic) gradient descent on linear \citep{tian2022deep,ji2021power} and two-layer ReLU neural networks \citep{wen2021toward,tian2022understanding}, properties augmentations must satisfy in order for CL to be successful \citep{tian2020makes}, the role of the projection head in CL \citep{wen2022mechanism,gupta2022understanding}, and the behavior of contrastive losses in (semi-)supervised settings \citep{khosla2020supervised,zheng2021weakly,chen2022perfectly}. Several other works analyze non-contrastive self-supervised learning methods \citep{wei2020theoretical,balestriero2022contrastive,garrido2022duality,lee2021predicting}.

\section{Problem Formulation} \label{section:formulation}

Our learning task consists of {\em (i)} a pretraining phase -- wherein we are not provided supervised labels but rather only \textit{associations} between images and {\em (ii)} a supervised learning phase in which we are provided (a few) labeled data points, labeled according to some specific downstream task. 
During the pre-training phase, we do not know what the downstream task is. However, we are provided \textit{augmentations} of the raw data points that the learner knows should be classified the same way as the raw data for \textit{any} downstream task. In a sense, the augmentations can be seen as modifying the data in a way that leaves the information contained in the data invariant with respect to the downstream tasks.
% {\color{blue} Modulo sounds a bit colloquial}. 
Ideally, we aim to learn a representation that is invariant to such augmentations so that downstream learning can be statistically efficient. 
For interpretability, we will work in the setting of ``images".

\textbf{Images and augmentations.} 
%To formalize the notion 
The images consist of features that are either important for classification or which function only as irrelevant details. Inspired by \citep{von2021self}, we consider an image generation model that consists of {\em (i)} content variables denoted by $\lc$ which capture innate qualities of the images (e.g., the `catness' of a cat), and {\em (ii)} style variables denoted by $\ls$ which capture the appearance of the image (e.g., `rotation' and `crop' for creating augmentations to an image; `long tail' and `furry' for different natural images of dogs). More precisely, each image $x$ is generated according to $x= \I(\lc,\ls)$, where $\I(.,.)$ is a mapping from the space of content and style variables to the space of images. We assume that the natural images are generated such that their content variables $\lc$ belong to the set  ${C}$ and their style variables $\ls$ belong to the set ${S}_\circ$.

We further consider that there is a set of augmentations $\Lambda$, which is a set of functions mapping natural images to augmented images. 
An augmented image of an image $x$ is denoted $\mathcal{A}(x)$, where $\mathcal{A}\in \Lambda$. We assume that the augmented image preserves the content of the original image, while its style may differ from the original image. More precisely, if the original image is given by $x= \I(\lc,\ls)$, then its augmented image $\mathcal{A}(x)$ satisfies the following property:  $\A(x)= \A(\I(\lc,\ls))= \I(\lc,\ls^+)$ for some  $\ls^+\in {S}$, where the set ${S}$ contains ${S}_\circ$. So the augmented images have possibly different style variables \textit{but the same content variables} as the natural images.
Further, the set of augmented images of the image $x = \I(\lc,\ls)$ is called  its \textit{augmentation set} and is defined as $A(x) := A(\I(\lc, \ls)):=\{ \A(\I(\lc, \ls))| \A\in \Lambda\}$, with all images having equal-sized augmentation sets for simplicity. We typically refer to an image $\I(\lc, \ls)$ as $x$ and its augmentation $\I(\lc, \ls^+)\sim A(x)$ as $x^+$, where, for all sets of images $B$, $\sim B$ denotes a random sample drawn uniformly from the set $B$. 
% {Considering the structure of the augmented images, their content variables $\lc$ belong to the set  $\mathcal{C}$, the same as the original images, while their style variables $\ls$ belong to the set $\mathcal{S}$ which contains the set $\mathcal{S}_\circ$, i.e.,  $\mathcal{S}_\circ\subseteq \mathcal{S}$.}
We let ${D}$ denote the set of all images and their augmentations and  ${D}_\circ\subset{D}$ denote the set of all natural images. 

\textbf{Clusters.} A collection of images (natural and augmented) forms a cluster if their content variables are the same; thus, $x = \I(\lc, \ls)$ and $x' =  \I(\lc, \ls')$ belong to the same cluster. We denote the cluster of images with content variable $\lc$ by $\Gamma_{\lc}$, and the natural images within cluster $\Gamma_{\mathbf{c}}$ by $\Gamma_{\mathbf{c},\circ}\coloneqq \Gamma_{\mathbf{c}} \cap D_\circ$. As an example, suppose that the content $\lc$ captures the `dogness' of an image. Then, different images of dogs would have the same content, but have different style variables (e.g., furry, skinny, long ears). Recall that the augmentations of an image also share the same content, but the style might be chosen from a different set (e.g., rotation, cropping, blur). {\color{black} In other words, a single cluster is a union of many augmentation sets since not all style variations within a cluster are captured by augmentations.}

{\color{black} \textbf{Representations and heads.} We consider a function class $\F$ of binary functions, $f(x') \in \{-1, +1\}$, where $x'$ is either an image $x$ or its augmentation $x^+$. This is a function class with bounded expressivity (e.g., a class of functions that can be expressed as the thresholded output of a neuron from a neural network with bounded width and depth). We search over $d$-dimensional representations, denoted by $\G$, such that each coordinate of the representation is an element of $\F$, i.e., $g = (f_1, f_2, \ldots, f_d)$,.  Thus a representation $g \in \G := \F^d$ is simply a concatenation of $d$ binary classifiers, mapping an image $x$ to the vertex of the Rademacher hypercube\footnote{Representations in CL often map to the unit hypersphere \citep{wang2020understanding}. Here, we consider a discretized  version of this output space for two reasons: (1) it allows us to construct a naturally restricted representation function class by extending natural properties of binary classifiers, and (2) it provides a tractable setting for us to show the first results that InfoNCE prefers cluster-preserving and uniform representations with finite samples, as it is still an open problem to determine uniform arrangements of finite points on the unit hypersphere \citep{thomson1904xxiv}.
% (1) it allows for a clear notion of uniformity with finite samples, and (2) provides a tractable setting for us to show the first results that InfoNCE prefers uniform and clean representations with finite negative samples.
}
% Formally, $\G$ is taken to be the $d$-ary Cartesian power of $\F$, 
% i.e., $\G = \F^d.$ 
%Thus, a specific representation $g \in \G$ is of the form  $g = (f_1, f_2, \ldots, f_d)$, which takes input $x'$ as above and whose output is a vertex on the $d-$hypercube
$\mathcal{H}_d = \{-1, 1\}^d$.
% (we misuse the terms `Boolean' and `bit' to refer to the binary set $\{-1, +1\}$). 
Note that each $g \in \G$ denotes only the representation (e.g., the body of a neural network). For downstream tasks, a full  classifier is formed by composing $g$ with a {\em head} $\omega\in \J$ for some class $\J$ of heads (e.g., the final classification layer of a neural network).}

\textbf{Goal of pretraining.}
Ultimately, we aim to find a representation that allows for easily solving  tasks from a set of possible downstream binary classification tasks $h \in \mathcal{T}$, where each task $h$ maps an image to a binary label $\{-1, 1\}$.
These tasks are assumed to be faithful to the clusters, meaning that for any pair of images $x, x'$ belonging to the same cluster, $h(x) = h(x')$.

% $\F_c$ as possible. 
Note that during pretraining, the learner does not have any knowledge about which task will be assigned among the solvable ones. After pretraining, the learner fixes the representation but can learn a task-specific head when it encounters a downstream task.
We define the error a representation $g$ on the downstream task $h \in \mathcal{T}$ with respect to the class $\mathcal{J} \subseteq \{\omega: \mathbb{R}^d \rightarrow \mathbb{R}\}$ of allowed heads as
% The error of  $g$ on is 
% with $n$ samples allowed for learning the head is:
   \begin{align}
    \uL_{h,\mathcal{J}}(g) \coloneqq \inf_{\omega \sim \mathcal{J}}\mathbb{P}_{x \sim {\mathcal{D}}}[ \omega \circ g(x) \neq h(x) ]. %_{S,y} 
\end{align}
% where 
% % \begin{align}
% %     \omega_{S,y} \coloneqq \min_{\omega \in \J} \frac{1}{n}\sum_{x \in S} 1{\{\omega \circ g(x) \neq y(x)\}},
% % \end{align}
% i.e. $\omega_{S,y}$ minimizes the empirical loss of $\omega \circ g$ on $n$ labeled samples from the downstream task.
The error of $g$ on a family of downstream tasks $\mathcal{T}\subseteq \mathcal{F}_c$ is the worst case  error among tasks in $\mathcal{T}$:
\begin{align}
    \uL_{\T,\J}(g) \coloneqq \sup_{h \in \mathcal{T}} \uL_{h,\mathcal{J}}(g).
\end{align}
% Our main assumption describes a specific relationship between $\T$ and $\F$ and formalizes the notion of an inductive bias in $\F$ with respect to $\T$: For any function $f\in \F$ satisfying $f(\I(\lc, \ls)) \ne f(\I(\lc', \ls'))$ for $\lc\sim \lc'$, there exists $\lc^{(1)}, \ls^{(1)}$ and $\ls^{(2)}$ satisfying $f(\I(\lc^{(1)}, \ls^{(1)})) \ne f(\I(\lc^{(1)}, \ls^{(2)}))$. In other words, any function in the supervised function class $f\in \F$ that splits a cluster must also split an augmentation set within the cluster.
To summarize, for a task that is realizable \textit{with supervision} using function class $\F$, we would like to learn a representation entirely from unlabelled data such that the task on the embedded images is still realizable for $\J$.
The overall motivation is that learning $\omega\in \J$ can generally require fewer \textit{labeled} samples than learning the joint model $\omega\circ g$.

\subsection{InfoNCE loss}

We denote $\mathbb{E}_{x,x^+} \coloneqq \mathbb{E}_{x \sim {D}_{\circ},x^+\sim A(x)}$ and $\mathbb{E}_{x,x^+, \{x^-_i\}_{\ell}} \coloneqq \mathbb{E}_{x \sim {{D}_{\circ}},x^+\sim A(x), \{x^-_i\}_{\ell} \sim {D}_{\circ}^\ell}$ for simplicity.  The InfoNCE loss we consider is given by\footnote{{
   For ease of exposition
we consider the case wherein negative samples are drawn from the set of {natural images, as in \citep{wen2021toward}.} Although this may not hold in practice, it greatly simplifies the presentation of our results.}}
\begin{align}\label{eq:InfoNCE} %_{\text{SimCLR}}
    &\uL({g}) = -\underbrace{\beta \mathbb{E}_{x,x^+} [g(x)^\top g(x^+) ]}_{\text{alignment}} \;  + \; \underbrace{\mathbb{E}_{x,x^+, \{x^-_i\}_{\ell}}\left[ \log\bigg({ e^{\beta g(x)^\top g(x^+)}\!+\!\sum_{i=1}^\ell e^{\beta g(x)^\top g(x^-_{i})}  } \bigg)\right]}_{\text{uniformity}}
\end{align}
Following \citet{wang2020understanding}, we refer to the first term as the \textit{alignment} term, or the \textit{positive} term, and we refer to the second term as the \textit{uniformity} term or the \textit{negative} term.
By minimizing the first term, we are maximizing the alignment between the representation of an image and its augmentation, and by minimizing the second term we are enforcing the representation of different images to be as different as possible.

{\color{black}{The above formulation suggests that the representation learned by minimizing the above loss forces images and their augmentations to have a similar representation. What we show in the following sections is a stronger result which guarantees by minimizing the InfoNCE loss, all images that belong to the same cluster (share the same content) will have a similar representation. 
}}
\section{Bounded Function Class}

The goal of contrastive learning is to learn a representation from unlabeled samples that is useful for downstream tasks. Recall that the representations we consider map images to vertices on the Rademacher hypercube $\mathcal{H}_d$. A ``good'' representation should map images from the same cluster to the same vertex, and images from distinct clusters to distinct vertices.

% We consider a class of representations consisting of concatenated binary classifiers, with the classification boundaries of two example classifiers shown above ($f_1$ and $f_2$). The classifiers have limited complexity, which respects an clustering of images and their augmentations. Specifically, clusters are sets of images such that are not split by any ``clean'' classifier that does not pass through any augmentation set,  e.g. $f_1$ in the figure. Thus, if a classifier intersects a cluster, then it must also intersect at least one augmentation set, e.g. as the non-clean classifier $f_2$ intersects the blue and green augmentation sets in the Dogs cluster. Our motivating assumption is that the underlying clusters are related to downstream tasks.

Intuitively, this seems possible if images having the same content (i.e., from the same cluster) along with their set of augmentations are ``close and intertwined'' (see Figure~\ref{fig:spiky}), such that any function $f \in \F$ cannot split the cluster without also splitting an image from its augmentation.
Note that we do {\em not} need connected clusters with overlapping augmentations (meaning two images have the same augmentation, which is an unrealistic assumption); merely that the cluster has a complex geometry relative to the function class. 
%We formalize this through `cleanness' of representations.

%% The representation $g \in G$ should 

\subsection{Complexity of $\mathcal{F}$ Relative to Augmentations}
\label{sec:cluster-bias}
\begin{figure}
  \begin{minipage}[c]{0.45\textwidth}
    \vspace{-4mm}
    \includegraphics[height=2.2in,width=2.5in]{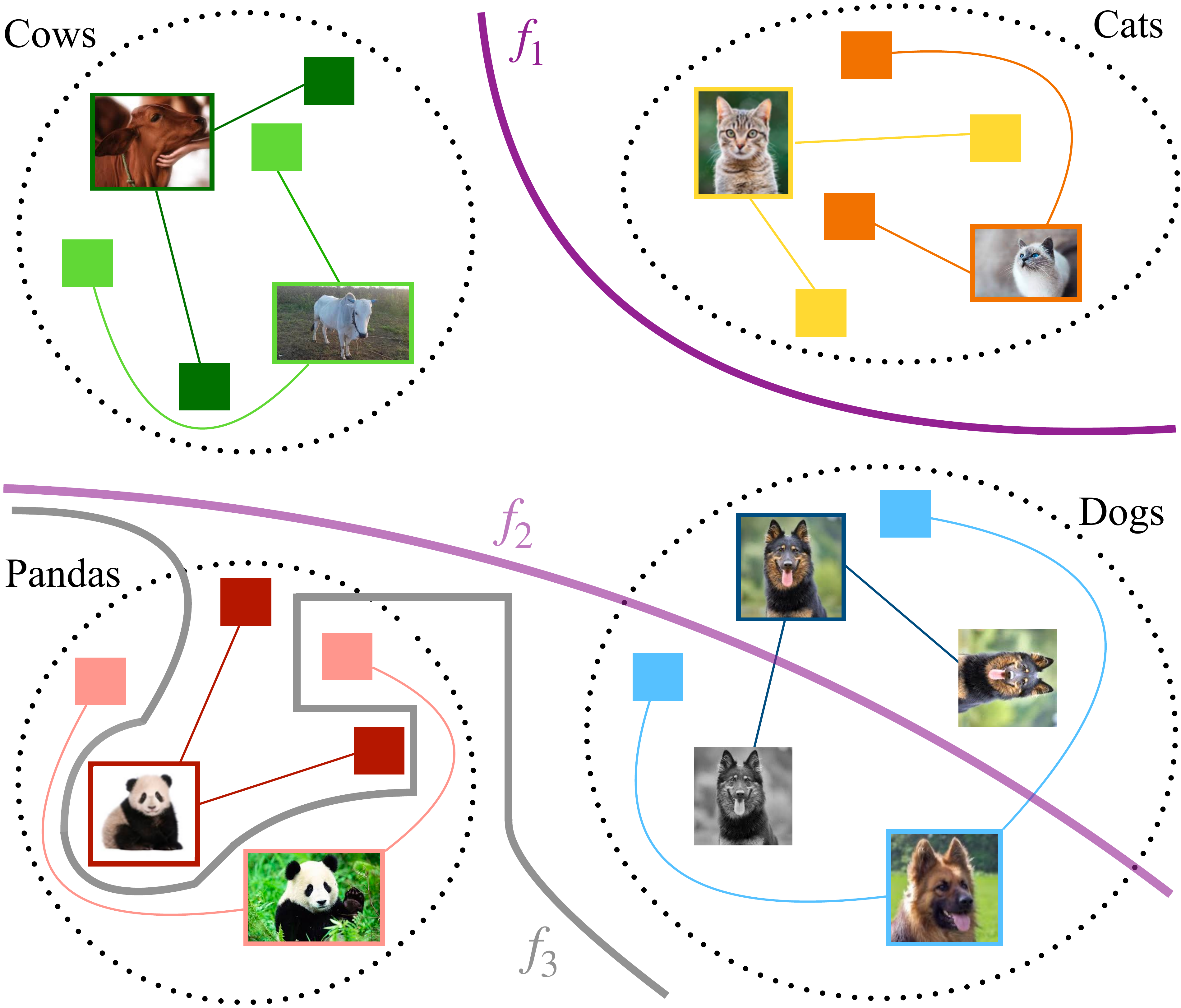}
  \end{minipage}\hfill %\hspace{-0.5in}
  \begin{minipage}[c]{0.55\textwidth}
\caption{We illustrate a setting with four clusters, and two augmentation sets within each cluster indicated by linked rectangles in distinct colors. The clean function $f_1$ does not split any augmentation or cluster (meaning, it maps all images from the same augmentation set alike, and likewise for clusters). $f_2$ splits the cluster of dogs, and in accordance with Assumption~\ref{assump:spikyaug}, also splits augmentation sets within that cluster (and is therefore non-clean).
% augmentation of one of the natural dog images, and also split the cluster of dogs (two natural dog images are split). 
Finally, $f_3$ violates Assumption~\ref{assump:spikyaug}, because it does not split any augmentation of pandas, yet it splits the pandas cluster.
Images are from the Animals V2 dataset: \cite{deepnets_2022}.\label{fig:spiky}}
\end{minipage}
\end{figure}
%\end{comment}

We formalize the notion of bounded expressivity of $\F$  relative to the geometry of clusters.
% through an assumption on the spikyness of ``clusters''.
% in the form of clusters.
% We now show that the classifier class $\mathcal{F}$ induces a clustering of images from possibly different augmentation sets, as discussed in the Introduction {\color{blue} [KS: This sentence does not read right. We want to say about an inductive bias assumption not a result we establish]}. 
We use this assumption to show in Section \ref{sec:realizable} that solutions to the InfoNCE loss optimized over $\mathcal{G}$ satisfy useful uniformity and alignment properties that lead to downstream performance guarantees on tasks that adhere to the clusters. 
%
% \textbf{Inductive bias in the form of clusters.} 
% We show that the representation class $\G$ inherits a natural inductive bias of the classifier class $\mathcal{F}$ in our setting with augmentations rather than labels. 
% The inductive bias of $\mathcal{F}$ consists of the data clusters that can result from classifiers in $\mathcal{F}$ that are ``clean'' with respect to augmentation sets.
% as illustrated in Figure \ref{fig:clean}. 
Formally, the function class $\F$ and the augmentations $\Lambda$ define a set of \textit{clean} functions $\F_c \subseteq \mathcal{F}$ that separate the data in a way that respects the augmentations. % 

\begin{definition}[Clean Function]
% \begin{align} 
%% f\in \F_c\iff f(\I(\lc, \ls)) = f(\I(\lc, \ls'))~\forall  %\ls, \ls'\in S. 
$f\in \F_c \ \ {\rm{is \ clean}} \ \iff f(x) = f(\mathcal{A}(x)))~\forall x \in  D_\circ,  \mathcal{A} \in \Lambda$. % \ls, \ls'\in S. 
\label{eq:defclean}
% \end{align}
\end{definition}
In other words, the binary function $f$ is clean if it does not separate any image and its augmentations from each other. 
% {\color{red} Note that $\mathcal{F}$ induces a bias of contrastive learning solutions  towards the set of clean classifiers $\mathcal{F}_c$, since contrastive learning aims to preserve alignment of augmentation sets.} 
Our main assumption is that if a classifier in $\mathcal{F}$ splits a cluster, then it is not clean.
% the clean classifiers in $\mathcal{F}_c$ are consistent with the underlying clusters in the data due to intertwined augmentations.
\begin{assumption}[Intertwined Augmentations] \label{assump:spikyaug}
% $f \in \mathcal{F}_c \iff f(x)= f(x') \; \forall \; x,x'\in \Gamma_{\mathbf{c}}\; \forall \; \mathbf{c} \in C$.
For all $f \in \mathcal{F}$, if $f(x)\neq f(x')$ for some $x,x'\in \Gamma_{\mathbf{c}}$, then $f(x'')\neq f(\mathcal{A}(x''))$ for some $x'' \in \Gamma_{\mathbf{c}} \cap D_{\circ}$, where $\mathcal{A}(x'') \in {A}(x'')$.
\end{assumption}
Note that if a classifier does not split any cluster, then it must be clean, since augmentation sets are contained within clusters. Thus, 
Assumption \ref{assump:spikyaug} implies that
% states that if 
% $f\notin \mathcal{F}_c$ splits natural images in a cluster then it must also split an augmentation set for some natural image in  that cluster. 
% In other words,  
$f \in \mathcal{F}_c$ if and only if $f$ labels all images with the same content (belonging to the same cluster) alike, in other words it is \textit{cluster-preserving}. 
This assumption holds if the augmentation sets within clusters are close and intertwined (they cannot be easily split from the rest of the cluster), relative to the complexity of $\mathcal{F}$. Importantly, the augmentation sets need not overlap, meaning a single image need not be an augmentation to multiple natural images, consistent with practice \citep{saunshi2022understanding}.
As prior works have pointed out \citep{saunshi2022understanding,haochen2022theoretical}, Assumption \ref{assump:spikyaug} or variants on the bounded complexity of the function class are necessary for the success of CL in the realistic setting in which the augmentation sets do not overlap.

However, while some condition like Assumption \ref{assump:spikyaug} is necessary, it is not clear if this suffices to show that CL learns useful representations. Consider the example in Figure \ref{fig:spiky}. 
% There are images of cats, dogs, pandas and cows, and the augmentation sets are sufficiently spiky relative to the function class  $\mathcal{F}$ such that any classifier in $\mathcal{F}$, e.g. $f_2$, that splits a cluster must also split an augmentation set within that cluster. 
% Note that $\bar{f}$ does not belong to $\mathcal{F}$ according to Assumption \ref{assump:spikyaug}.  Still, 
It may be the case, for instance, that CL on $\mathcal{F}$ does not learn the cluster-preserving classifiers, as in addition to trying to maximize the similarity between  images and their augmentations, CL also tries to minimize the similarity between negative pairs of images.  Thus, it may choose a non-cluster-preserving classifier such as $f_2$ in an effort to minimize similarity of negative pairs. This would lead to poor downstream generalization on tasks involving classifying dogs, since $f_2$ separates images of dogs. It thus becomes critical to quantify the extent to which non-cluster-preserving classifiers must intersect augmentation sets such that CL will not learn them, as we do in Section \ref{section:agnostic}. Before this, we must show that even if CL learns a representation consisting of cluster-preserving classifiers, this representation generalizes well, which may not happen if it maps two or more clusters to the same vertex.
For instance, if CL simply learned $d$ copies of the cluster-preserving classifier $f_1$ in Figure \ref{fig:spiky}, this representation would not be able to distinguish cows from pandas from dogs on downstream tasks. We thus desire representations to be {\em  both} cluster-preserving and uniform  such that their mapping is a bijection from clusters to vertices. Next, we show that when a cluster-preserving and uniform representation is realizable, CL with the InfoNCE loss learns it, even with finite negative samples per batch.

\section{Results for the Realizable Setting}
\label{sec:realizable}

% \section{SimCLR Leads to Clean Representations}
Our first result shows that when the dataset ${D}$ and representation class $\mathcal{G}$ allow for mapping the data \textit{uniformly} on the hypercube in a \textit{cluster-preserving} manner, then the representation learned by minimizing the InfoNCE loss over $\mathcal{G}$ results in such a mapping. We first formally define the terms \textit{uniform} and \textit{cluster-preserving} below.
 %We show that pretraining using augmentations allows us to learn a representation such that each coordinate respects the augmentations.
% {\color{red} (that is, which is like the green classifier above)}. 
%We refer to such classifiers, and such representations, as \textit{clean}. 

\begin{definition}[Cluster-Preserving]
A {\em cluster-preserving} representation $g \in \mathcal{G}$ is one that for all  {$\mathbf{c} \in C$ } and all $x,x'\in \Gamma_{\mathbf{c}}$, $g(x)= g(x')$.
% A representation $g\in \mathcal{G}$ is {\em clean} iff $g(x) = g(x')$ for all $x\equiv x'$. 
% A classifier $f\in \F$ is {\em clean} iff $f(x) = f(x')$ for all $x\sim x'$. 
\end{definition}

% A \textit{uniform} representation is one in which the push-forward measure onto the hypercube is uniform.
\begin{definition}[Uniform]
A {\em uniform} representation $g\in \mathcal{G}$ satisfies $\Pr_{x\sim {D_\circ}}[g(x)=v]=2^{-d}$ 
for all $v\in \mathcal{H}_d$.
\end{definition}
Next, our results in this section assume a cluster-preserving and uniform representation exists in $\mathcal{G}$.
\begin{assumption}[Realizability]\label{assump:realizability}
    There exists a $g \in \mathcal{G}$ that is both cluster-preserving and uniform.
\end{assumption}
In order for there to exist a representation that is both cluster-preserving and uniform, there must be an integral multiple of $2^d$ clusters in the dataset and they must be balanced.
Before stating our main result, we must prove a key lemma that shows that among all ``clean'' representations, those that minimize the InfoNCE loss are uniform. We define $\mathcal{G}_c \subseteq \mathcal{G}$ as the set of clean representations in $\mathcal{G}$ consisting of $d$ concatenated clean classifiers from $\mathcal{F}_c$. 
% In practice, training on class-unbalanced datasets tends to reduce the performance of models optimized via InfoNCE \citep{assran2022hidden}. 

% {\color{red} We denote $m := m(\mathcal{F}_c, {D})$ as the number of distinct clusters in the dataset.}

% \begin{definition}[Clean]
% A {\em clean} representation $g \in \mathcal{G}_c$ is one that is a concatenation of $d$ clean functions from $\mathcal{F}_c$.
% % A representation $g\in \mathcal{G}$ is {\em clean} iff $g(x) = g(x')$ for all $x\equiv x'$. 
% % A classifier $f\in \F$ is {\em clean} iff $f(x) = f(x')$ for all $x\sim x'$. 
% The set of clean representations is denoted $\G_c$.
% \end{definition}

\begin{lemma}\label{thm:uniform}
     If Assumptions \ref{assump:spikyaug} and \ref{assump:realizability} hold, $\beta > c \log d$ for an absolute constant $c$, $d>3$, and $\ell \geq 1$, then $g^*\in \underset{g\in \mathcal{G}_c}{\arg\min}  \ {\mathcal{L}}(g) $ if and only if $g^*$ is uniform.
\end{lemma}

% {\color{red}
% % At first glance, Lemma \ref{thm:uniform} is not surprising. However, upon further inspection the statement is not as obvious as it may initially seem. 
% Perhaps surprisingly, subtle nuances of the InfoNCE loss are required to prove Lemma \ref{thm:uniform}.
% For example, we can show that when the xx term is removed from the log sum exp in the loss, and $d=3$, minimizer is not uniform... evidence for why this statement is not obvious
% % Consider for example the distribution that maps clusters to the embedded 
% }

\begin{proof}[Proof sketch]
% of Theorem \ref{thm:uniform}]
% {\color{gray}
% By regrouping the terms in \eqref{eq:InfoNCE}, the InfoNCE loss can be written as
% %%%
% \begin{equation}\label{sth}
% {{\mathcal{L}}}(g) =\underset{x,x^+,\{x_i^-\}_{\ell}}{\mathbb{E}} \bigg[  \log\bigg(1+ \sum_{i= 1}^\ell e^{\beta g(x)^\top \left(g(x_{i}^-) - g(x^+)\right)} \bigg)\bigg].
% \end{equation}
% }
Since the optimization problem is over representations composed of clean functions, we know that for all $g\in \mathcal{G}_c$, the term $g(x)^\top g(x^+)$ in the InfoNCE loss is exactly equal to $ d$. Hence, 
% the optimizing $\mathcal{L}(g)$ over $\mathcal{G}_c$ simplifies to minimizing
by regrouping the terms in \eqref{eq:InfoNCE}, the optimization problem simplifies to:
\begin{equation}\label{tilde_loss}
\min_{g\in\mathcal{G}_c}\mathcal{L}(g) = \min_{g\in\mathcal{G}_c}\bigg\{\hat{\mathcal{L}}(g)   \coloneqq \underset{x,x^+,\{x_i^-\}_{\ell}}{\mathbb{E}} \bigg[  \log\bigg(1+ \sum_{i= 1}^\ell e^{\beta g(x)^\top g(x_{i}^-) - \beta d } \bigg)\bigg]\bigg\}
\end{equation}
% over $\mathcal{G}_c$.
By Assumption~\ref{assump:realizability}, at least one uniform representation belongs to the set $\mathcal{G}_c$. We show that it minimizes the loss $\hat{\mathcal{L}}(g)$. To do so, 
we observe that we can think of minimizing $\hat{\mathcal{L}}(g)$ as an optimization with respect to distributions over the hypercube induced by $g$. To better understand this connection, consider the random variable $g(x)$ for $x \sim D_\circ$.
% , where ${D_I}$ is the underlying distribution over images.
Further, denote the corresponding induced distribution over $g(x)$ as $\mathcal{D}_g$, i.e., $\mathcal{D}_g$ is a distribution over the vertices of the hypercube ${\mathcal H}_d$. Letting $y = g(x)$, 
% (and correspondingly, $Y = g(X)$), 
the objective above can now be rewritten in terms of these distributions:
\begin{equation}\label{eq:uniformity}
\min_{\{\mathcal{D}_g: g\in \G_c\}} \bigg\{\tilde{\mathcal{L}}(\mathcal{D}_g) \coloneqq \underset{y,\{y_i^-\}_{\ell}\sim \mathcal{D}_g}{\mathbb{E}}\bigg[  \log\bigg(1+ \sum_{i= 1}^\ell e^{\beta y^\top y_{i}^- - \beta d} \bigg) \bigg]\bigg\}
\end{equation}
% thus our problem is equivalent to:
% $\mathcal{D}_g^* = \argmin_{\{\mathcal{D}_g: g\in \G_c\}} \hat{\mathcal{L}}(\mathcal{D}_g).$

Suppose the the minimizing distribution was not uniform over the hypercube, i.e. for $\mathcal{D}_g^*\in \argmin_{\{\mathcal{D}_g: g\in \G_c\}} \tilde{\mathcal{L}}(\mathcal{D}_g)$, $\mathcal{D}_g^*\ne \mathcal{U}$, where $\mathcal{U} $ is the uniform distribution over the hypercube $\mathcal{H}_d$. 
For any sample
% \footnote{We use capital letters $Y, \{Y_i\}$ to denote random samples, and small letters $y, \{y_i\}$ to denote specific instances.} 
$y, \{y_i\}\sim \mathcal{D}_g$, consider a random walk that starts from this sample and evolves over time. For this random walk, denote the variables at time $t$ by $y^t, \{y_i^t\}$ where $y^t$ (and similarly $y_i^t$ for all $i$), with $y^0 = y$ (correspondingly $y_i^0 = y_i$). The random walk evolves from $y^{t-1}$ to $y^t$ by flipping a uniformly random bit of $y^{t-1}$ with probability $\frac{1}{2}$, and with probability $\frac{1}{2}$, not changing anything; this construction is independent across all samples. We now observe that this construction induces an irreducible, aperiodic Markov chain with uniform stationary distribution over the hypercube. 

With this construction, the critical step in our proof is a surprising ``monotonicity'' property over time:  we show in Appendix \ref{appendix:realizable} that each transition over time decreases the function value as long as $\mathcal{D}_g$ is not uniform. Intuitively, {\em ``blurring'' the distribution $\mathcal{D}_g$ decreases the objective.} 

This result implies that $g$ is a minimizer of the loss ${{\hat{\mathcal{L}}}}(g)$ if and only if $g$ is a uniform representation. 
Consequently, we obtain that among all the representations in $\mathcal{G}_c$, the ones that are uniform minimize the loss in \eqref{eq:uniformity} and the statement of Lemma \ref{thm:uniform} follows. See Appendix~\ref{sec:lemma-prf} for details.
\end{proof}

Using Lemma \ref{thm:uniform}, we show our main result that all minimizers of the InfoNCE loss are uniform and cluster-preserving. 
% {\color{red} The proof also uses the fact that for all representations $g$ that are not cluster-preserving, ${\mathcal{L}}(g)>\hat{\mathcal{L}}(g)$ by Assumption \ref{assump:spikyaug}, where $\hat{\mathcal{L}}$ is defined in \eqref{tilde_loss}. --- MOVE to appendix?} 
To the best of our knowledge, this is the first result characterizing the minimizers of the InfoNCE loss with a finite batch of negative samples. The proof is provided in Appendix \ref{app:uniform-b}. 

\begin{theorem}\label{thm:uniformandclean}
    If Assumptions \ref{assump:spikyaug} and \ref{assump:realizability} hold, and we have $d>3$, $\ell \geq 1$, and $\beta > c \log d$ for an absolute constant $c$, then a representation $g^*\in \mathcal{G}$ is a global minimizer of the loss {${\mathcal{L}}(g) $} optimized over $\mathcal{G}$ if and only if it is uniform and cluster-preserving.
    % , meaning for any $\mathbf{c}\in C$ and any $x,x'\in \Gamma_{\mathbf{c}}\cap D_{\circ}$, $g^*(x)=g^*(x')$. 
    % Additionally, any clean and uniform representation $g$ is a global minimizer of the loss $\mathcal{L}(g)$.
\end{theorem}
\subsection{Downstream Guarantees}
% {\color{red} move after section 4.2? the positive downstream guarantee is only for the realizable setting}

% We first define the space of downstream tasks consistent with our inductive bias assumptions on $\mathcal{F}$. Our main assumption is that downstream tasks label all images from the same cluster alike, hence the clusters are meaningful for downstream performance.
% \begin{assumption}[Meaningful clusters] \label{assump:downstream}
% The set of downstream tasks $\mathcal{T}$ is a subset of $\mathcal{F}_c$.
% \end{assumption}

% Assumption~\ref{assump:downstream} ensures that the inductive bias of the function class is aligned with downstream tasks, which is necessary for any downstream guarantees. Importantly, recall that this bias is inherited by the  representation class $\G$ because each representation in $\G$ consists of concatenated supervised functions.
%
Next we translate the aforementioned representation learning results for $\mathcal{G}$ into downstream performance guarantees. We consider the class of heads consisting of single-layer ReLU neural networks with $m$ neurons.  Formally, $\mathcal{J}_{\text{ReLU}} := \{ \omega_{a,\mathbf{W}, b}: \mathbb{R}^d \rightarrow  \mathbb{R}\; \text{s.t.}\;  \omega_{a,\mathbf{W}, b}(g(x)) = a^\top\text{ReLU}(\mathbf{W}g(x) - b), \; a \in \mathbb{R}^m, \mathbf{W}\in \mathbb{R}^{m \times d}, b \in \mathbb{R}^m  \}$, where $\text{ReLU}(h) = \max(h,0)$. 
% We slightly abuse notation by letting $j\in [m]$ be an indexing over clusters, and defining $\Gamma_j$ as the $j$-th cluster.
\begin{theorem} \label{thm:downstream1} 
% Consider the family of tasks
% $\mathcal{T}_{N}:= \{f \in \mathcal{F}_c : \sum_{j=1}^m \chi\{ y(x)=1 \; \forall\;  x \in \Gamma_j \} \leq N\}$, where $\chi\{\cdot\}$ is the indicator variable. 
Suppose the representation $g^* \in \arg \min_{g \in \mathcal{G}} \mathcal{L}(g)$ under Assumptions \ref{assump:spikyaug} and \ref{assump:realizability}, $\beta > c\log d$ for an absolute constant $c$ and $d>3$. Then for any set of cluster-preserving downstream tasks $\mathcal{T}$, $\mathcal{L}_{\mathcal{T},\J_{\text{ReLU}}}(g^*) = 0$.

% {\color{blue} [KS: $\mathcal{L}_{\mathcal{T},\J_{\text{ReLU}}}$ has not been defined anywhere.][SS: It is in Equation (1), (2)]}

% for any $N \in [m]$ and $\tilde{N} := \max(N,m-N)$, 
% % for any downstream task $y$, there exists a linear head ${w}\in\J_{\text{lin},} :=\mathbb{R}^{ d}$ such that $w^\top g$ achieves zero error, i.e.
% \begin{align}
%      \nonumber
% \end{align}
% or:
% Among all mappings from $m=2^d$ objects to $\mathcal{H}_{d}$, those that map these objects uniformly on the hypercube can express the largest possible number of downstream tasks. 
\end{theorem} 
% In words, $\mathcal{T}_N$ is the set of tasks which positively label images from at most $N$ clusters. 
Theorem \ref{thm:downstream1} shows that any representation learned by minimizing the InfoNCE loss achieves zero downstream error on any task from $\mathcal{F}_c$ with a sufficiently wide two-layer ReLU head.  
% Note that $\mathcal{T}_m = \mathcal{F}_c$, and the maximum number of neurons needed in the head to learn any clean task is $\frac{m}{2}$.

% Next we consider the case in which the heads are two-layer ReLU networks.

% \begin{theorem} \label{thm:downstream1}
%     Let ${U}\in \{-1,1\}^{ 2^d \times d}$ such that each row of $U$ is a unique vertex in $\mathcal{H}_d$. 
% Let $\tilde{g}(x)\coloneqq \text{ReLU}({U}g(x) - (d-1)\mathbf{1})$. Then for any downstream task $\mathcal{T}$, 
% \begin{align}
%     \min_{W \in \mathbb{R}^{r \times 2^d}} \mathcal{L}_{\mathcal{T}}( W, \tilde{g}  ) &= 0
% \end{align}
% \end{theorem}

% lower bound on downstream task error without spiky augmentations

% what if $m \neq 2^d$?

% what if clusters are not the same size?

% example showing how HaoChen assumption/technique can fail on SimCLR? their result follows almost immediately from their assumptions. Ours is more complicated, need to show why. Why is our case hard? would be nice to have positive result for different cluster sizes...

% special case with quantized linear models

% Define the downstream error of a pre-trained representation $g$ on a family of tasks $\mathcal{T}$, with family of heads $\mathcal{W}$.
% \begin{align}
%     \mathcal{L}_{\mathcal{T}, \mathcal{W}}(g) \coloneqq \sup_{T \in \mathcal{T}} \inf_{w \in \mathcal{W}} \mathbb{P}_{x \sim \mathcal{I}}[ w \circ g(x) \neq {T}(x) ].
% \end{align}

{Next, we show that controlling the expressivity of $\mathcal{G}$ is necessary to achieve meaningful downstream performance guarantees. Suppose that instead of optimizing the InfoNCE loss over $\mathcal{G}$, we instead optimized it over a representation class ${\mathcal{G}}_{\star}:= \mathcal{F}_\star^d$ where $\mathcal{F}_\star\coloneqq \{f : {D} \rightarrow \{-1,1\}\}$ consists of all classifiers mapping from images to binary labels.
\begin{theorem} \label{thm:negative}
   Let $\beta > c\log d$ for an absolute constant $c$ and $d>3$. There exists a dataset ${D}$ that satisfies Assumptions \ref{assump:spikyaug} and \ref{assump:realizability} for $\mathcal{G}$, representation $g \in \arg \min_{g'\in \mathcal{G}_{\star}} \mathcal{L}(g')$, and a downstream task $h\in \mathcal{F}_c$ such that 
    % \begin{align}
        $\mathcal{L}_{h,\mathcal{J}_\star}(g) \geq 0.5$,
    % \end{align}
    where $\mathcal{J}_\star= \{\omega : \mathcal{H}_d\rightarrow \{-1,1\} \}$ is the set of all mappings from $\mathcal{H}_d\rightarrow \{-1,1\}$.
\end{theorem}
% \begin{proof} 
% Dataset has two clusters, and there are 2^d vertices. For each vertex, half of the images mapped there by g are from one cluster, the other half are from the other cluster.
% \end{proof}
}

\section{Results for the Agnostic Setting}\label{section:agnostic}

In this section, we consider the setting in which there may not exist any cluster-preserving and uniform representation (that is, Assumption \ref{assump:realizability} is violated). We show that even in this setting, the InfoNCE loss prioritizes cluster-preserving representations. Specifically, we show that if an optimal solution of the InfoNCE loss on $\mathcal{G}$ is close to uniform, then it must also be cluster-preserving. This result requires two new assumptions that we describe below. 

First, the function class  $\mathcal{F}$ must be closed under operations that make classifiers cluster-preserving, in the sense that if $f\in \mathcal{F}$ and $f$ does not preserve the cluster $\Gamma_{\mathbf{c}}$, then the two perturbations of $f$ that preserve $\Gamma_{\mathbf{c}}$ (by assigning $\pm 1 $ to all images within it) and do not change $f$ otherwise are also in $\mathcal{F}$.
\begin{assumption}[Expressivity of $\mathcal{F}$]\label{assump:exist}
For any cluster $\Gamma_{\mathbf{c}}$, if any $f\in \mathcal{F}$ is such that $f(x)\neq f(x')$ for some $x,x'\in \Gamma_{\mathbf{c}}$, then $f' \in \mathcal{F}$ and $f'' \in \mathcal{F}$, where $f'(x) = f''(x) = f(x) \;  \forall x \notin \Gamma_{\mathbf{c}}$, and $f'(x)=1, f''(x)=-1 \;\forall x \in \Gamma_{\mathbf{c}}$.
\end{assumption}

% \begin{assumption}[$(\gamma, \delta)$-Regularity of $\mathcal{F}$] \label{assump:spiky_complete}
% A function class $\F$ is called $(\gamma, \delta)-$regular if 
% \begin{enumerate}[ref={\theassumption.\Alph*}]
% \item \label{assump-a}For every $f\in F$, for any $x\in D_\circ$, $$\Pr_{x^+\sim \text{Unif}( \kappa(x))}[f(x)\ne f(x^+)] \le \gamma.$$
% \item \label{assump-b}For any $f\in \F\setminus \F_c$, let 
% $$\Delta_f := \min_{f_c\in \F_c} \Vert \{x \in D_\circ \mid f(x) \ne f_c(x)\}\Vert_\circ.$$ 
% Then we have $\Pr_{x^+\sim \text{Unif}( \kappa(x))}[f(x)\ne f(x^+)] \ge \delta \Delta_f$.
% \end{enumerate}
% \end{assumption}
Next we define a regularity condition of a function class and augmentation set that
captures the extent to which non-cluster-preserving classifiers classify  images in positive pairs differently within clusters that they intersect. 
% Recall that  a cluster is a set of points classified alike by all   classifiers that classify the samples in every positive pair alike. 
% This means that any classifier that splits a cluster (i.e., any non-clean classifier) must also split an augmentation set within that cluster, as we used critically in the proof of Theorem \ref{thm:uniformandclean}. 
So far, we have only assumed that non-cluster-preserving classifiers misclassify at least  
% does not entail that non-clean classifiers must in 
% only implies that any non-clean classifier classifies the samples in at least 
{\em one}  positive pair differently within any cluster they intersect (Assumption \ref{assump:spikyaug}).
However, for regular classes of binary classifiers and intertwined augmentation sets within clusters, we can expect that the number of positive pairs split in a cluster that are split by any binary classifier scales with the number of negative pairs in the same cluster that are split by the classifier.
% Here, we define the regularity of a class of binary classifiers  by the minimum number of positive pairs in a cluster that must be misclassified by any classifier in the class that does not preserve that cluster. 
% Recall that Theorem \ref{thm:uniformandclean} only required 
For a set of images $B\subseteq D$, we employ the notations  $\|B\|_\circ \coloneqq \mathbb{P}_{x \sim D_{\circ}}[x \in B]$ and $\|B\| \coloneqq \mathbb{P}_{x \sim D\setminus D_\circ}[x \in B]$.
% Further, recall that $\Gamma_{\mathbf{c},\circ}\coloneqq {\Gamma}_{\mathbf{c}}\cap D_\circ$.
\begin{definition}[$\delta$-Regularity] \label{defn:reg} %\label{assump:spiky_complete}
% A function class $\F$ is called $(\gamma, \delta)-$regular if 
% \begin{enumerate}[ref={\theassumption.\Alph*}]
% \item \label{assump-a}For every $f\in F$, for any $x\in D_\circ$, $$\Pr_{x^+\sim  \kappa(x)}[f(x)\ne f(x^+)] \le \gamma.$$
% \item \label{assump-b}
For any $f\in \F$, let $\Sigma_f \coloneqq \{\mathbf{c}\in {C}: %\; \text{s.t.}\; 
\exists  x,x'\in \Gamma_{\mathbf{c}} \; \text{s.t.}\; f(x) \neq f(x')\}$ be the set of content variables corresponding to clusters split by $f$. 
% {\color{blue} Let $S_{\mathbf{c},\circ}$ be ...}
For all ${\mathbf{c}} \in \Sigma_f$ and $\sigma \in \{-1,1\}$, define %\begin{align}
    $f^{({\mathbf{c}},\sigma)}(x) \coloneqq \begin{cases} f(x) & x \notin \Gamma_{\mathbf{c}} \\
    \sigma & x \in \Gamma_{\mathbf{c}} \\
    \end{cases}$ as the classifier that outputs the same label as $f$ on all images not in $\Gamma_{\mathbf{c}}$ and  $\sigma$ on $\Gamma_{\mathbf{c}}$. Further define 
    %,\quad  \nonumber \\
    $\Delta_{f,\mathbf{c}} \coloneqq  %\min_{{\mathbf{c}}  \in \Sigma_f} 
    \min_{\sigma \in \{-1,1\}} \Vert \{x \in \Gamma_{\mathbf{c},\circ} \mid f(x) \ne f^{({\mathbf{c}},\sigma)} (x)\}\Vert_\circ $ as the minimum measure of the set on which $f^{\mathbf{c},\sigma}$ and $f$ differ among all possible choices of $\sigma \in \{-1,1\}$. 
%\end{align}
% Lastly, let $R\coloneqq \{x^+\in A(x): x\in \Gamma_{\mathbf{c},\circ}, \; f(x)\ne f(x^+)\}$ be the set of   augmentations in the cluster $\Gamma_{\mathbf{c}}$ that are misclassified by $f$.
Then $(\F,\Lambda)$ is $\delta-$regular if for all $\mathbf{c} \in \Sigma_f$,  $\|\{\mathcal{A}(x):\mathcal{A}\in \Lambda, x\in \Gamma_{\mathbf{c},\circ}, \; f(x)\ne f(x^+)\}\| \ge \delta \Delta_{f,\mathbf{c}}$.
% \end{enumerate}
\end{definition}

 % Assumption \ref{assump-a} states that \textit{every} classifier in the function class $\F$ can only separate an image from some $\gamma$ fraction of its augmentations. Suppose the natural images are points in $\R^D$, and suppose an augmentation was a perturbation of the natural image by a mean-zero Gaussian. If the class $\F$ consisted of linear half-spaces, then this assumption is satisfied with $\gamma = \frac{1}{2}$ (any halfspace that contains the center of the Gaussian must contain at least half of the augmented data).

 % Below, we effectively assume that any non-clean classifier that splits $k$ natural images from the rest of their cluster must also split a constant fraction of $k$ augmentations from their natural images in that cluster.

\noindent Definition \ref{defn:reg} states that a function class and set of augmentations $(\mathcal{F},\Lambda)$ is $\delta$-regular if the number of positive pairs split by a classifier that intersects a cluster is at least $\delta$ fraction of the extent to which the classifier intersects the cluster. If classifier $f\in \mathcal{F}$ barely intersects the cluster, i.e., classifies most natural images from the cluster alike, then $\Delta_{f,\mathbf{c}}$ is small and the lower bound on the number of positive pairs that it intersects is weaker. However, if $f$ splits the cluster almost in half, then $\Delta_{f,\mathbf{c}}$ is large and the classifier separates many augmentations from their associated natural images within the cluster.  
Next, we state our regularity assumption and the result for the agnostic case. 

\begin{assumption}[$\delta$-Regularity of $(\mathcal{F},\Lambda)$] \label{assump-b}
    The pair $(\mathcal{F},\Lambda)$ is $\delta-$regular with $\delta\geq 0.4$.
\end{assumption}

\vspace{-1mm}
% Now we are ready to state our main result for the agnostic setting. The below theorem states that if $g$ is in some sense close to a uniform representation, it must be clean in order to optimize the InfoNCE loss with sufficiently large $\ell$ and $\beta$. 

%Our main result for the agnostic setting states that if $g$ is in some sense close to a uniform representation, it must be cluster-preserving in order to optimize the InfoNCE loss. % with sufficiently large $\ell$ and $\beta$. 

 % \vspace{-10mm}

\vspace{-1mm}
\begin{theorem} \label{thm:agnostic}
% Let $\ell = \frac{C}{\epsilon} d 2^d $ for an absolute constant $C$ and $\beta \geq c_3 2^d$. 
Suppose Assumptions \ref{assump:exist} and \ref{assump-b} hold and $g= [f_1,...,f_d]$ is \textbf{not} cluster-preserving with 
$\min_{j\in [d]} \min_{\mathbf{c} \in \Sigma_{f_j}} \mathbb{P}_{x,x'\sim D}[x, x' \in \Gamma_{\mathbf{c}},f_j(x)\neq f_j(x')] \geq \epsilon >0$. Let $\ell \geq \frac{ c  }{\epsilon} d 2^d $, $\beta \geq {c}\log(\frac{c}{\epsilon}) 2^d$ for a sufficiently large constant $c$.
% $\min_{j\in [d]} \min_{\mathbf{c} \in \Sigma_{f_j}} \|\{x^+\in A(x): x\in \Gamma_{\mathbf{c},\circ}, \; f_j(x)\ne f_j(x^+)\}\| \geq \epsilon >0$. 
Moreover, suppose $g$
 is close to a uniform
representation in the sense that\footnote{Note that this near-uniformity condition allows for representations that for each vertex put  mass at least a constant factor of $\frac{1}{d}$ times $2^{-d}$, or essentially treat the vertex as   inactive,
 % Recall that for a uniform representation, $\mathcal{D}_g(v) = 2^{-d}$ for all $v$. Here, we consider representations that for all vertices have a least a factor of $\frac{\log(d)}{d}$ from $2^{-d}$ mass, or treat the vertex as essentially  inactive, 
 which allows for the case wherein the number of clusters is less than $2^d$ and some vertices are inactive for cluster-preserving representations.}  $\mathbb{P}_{x\sim D_\circ}[g(x)=v] \geq \frac{10}{ c d2^d}$ 
or $\mathbb{P}_{x\sim D_\circ}[g(x)=v]  \leq \frac{ \epsilon}{ 100 c d 2^{2d}} $
% $\|Q_v\|_\circ > \frac{\log(\eta)}{c 2^d}$ or $\|Q_v\|_\circ \leq \frac{c'}{ \eta d 2^{2d}}$ 
for all $v\in \mathcal{H}_d$.
Then $g$ is \textbf{not} a minimizer of the InfoNCE loss. 
% $\mathcal{L}(g)> \mathcal{L}(g')$ for some nearby representation $g'$. 
% Consider the InfoNCE objective \ref{eq:InfoNCE} when $l \ge m \log m$. 
% The optimal representation is clean and such that it maximizes the entropy of the downstream representations to within an additive error.
\end{theorem}

\begin{comment}

\begin{figure}[t]
\centering
% \resizebox{5cm}{!}{\input{spikyproof.pdf}} %{\input{dirtyToClean.tikz}}
\centerline{\includegraphics[width=0.9\columnwidth]{spikyproof2.pdf}}
\caption{Example partitioning of images with $d=2$ and $D\subset \mathbb{R}^2$. Green triangles denote natural images, and black ellipses their corresponding augmentation sets (we have drawn the augmentation sets as ellipses for ease of presentation, but in reality they are typically highly non-smooth and non-contiguous). Clusters are indicated by dotted black ellipses. Here the ``cleaner'' representation $g'$ is constructed by making $f_1$ clean with respect to the cluster $S_{\mathbf{c}'}$. The region $R$  consisting of augmentations in $S_{\mathbf{c}'}$ misclassified by $f_1$ is shaded red, and the set $B$ of images which are classified differently by $f_1$ and $f_{1}^{(\mathbf{c}',\sigma')}$ is indicated by blue diagonal lines. By Assumption \ref{assump-b}, $\|R\| \geq \delta \|B\|_{\circ}$.} \label{fig:agnostic}
 % Illustration of a perturbation from a representation that is not clean to a cleaner representation. The lines represent the decision boundary of a single coordinate of the representation, which is hypothesized to be non-clean. The black dots represent the natural images that correspond to each of the clusters/augs sets, while the ellipses represent the augmentation sets corresponding to the black dots. The red shaded region represents the misclassification of the classifier corresponding to the non-clean coordinate.
\end{figure}
\end{comment}

%\begin{comment}
\begin{figure}
  \begin{minipage}[c]{0.35\textwidth}
  \vspace{-6mm}
    \includegraphics[height=2.4in,width=1.9in]{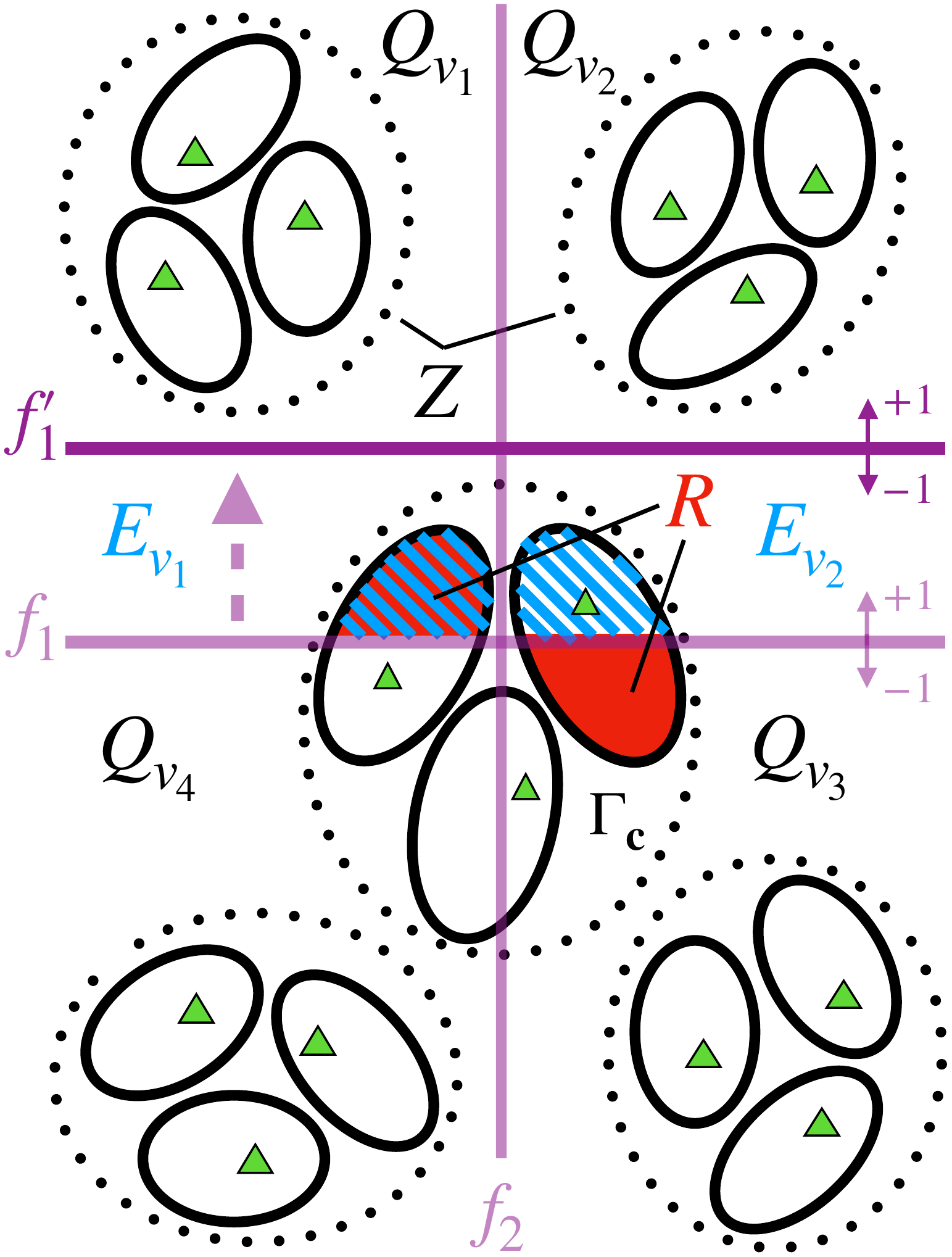}
  \end{minipage}\hfill %\hspace{-0.5in}
  \begin{minipage}[c]{0.65\textwidth}
\caption{Example partitioning of images with $d=2$ and $D\subset \mathbb{R}^2$. Green triangles denote natural images and solid black ellipses denote their corresponding augmentation sets (here we have drawn the augmentation sets as compact convex sets for ease of presentation, but in reality they may be non-simply connected and non-smooth). Clusters are indicated by dotted black ellipses. The non-cluster-preserving representation $g=(f_1,f_2)$, and we construct $g'=(f_1',f_2)$ by making $f_1$ preserve the cluster $\Gamma_{\mathbf{c}}$. The region $R$  consisting of augmentations in $\Gamma_{\mathbf{c}}$ misclassified by $f_1$ is shaded red, and the set $E$ of images which are classified differently by $f_1$ and $f_{1}'$ is indicated by blue diagonal lines. By Assumption \ref{assump-b}, $\|R\| \geq \delta \|E\|_{\circ}$, where $\|B\| \coloneqq \mathbb{P}_{x \sim D\setminus D_\circ}[x \in B]$  and  $\|B\|_\circ \coloneqq \mathbb{P}_{x \sim D_{\circ}}[x \in B]$ for any set of images $B \subseteq D$ \label{fig:agnostic}.
}
\end{minipage}
\end{figure}
%\end{comment}

\begin{proof}[Proof sketch of Theorem \ref{thm:agnostic}] For a non-cluster-preserving representation $g$ that is ``close" to a uniform representation, we construct a nearby representation $g'$ by changing one coordinate of $g$ such that it preserves one additional cluster, and show that the resulting $g'$ achieves smaller InfoNCE loss than $g$. 
% We construct $g'$ by changing only one coordinate of $g$ such that it preserves one cluster that it intersects. 
In particular suppose WLOG that $f_1$ does not preserve the cluster $\Gamma_{\mathbf{c}}$.
% has largest misclassification of positive samples on its most misclassified cluster among all $d$ classifiers in $g$, specifically that $1 \in \arg\max_{j\in [d]} \max_{\mathbf{c}' \in \Sigma_{f_j}} \|\{x^+\in A(x): x\in \Gamma_{\mathbf{c}',\circ}, \; f_j(x)\ne f_j(x^+)\}\|$. 
% meaning Suppose WLOG  that $1 \in \arg\max_{j\in [d]} \max_{\mathbf{c} \in \Sigma_{f_j}} \|\{x^+\in A(x): x\in \Gamma_{\mathbf{c},\circ}, \; f_j(x)\ne f_j(x^+)\}\|$, meaning 
Further, let $f_1^{(\mathbf{c},\sigma)}$ be the smallest perturbation of $f_1$ that preserves $\Gamma_{\mathbf{c}}$, as defined in Definition \ref{defn:reg}.
% $ \sigma \in \arg \min_{\sigma' \in \{-1,1\}} \Vert \{x \in D_\circ \mid f_1(x) \ne f_{1}^{(\mathbf{c},\sigma')} (x)\}\Vert_\circ $, using the notation defined in Assumption \ref{assump-b}, and 
Denote $f_1'= f_1^{(\mathbf{c},\sigma)}$.  {By Assumption \ref{assump:exist}, $f_1' \in \mathcal{F}$.} 
Construct $g' =[f_1',f_2,\dots,f_d] \in \mathcal{G}$. 
Note that $g'$ is equivalent to $g$ on all but one coordinate, and the one differing coordinate differs only on one cluster.

% where 
% \begin{align*}
% \mathcal{L}_{\text{pos}}(g) &:= -\beta \mathbb{E}_{x, x^+}\left[g(x)^\top g(x^+)\right] \nonumber \\
% {L}_{\text{neg}}(g) 
% &:= \mathbb{E}_{x,x^+, \{x^-_i\}_{\ell}} \bigg[\log\bigg({ e^{\beta g(x)^\top g(x^+)}\!+\!\sum_{i=1}^\ell e^{\beta g(x)^\top g(x^-_{i})}  } \bigg)\bigg] \nonumber
% \end{align*}

% To analyze the negative loss, it is helpful to first partition the space 

To characterize the variation in the InfoNCE loss when moving from $g$ to $g'$, we first consider a specific partition of the space of images defined based on the representations $g$ and $g'$. In particular, for a given vertex $v \in \mathcal{H}_d$, consider the set $Q_v := \{x \in {D} : g(x)=v, g'(x)=v \}$ which denotes the set of images that both $g$ and $g'$ map to vertex $v$, and the set $E_v := \{x \in  {D} : g(x)= v, g'(x)\neq v \}$ which denotes the set of images that $g$ maps to $v$ and $g'$ maps to another vertex. Considering these definitions, the set  $Q:= \cup_{v \in \mathcal{H}_d} Q_v$ corresponds to the set of all images that $g$ and $g'$ map to the same vertex, while $E:= \cup_{v \in \mathcal{H}_d} E_v$ denotes the set of all images which $g$ and $g'$ map to different vertices. Based on this construction, it is not hard to observe that for any $v\neq v'$ the sets $Q_v$, $Q_{v'}$, $E_v$, and $E_{v'}$  are disjoint, and each image belongs to either some $Q_v$ or $E_v$. Hence, the concatenation of these sets partitions the space of images. 
 Figure \ref{fig:agnostic} illustrates this partition for a special case with $d=2$. 
 % Finally, note that the sets $Q$ and $E$ are also disjoint and their union is the set of all images. 
 The above partition is critical as we divide our sensitivity analysis into multiple cases based on the location of the positive and negative images in this partition.

% For all $v \in \mathcal{H}_d$, define the sets $Q_v := \{x \in {D} : g(x)=v, g'(x)=v \}$ and $E_v := \{x \in  {D} : g(x)= v, g'(x)\neq v \}$.  

%Let  $Q:= \cup_{v \in \mathcal{H}_d} Q_v$ and $E:= \cup_{v \in \mathcal{H}_d} E_v$, meaning $Q$ is the set of images on which $g$ and $g'$ are the same, whereas $E$ is the set of images on which they differ.
% Note that $\|B\|_\circ = \Delta_{f_1} $ (defined in Assumption \ref{assump-b}) by construction of $g'$. 
%Also observe that all of the $Q_v$'s and $E_v$'s are disjoint, and $Q \cup E = {D}$. 

Let us define $\mathcal{L}^+$ and $\mathcal{L}^-$ as the alignment and uniformity losses in \eqref{eq:InfoNCE}, respectively. We refer to $\mathcal{L}^+$ as the positive part of the loss as it deals with positive samples (augmented images), and we  refer to $\mathcal{L}^-$ as the negative part of the loss as it contains negative samples. 
To prove that moving from $g$ to $g'$ decreases the loss, i.e., $\mathcal{L}(g)- \mathcal{L}(g')>0$, we show that the amount that the positive part of the loss decreases is more than the amount the negative part might increase: 
$\mathcal{L}^+(g) - \mathcal{L}^+(g') > \mathcal{L}^-(g') - \mathcal{L}^-(g)$.
To do so, first, note that the variation in the positive part is
\begin{align}
    \mathcal{L}^+(g) - \mathcal{L}^+(g') = 2\beta \left(\Pr\left[x\in Q, x^+\in E\right] + \Pr\left[x\in Q, x^+\in E\right]\right). \label{eq:sumv} % =: \Lambda^+. _{x \sim D_\circ, x^+\sim A(x)}
\end{align}
This holds as $ \beta g(x)^\top g(x^+)=\beta g(x)'^\top g'(x^+)$ except for the cases that $x\in Q, x^+\in E$ or $x\in Q, x^+\in E$. In these two cases, they differ by $2\beta$. Note that the augmentations that belong to either of these two cases lie in the area shaded red in Figure \ref{fig:agnostic}. {{
% Recall $f_1^{(\mathbf{c},\sigma)}(x)$ as the replacement of $f_1$ in the representation $g'$. Now define the set $Z \coloneqq \{ Q_v:f_1^{(\mathbf{c},\sigma)}(x) = f_1(x') \; \forall x \in Q_v, x' \in E\}$ as the set of images that $g$ and $g'$ 
% (e.g. $\{Q_{v_1},Q_{v_2}\}$ in Fig. \ref{fig:agnostic}). 
We refer to the set of augmentations in this region as $R$, in other words, $R$ is the set of augmentations in $\Gamma_{\mathbf{c}}$ that are classified differently than their natural image by $f_1$. 
Thus, we can write $\mathcal{L}^+(g) - \mathcal{L}^+(g') = 2\beta\|R\|$.
% , i.e. the size of $R$ is the error of $f_1$ on the positive loss on  $\Gamma_{\mathbf{c}'}$. 
}}

%Note that the expression  $\Pr\left[x\in Q, x^+\in E\right] + \Pr\left[x\in E, x^+\in Q \right] $ is exactly the measure of the region $R$ shaded red in Figure \ref{fig:agnostic}.

% In particular, we construct $g'$ by making one classifier in $g$ clean with respect to one cluster 
% Denote by $\Vert \cdot \Vert_\circ$ the measure associated with natural images, that is 
% $$ \Vert A\Vert_\circ \coloneqq \Pr_{x\sim \I}[x\in A].$$ Similarly denote by $\Vert \cdot \Vert$ the measure associated with augmentations, so 
% $$ \Vert A\Vert \coloneqq \Pr_{x\sim \I, x^+\sim \A(x)}[x^+\in A].$$

Next, we consider the difference in negative parts of the loss. To bound this difference, we leverage the partitioning of the space of images defined above to 
% consider all disjoint cases with $x \in Q_v$ and $x \in E_v$, for all vertices $v$. We 
decompose the variation of the losses based on the set that image $x$ belongs to. In particular, if we define the function $\mathcal{L}^-_{B}(g):=\mathbb{E}_{x,x^+, \{x^-_i\}_{\ell}} \big[ \chi({B})  \log\big({ e^{\beta g(x)^\top g(x^+)}+\sum_{i=1}^\ell e^{\beta g(x)^\top g(x^-_{i})}  } \big)\big] $ for any event~$B$, where $\chi({B})$ is the indicator random variable for the event $B$, then using the fact that each image $x$ either belongs to one of the $Q_v$'s or $E_v$'s we can write  
\begin{align}
    \mathcal{L}^-(g') - \mathcal{L}^-(g) = \sum_{v \in \mathcal{H}_d} \big[\mathcal{L}^-_{\{x\in Q_v\}}(g') - \mathcal{L}^-_{\{x\in Q_v\}}(g)\big] +\big[\mathcal{L}^-_{\{x\in E_v\}}(g') - \mathcal{L}^{-}_{\{x\in E_v\}}(g)\big],
\end{align}
Since the cases with $x \in E_v$ utilize similar analysis for those with $x \in Q_v$, we focus on the $x \in Q_v$ cases here and defer the $x \in E_v$ cases to Appendix \ref{appendix:agnostic}. 

To analyze $\mathcal{L}^-_{\{x\in Q_v\}}(g') - \mathcal{L}^-_{\{x\in Q_v\}}(g)$, we first observe that this difference is non-positive for a subset of the $Q_v$'s. Note in Fig. \ref{fig:agnostic} that if $x$ belongs to  $Q_{v_1}$ or $Q_{v_2}$, then moving from $f_1$ to $f_1'$  decreases the representation similarity for some pairs of negative samples (those with $x_i^-\in E$) while keeping the rest the same. So, the negative part of the loss cannot increase going from $g$ to $g'$ if $x$ lies  in either $Q_{v_1}$ or $Q_{v_2}$. We formally define this set of $Q_v$'s as $Z \coloneqq \{Q_{v}: f'_{1}(x)\neq f'_1(x^-) \; \forall x \in Q_{v}, x^-\in E\}$.
At a high level, the reason this definition implies the negative part of the loss does not increase if $x \in Q_v \in Z$ is because $f_1$ and $f_1'$ must agree on $Q_v\in Z$  and disagree on $E$, so since $f_1'$ differs on $Q_v\in Z$  and $E$, $f_1$ must agree on these sets. Thus, the similarity between negative pairs consisting of $x \in Q_v \in Z$ and $x_i^-\in E$ diminishes when moving from $g$ to $g'$. Thus, we have  
\begin{align}
   \sum_{v \in \mathcal{H}_d} \mathcal{L}^-_{\{x\in Q_v\}}(g') - \mathcal{L}^-_{\{x\in Q_v\}}(g) \leq \sum_{v \in \mathcal{H}_d} \mathcal{L}^-_{\{x\in Q_v\notin Z\}}(g') - \mathcal{L}^-_{\{x\in Q_v\notin Z\}}(g)  
\end{align}

% We know $f_1(x) = f_1'(x)$ by definition of $Q_v$

% If $x^-\in Q_{v'}$ then $f_1(x^-) = f_1'(x^-)$ by definition of $Q_v'$
% Therefore $g'(x)^\top g'(x^-) = g(x)^\top g(x^-)$ since no other coordinates change

% If $x^-\in E$ then $f_1(x^-) \neq f_1'(x^-)$ by definition of $E$
% Also, $f_1'(x) \neq f_1'(x^-)$  by definition of Z
% Therefore $f_1(x) = f_1(x^-)$
% Therefore $g'(x)^\top g'(x^-) = g(x)^\top g(x^-) - 2 $ since no other coordinates change

% In all cases $g'(x)^\top g'(x^-) \leq g(x)^\top g(x^-)$.

% then $\mathcal{L}^-_{\{x\in Q_v\}}(g') - \mathcal{L}^-_{\{x\in Q_v\}}(g) \leq 0$. This is best understood via  Fig. \ref{fig:agnostic}, in which $Z = \{Q_{v_1},Q_{v_2}\}$. Observe that if $x \in Q_{v_1}$ or $x \in Q_{v_2}$, the representations of all negative samples only become less 

% We first note that  $\mathcal{L}^-_{\{x \in Q_v \in Z\}}(g) - \mathcal{L}^-_{\{x \in Q_v \in Z\}}(g')> 0$ since  $g'(x^-)$ is less similar to $g'(x)$ than $g(x^-)$ is to $g(x)$ for all $x^- \in B$ and $x \in Q_v \in Z$ (as seen in Fig. \ref{fig:agnostic}), and the similarities are unchanged for $x^+$ and all other $x^-$. So, we can upper bound $$\mathcal{L}^-(g') - \mathcal{L}^-(g) \leq \sum_{v \in \mathcal{H}_d} \mathcal{L}^-_{\{x\in Q_v\notin Z\}}(g') - \mathcal{L}^-_{\{x\in Q_v\notin Z\}}(g) +\mathcal{L}^-_{\{x\in E_v\}}(g') - \mathcal{L}^{-}_{\{x\in E_v\}}(g).   $$

Now, for each event $\{x \in Q_v\notin Z
\}$, we consider two cases depending on the number of negative samples in $Q_v$. (1) If there is at least one negative sample $x_i^- \in Q_v$, then
$g'(x)= g'(x_i^-) = g(x) = g(x_i^-)$, so both the log-sums in $\mathcal{L}^-_{\{x\in Q_v\notin Z\}}(g')$ and $\mathcal{L}^-_{\{x\in Q_v\notin Z\}}(g)$ are dominated by $e^{\beta d}$ terms and the losses do not significantly differ (using that log-sum is approximately a max operation). (2) If no negative samples lie in $Q_v$, then the dominant terms in the log-sum for $\mathcal{L}^-_{\{x\in Q_v\notin Z\}}(g')$ may be a factor of $e^{2\beta}$ larger than the dominant terms for $\mathcal{L}^-_{\{x\in Q_v\notin Z\}}(g)$, requiring a sharp analysis to control the probability   these events occur.
Letting $n_{1,v}$ denote the number of negative samples in $Q_v$, we define these two  cases above  as $B_{v,1}\coloneqq\{x \in Q_v \notin Z, n_{1,v}>0\}$ and $B_{v,2
}\coloneqq \{x \in Q_v \notin Z, n_{1,v}=0\}$, respectively. Note that they form a partition of $\{x \in Q_v\notin Z\}$, so we have $\mathcal{L}^-_{\{x\in Q_v\notin Z\}}(g') - \mathcal{L}^-_{\{x\in Q_v\notin Z\}}(g) = \sum_{j=1}^2 \mathcal{L}^-_{B_{v,j}}(g') - \mathcal{L}^-_{B_{v,j}}(g)$. 
We detail each case below, where $n_2$ is the number of negative samples in $E$.

 \textbf{Case 1:} $B_{1,v} := \{x \in Q_v \notin Z, n_{1,v}>0\}$. In this case %$ g'(x)^\top g'(x_{i}^-) = g(x)^\top g(x_{i}^-) = d$ for at least one $i \in [\ell]$, meaning that
the dominant terms in the log sums for $\mathcal{L}^-_{B_{1,v}}(g')$ and $\mathcal{L}^-_{B_{1,v}}(g)$ are both $e^{\beta d}$, although the losses may  differ in the number of such terms, which can be, in the worst case, $n_{1,v}+n_2+1$ for $g'$ and $n_{1,v}$ for $g$. This is because  $g$ and $g'$ can disagree on at most $n_2$ negative samples, and they can also disagree on the positive sample. Thus, $\!\mathcal{L}_{B_{1,v}}^-(g')\!- \!\mathcal{L}_{B_{1,v}}^-(g) \! \leq \mathbb{E}\big[\chi(B_{1,v}) \log\big(\tfrac{ n_{1,v} + n_2 + 1 }{n_{1,v}} \big) \big]  \leq \mathbb{E}[\chi(B_{1,v}) \frac{2(n_2+1)}{n_{1,v}+1}]$, where the last inequality follows using $\log(1+x) \leq x$.                                                                                    
% If $n_2=0$ then the loss of $g'$ can only increase if $x^+\in B$, and a simple calculation shows that $\mathcal{L}_{E_1\cap \{n_2=0\}}^-(g') - \mathcal{L}_{E_1\cap \{n_2=0\}}^-(g)\leq \log(2) \mathbb{P}(x\in Q_v, x^+ \in B) \mathbb{P}(n_{1,v}=0, n_2=0)$. For general $n_2$ we can show that 
% \begin{align}
%     \mathcal{L}_{E_1}^-(g') - \mathcal{L}_{E_1}^-(g)\leq \mathbb{E}\big[ \chi\{E_1\} \log( ) \big] 
% \end{align}
% \begin{align}
%     \!\mathcal{L}_{B_{v,1}}^-(g')\!- \!\mathcal{L}_{B_{v,1}}^-(g) \!\leq \mathbb{E}\big[\chi\{B_{v,1}\} \log\big(\tfrac{ n_{1,v}e^{\beta d} + (n_2+1)e^{\beta d} }{n_{1,v}e^{\beta d} } \big) \big] \leq \mathbb{E}\big[\chi\{B_{v,1}\} \log\big(\tfrac{ n_{1,v} + n_2 + 1 }{n_{1,v}} \big) \big], \nonumber
% \end{align}
% where $ \log\big(\tfrac{ n_{1,v} + n_2 + 1 }{n_{1,v}} \big)\leq \frac{2(n_2+1)}{n_{1,v}+1}$ using $\log(1+x) \leq x$.  
We bound $\mathbb{E}[\chi(B_{1,v}) \frac{2(n_2+1)}{n_{1,v}+1}]$ by writing the trinomial expansion of the expectation (note that the joint distribution of ($n_{1,v},n_2$) is trinomial with parameters ($\|Q_v\|_\circ, \|E\|_\circ$)), 
% and factoring the $n_2$ and $n_{1,v}+1$ terms to obtain a new trinomial 
and further simplifying to result in an upper bound of $\|E\|_\circ$. Importantly, this bound is  $O(\|R\|)$ by Assumption \ref{assump-b} and independent of $\beta$, so we control it by making $\beta$ large enough. 
% Note that this analysis avoids requiring scaling $\beta$
% inversely with $\|R\|$.

% Thus, a naive bound gives $$\mathcal{L}_{E_1}^-(g')- \mathcal{L}_{E_1}^-(g) \leq \mathbb{E}\big[\chi\{E_1\} \log\big(\tfrac{ (\ell+1)e^{\beta d}  }{ n_{1,v}e^{\beta d}} \big) \big]\leq \mathbb{P}(x\in Q_v\notin Z)\log(\ell+1),$$ 
% where the RHS is $O(\beta \|R\|)$ for sufficiently large $\beta$. We give a more fine-grained analysis to avoid scaling  $\beta$ with $1/\|R\|$. 

\textbf{Case 2:} $B_{2,v} := \{x \in Q_v \notin Z, n_{1,v}=0\}$. 
Since here there is no shared dominant $e^{\beta d}$ term in the log-sums for $\mathcal{L}^-_{B_{2,v}}(g')$ and $\mathcal{L}^-_{B_{2,v}}(g)$, the dominant terms for $g'$ may involve strictly larger similarities than those for $g$, corresponding to $x_i^- \in E$ and $x^+ \in E$ (the only samples on which $g'$ and $g$ can disagree).
% Specifically, the loss for $g$ may be up to a factor $2\beta$ larger than the loss for $g$ for some batches  $(x,x^+,\{x_i^-\}_\ell)$, which means that $\mathcal{L}_{B_{v,2}}^-(g')- \mathcal{L}_{B_{v,2}}^-(g)$ is up to $2\beta$ times the probability that these batches are drawn. 
These events are bounded depending on whether $n_2=0$. If $n_2=0$, the loss of $g'$ exceeds that of $g$  iff $g'(x)^\top g'(x^+) = g(x)^\top g(x^+) + 2$, which occurs iff $x^+ \in E$. 
% Thus  $\mathcal{L}_{\{x \in Q_v \notin Z, n_1=0,n_2=0\}}^-(g') - \mathcal{L}_{\{x \in Q_v \notin Z, n_1=0,n_2=0\}}^-(g)\leq 2\beta \mathbb{P}(x \in Q_v, x^+ \in B) \mathbb{P}(n_{1,v}=0,n_2=0)$. 
% For $n_2 >0$, a similar argument shows that $\mathcal{L}_{\{x \in Q_v \notin Z, n_1=0,n_2>0\}}^-(g') - \mathcal{L}_{\{x \in Q_v \notin Z, n_1=0,n_2>0\}}^-(g)\leq 2\beta \mathbb{P}(x \in Q_v) \mathbb{P}(n_{1,v}=0,n_2>0)$.
If $n_2 >0$, the loss can increase by $2\beta$ regardless of the value of $x^+$. 
Combining these sub-cases yields
\begin{align}\mathcal{L}_{B_{2,v}}^-(g') - \mathcal{L}_{B_{2,v}}^-(g)&\leq 2\beta \mathbb{P}(x \in Q_v\notin Z, x^+ \in E) \mathbb{P}(n_{1,v}=0,n_2=0) \nonumber \\
&\quad + 2\beta \mathbb{P}(x \in Q_v\notin Z) \mathbb{P}(n_{1,v}=0,n_2>0). \nonumber
\end{align}
For each term above, we need to show that the coefficient of $2\beta$ is $o(\|R\|)$ even after it is summed over $v$. Note that both terms scale with the probability that $x \in Q_v \notin Z$ {\em and} no negative samples are in $Q_v$. To control this probability we leverage that the distribution induced by $g$ is close to uniform in the sense that every ``active'' vertex $v$ has mass $\mathbb{P}[g(x)=v]=\tilde{\Omega}(\frac{1}{d2^d})$. We use this fact to bound $\mathbb{P}[x \in Q_v]$. Note that the set of images that $g$ maps to $v$ is $Q_v \cup E_v$, yet for all $Q_z \notin Z$, $E_v=\emptyset$ since, at a high level, $f_1$ must separate these $Q_v$ from $E$. So, $\mathbb{P}[g(x)=v] = \mathbb{P}[x \in Q_v]$ for all $Q_v\notin Z$. Therefore, we can show that  with large $\ell$  it is highly unlikely that $x\in Q_v \notin Z$ and none of the negative samples are in $Q_v$. To complete the bounds, we leverage the facts that  $\mathbb{P}[x \in Q_v\notin Z, x^+ \in E]$ scales with $\|R\|$ for the first term, and $\mathbb{P}[n_2>0]$ scales with $\|E\|_\circ$ for the second term, where $\|E\|_\circ = O(\|R\|)$ by Assumption \ref{assump-b}.

After performing a similar analysis for $\{ x \in E_v\}$ and summing the resulting bounds over $\{v \in \mathcal{H}_d\}$, as in \eqref{eq:sumv},  we obtain $\mathcal{L}^-(g') - \mathcal{L}^-(g)< 2 \beta \|R\| = \mathcal{L}^+(g) - \mathcal{L}^+(g')$.
\begin{comment}
{\color{gray}
Note that the two events above form a partition of $\{x \in Q_v\notin Z\}$, i.e. $\mathcal{L}^-_{\{x\in Q_v\notin Z\}}(g') - \mathcal{L}^-_{\{x\in Q_v\notin Z\}}(g) = \sum_{j=1}^2 \mathcal{L}^-_{B_j}(g') - \mathcal{L}^-_{B_j}(g)$. After summing over $\{v \in \mathcal{H}_d: Q_v\notin Z\}$ and performing a similar analysis for $\{ x \in E_v\}$, we obtain
\begin{align}
    \mathcal{L}^-(g') - \mathcal{L}^-(g) &\leq  2\beta \sum_{v: Q_v \notin Z} \big(\mathbb{P}(x \in Q_v \cap x^+ \in E) (1 - \|Q_v\|_\circ)^\ell   +   \ell \|B\|_\circ \|Q_v \|_\circ (1 - \|Q_v \|_\circ)^{\ell-1} \big)\nonumber\\
&\quad +  2\beta \sum_{v\in \mathcal{H}_d} \|E_v\|_\circ (1 - \|E_v\|_\circ)^\ell  + o(\beta \|R\|)  \nonumber
    % \|Q_v\|_\circ(1 - \|Q_v\|_\circ - \|B\|_\circ)^\ell + 
\end{align}
To bound the RHS, we leverage $g$ being close to uniform to upper bound $\sum_{v \in \mathcal{H}_d: Q_v \notin Z} \|Q_v \|_\circ (1 - \|Q_v \|_\circ)^{\ell-1}$, noting that $\mathcal{D}_g(v) = \|Q_v\|_\circ$ iff $Q_v \notin Z$. Finally, we show that $\sum_{v \in \mathcal{H}_d} \mathbb{P}(x \in Q_v \cap x^+ \in B) (1 - \|Q_v\|_\circ)^\ell \lesssim \frac{\|R\|}{e}$,  and  $\sum_{v \in \mathcal{H}_d} \|E_v\|_\circ (1 - \|E_v\|_\circ)^\ell\lesssim \frac{\|E\|_\circ}{e} \leq \frac{\|R\|}{\delta e}$ (by Assumption \ref{assump-b}), meaning that $\mathcal{L}^-(g') - \mathcal{L}^-(g) < \mathcal{L}^-(g) - \mathcal{L}^+(g') $ for $\delta \geq 0.6$, completing the proof. }
% {\color{red} need to use Weighted InfoNCE loss?}
\end{comment}

\end{proof}

Theorem \ref{thm:agnostic} shows that for large $\ell$ and $\beta$, all minimizers of the InfoNCE loss that are near-uniform must be cluster-preserving regardless of the sizes of each cluster or the number of clusters. 
% Combined with Lemma \ref{lemma:entropy}, which shows that the InfoNCE loss almost behaves like the Shannon entropy and therefore encourages uniform representations, Theorem \ref{thm:agnostic} suggests that  in the agnostic case. 
However, it does not rule out that there could be a highly non-uniform and non-cluster-preserving optimal representation. In Appendix \ref{app:weighted}, we show that if we re-weight the alignment and uniformity losses in the InfoNCE loss, we can ensure that {\em all} minimizers of the InfoNCE loss are cluster-preserving.

% {\color{red} define near-uniform formally before theorem statement?}

% Discussion of why this result is significant? 

% Weighted InfoNCE loss for non-uniform distributions. 

% {\color{red}Finally, we have the following Lemma that states that, among clean representations, the InfoNCE loss maximizes the entropy of the pre-trained representations when $l=\{1, 2\}$. Note that the result is similar to that of Lemma \ref{thm:uniform}, however we can completely characterize the ordering on distributions in this case and the the proof is specific to $d=1$ (Lemma \ref{thm:uniform} only holds for $d>4$). 

% \begin{lemma}
%     Equation (\ref{eq:Ltilde}) is minimized by the clean representation $g$ that induces a distribution $\mathcal{D}_g$ with maximal entropy.
% \end{lemma}
% The proof follows by writing out the loss explicitly and optimizing over the one parameter family of distributions of Rademacher random variables, and is provided in Appendix~\ref{appendix:maxentropy}.}

% \section{Proof Sketch for Theorem \ref{thm:uniform}}

% \input{Theory/SparseCoding.tex}
% \input{Experiments.tex}
 \section{Conclusion}

 We study properties of minimizers of the InfoNCE loss optimized over function classes with restricted complexity relative to the complexity of augmentations in the dataset, in realistic settings with disjoint augmentation sets and finite negative samples.  Our results show that such representations are uniform and cluster-preserving in the realizable setting, and must be cluster-preserving if they are close to uniform in the agnostic setting. We believe that our novel analytical tools, namely our stochastic argument for the optimality of representations and our inverse partitioning of the space of images, may be of use for future studies of the InfoNCE loss.

% We do not explore the role of bias induced by the optimization algorithm, although it is feasible that it induces a similar bias as the function class, i.e.
% For instance, it may be the case the that the function class contains representations that can cleanly separate all augmentation sets from each other, meaning that the clusters induced by $\mathcal{F}_c$ consist only of single augmentation sets. However, 
% the set of clean  representations attainable by a particular algorithm may induce a meaningful clustering of the data. We leave future work to study this hypothesis.

% Acknowledgments---Will not appear in anonymized version
% \acks{We thank a bunch of people and funding agency.}

\section*{Acknowledgements}
This research is supported in part by NSF Grants 2127697, 2019844 and 2112471, ARO Grant
W911NF2110226, the Machine Learning Lab (MLL) at UT Austin, and the Wireless Networking and Communications Group (WNCG) Industrial Affiliates Program.

% This research is supported in part by NSF Grants 2127697, 2019844, and 2112471, ARO Grant
% W911NF2110226, the Machine Learning Lab (MLL) at UT Austin, and the Wireless Networking
% and Communications Group (WNCG) Industrial Affiliates Program.

\bibliography{refs}

\begin{thebibliography}{56}
\providecommand{\natexlab}[1]{#1}
\providecommand{\url}[1]{\texttt{#1}}
\expandafter\ifx\csname urlstyle\endcsname\relax
  \providecommand{\doi}[1]{doi: #1}\else
  \providecommand{\doi}{doi: \begingroup \urlstyle{rm}\Url}\fi

\bibitem[Arora et~al.(2019)Arora, Khandeparkar, Khodak, Plevrakis, and
  Saunshi]{arora2019theoretical}
Sanjeev Arora, Hrishikesh Khandeparkar, Mikhail Khodak, Orestis Plevrakis, and
  Nikunj Saunshi.
\newblock A theoretical analysis of contrastive unsupervised representation
  learning.
\newblock In \emph{36th International Conference on Machine Learning, ICML
  2019}, pages 9904--9923. International Machine Learning Society (IMLS), 2019.

\bibitem[Ash et~al.(2021)Ash, Goel, Krishnamurthy, and
  Misra]{ash2021investigating}
Jordan~T Ash, Surbhi Goel, Akshay Krishnamurthy, and Dipendra Misra.
\newblock Investigating the role of negatives in contrastive representation
  learning.
\newblock \emph{arXiv preprint arXiv:2106.09943}, 2021.

\bibitem[Awasthi et~al.(2022)Awasthi, Dikkala, and Kamath]{awasthi2022more}
Pranjal Awasthi, Nishanth Dikkala, and Pritish Kamath.
\newblock Do more negative samples necessarily hurt in contrastive learning?
\newblock In \emph{International Conference on Machine Learning}, pages
  1101--1116. PMLR, 2022.

\bibitem[Bachman et~al.(2019)Bachman, Hjelm, and
  Buchwalter]{bachman2019learning}
Philip Bachman, R~Devon Hjelm, and William Buchwalter.
\newblock Learning representations by maximizing mutual information across
  views.
\newblock \emph{Advances in neural information processing systems}, 32, 2019.

\bibitem[Balestriero and LeCun(2022)]{balestriero2022contrastive}
Randall Balestriero and Yann LeCun.
\newblock Contrastive and non-contrastive self-supervised learning recover
  global and local spectral embedding methods.
\newblock \emph{arXiv preprint arXiv:2205.11508}, 2022.

\bibitem[Bao et~al.(2022)Bao, Nagano, and Nozawa]{bao2022surrogate}
Han Bao, Yoshihiro Nagano, and Kento Nozawa.
\newblock On the surrogate gap between contrastive and supervised losses.
\newblock In \emph{International Conference on Machine Learning}, pages
  1585--1606. PMLR, 2022.

\bibitem[Bremaud(2001)]{bremaud}
Pierre Bremaud.
\newblock \emph{{Markov chains: Gibbs fields, Monte Carlo simulation, and
  queues; 1st ed.}}
\newblock Texts in applied mathematics. Springer, Berlin, 2001.

\bibitem[Brown et~al.(2020)Brown, Mann, Ryder, Subbiah, Kaplan, Dhariwal,
  Neelakantan, Shyam, Sastry, Askell, et~al.]{brown2020language}
Tom Brown, Benjamin Mann, Nick Ryder, Melanie Subbiah, Jared~D Kaplan, Prafulla
  Dhariwal, Arvind Neelakantan, Pranav Shyam, Girish Sastry, Amanda Askell,
  et~al.
\newblock Language models are few-shot learners.
\newblock \emph{Advances in neural information processing systems},
  33:\penalty0 1877--1901, 2020.

\bibitem[Caron et~al.(2020)Caron, Misra, Mairal, Goyal, Bojanowski, and
  Joulin]{caron2020unsupervised}
Mathilde Caron, Ishan Misra, Julien Mairal, Priya Goyal, Piotr Bojanowski, and
  Armand Joulin.
\newblock Unsupervised learning of visual features by contrasting cluster
  assignments.
\newblock \emph{Advances in Neural Information Processing Systems},
  33:\penalty0 9912--9924, 2020.

\bibitem[Chen et~al.(2022)Chen, Fu, Narayan, Zhang, Song, Fatahalian, and
  R{\'e}]{chen2022perfectly}
Mayee Chen, Daniel~Y Fu, Avanika Narayan, Michael Zhang, Zhao Song, Kayvon
  Fatahalian, and Christopher R{\'e}.
\newblock Perfectly balanced: Improving transfer and robustness of supervised
  contrastive learning.
\newblock In \emph{International Conference on Machine Learning}, pages
  3090--3122. PMLR, 2022.

\bibitem[Chen et~al.(2020{\natexlab{a}})Chen, Kornblith, Norouzi, and
  Hinton]{chen2020simple}
Ting Chen, Simon Kornblith, Mohammad Norouzi, and Geoffrey Hinton.
\newblock A simple framework for contrastive learning of visual
  representations.
\newblock In \emph{International conference on machine learning}, pages
  1597--1607. PMLR, 2020{\natexlab{a}}.

\bibitem[Chen et~al.(2020{\natexlab{b}})Chen, Kornblith, Swersky, Norouzi, and
  Hinton]{chen2020big}
Ting Chen, Simon Kornblith, Kevin Swersky, Mohammad Norouzi, and Geoffrey~E
  Hinton.
\newblock Big self-supervised models are strong semi-supervised learners.
\newblock \emph{Advances in neural information processing systems},
  33:\penalty0 22243--22255, 2020{\natexlab{b}}.

\bibitem[Chen et~al.(2021)Chen, Luo, and Li]{chen2021intriguing}
Ting Chen, Calvin Luo, and Lala Li.
\newblock Intriguing properties of contrastive losses.
\newblock \emph{Advances in Neural Information Processing Systems},
  34:\penalty0 11834--11845, 2021.

\bibitem[DeepNets(2022)]{deepnets_2022}
DeepNets.
\newblock Animals - v2: Image classification dataset, Nov 2022.
\newblock URL
  \url{https://www.kaggle.com/datasets/utkarshsaxenadn/animal-image-classification-dataset}.

\bibitem[Gao et~al.(2021)Gao, Yao, and Chen]{gao2021simcse}
Tianyu Gao, Xingcheng Yao, and Danqi Chen.
\newblock Simcse: Simple contrastive learning of sentence embeddings.
\newblock \emph{arXiv preprint arXiv:2104.08821}, 2021.

\bibitem[Garrido et~al.(2022)Garrido, Chen, Bardes, Najman, and
  Lecun]{garrido2022duality}
Quentin Garrido, Yubei Chen, Adrien Bardes, Laurent Najman, and Yann Lecun.
\newblock On the duality between contrastive and non-contrastive
  self-supervised learning.
\newblock \emph{arXiv preprint arXiv:2206.02574}, 2022.

\bibitem[Gupta et~al.(2022)Gupta, Ajanthan, Hengel, and
  Gould]{gupta2022understanding}
Kartik Gupta, Thalaiyasingam Ajanthan, Anton van~den Hengel, and Stephen Gould.
\newblock Understanding and improving the role of projection head in
  self-supervised learning.
\newblock \emph{arXiv preprint arXiv:2212.11491}, 2022.

\bibitem[Gutmann and Hyv{\"a}rinen(2010)]{gutmann2010noise}
Michael Gutmann and Aapo Hyv{\"a}rinen.
\newblock Noise-contrastive estimation: A new estimation principle for
  unnormalized statistical models.
\newblock In \emph{Proceedings of the thirteenth international conference on
  artificial intelligence and statistics}, pages 297--304. JMLR Workshop and
  Conference Proceedings, 2010.

\bibitem[HaoChen and Ma(2022)]{haochen2022theoretical}
Jeff~Z HaoChen and Tengyu Ma.
\newblock A theoretical study of inductive biases in contrastive learning.
\newblock \emph{arXiv preprint arXiv:2211.14699}, 2022.

\bibitem[HaoChen et~al.(2021)HaoChen, Wei, Gaidon, and Ma]{haochen2021provable}
Jeff~Z HaoChen, Colin Wei, Adrien Gaidon, and Tengyu Ma.
\newblock Provable guarantees for self-supervised deep learning with spectral
  contrastive loss.
\newblock \emph{Advances in Neural Information Processing Systems},
  34:\penalty0 5000--5011, 2021.

\bibitem[HaoChen et~al.(2022)HaoChen, Wei, Kumar, and Ma]{haochen2022beyond}
Jeff~Z HaoChen, Colin Wei, Ananya Kumar, and Tengyu Ma.
\newblock Beyond separability: Analyzing the linear transferability of
  contrastive representations to related subpopulations.
\newblock \emph{arXiv preprint arXiv:2204.02683}, 2022.

\bibitem[He et~al.(2020)He, Fan, Wu, Xie, and Girshick]{he2020momentum}
Kaiming He, Haoqi Fan, Yuxin Wu, Saining Xie, and Ross Girshick.
\newblock Momentum contrast for unsupervised visual representation learning.
\newblock In \emph{Proceedings of the IEEE/CVF conference on computer vision
  and pattern recognition}, pages 9729--9738, 2020.

\bibitem[Henaff(2020)]{henaff2020data}
Olivier Henaff.
\newblock Data-efficient image recognition with contrastive predictive coding.
\newblock In \emph{International conference on machine learning}, pages
  4182--4192. PMLR, 2020.

\bibitem[Hjelm et~al.(2018)Hjelm, Fedorov, Lavoie-Marchildon, Grewal, Bachman,
  Trischler, and Bengio]{hjelm2018learning}
R~Devon Hjelm, Alex Fedorov, Samuel Lavoie-Marchildon, Karan Grewal, Phil
  Bachman, Adam Trischler, and Yoshua Bengio.
\newblock Learning deep representations by mutual information estimation and
  maximization.
\newblock \emph{arXiv preprint arXiv:1808.06670}, 2018.

\bibitem[Huang et~al.(2021)Huang, Yi, and Zhao]{huang2021towards}
Weiran Huang, Mingyang Yi, and Xuyang Zhao.
\newblock Towards the generalization of contrastive self-supervised learning.
\newblock \emph{arXiv preprint arXiv:2111.00743}, 2021.

\bibitem[Ji et~al.(2021)Ji, Deng, Nakada, Zou, and Zhang]{ji2021power}
Wenlong Ji, Zhun Deng, Ryumei Nakada, James Zou, and Linjun Zhang.
\newblock The power of contrast for feature learning: A theoretical analysis.
\newblock \emph{arXiv preprint arXiv:2110.02473}, 2021.

\bibitem[Khosla et~al.(2020)Khosla, Teterwak, Wang, Sarna, Tian, Isola,
  Maschinot, Liu, and Krishnan]{khosla2020supervised}
Prannay Khosla, Piotr Teterwak, Chen Wang, Aaron Sarna, Yonglong Tian, Phillip
  Isola, Aaron Maschinot, Ce~Liu, and Dilip Krishnan.
\newblock Supervised contrastive learning.
\newblock \emph{Advances in Neural Information Processing Systems},
  33:\penalty0 18661--18673, 2020.

\bibitem[Lee et~al.(2021)Lee, Lei, Saunshi, and Zhuo]{lee2021predicting}
Jason~D Lee, Qi~Lei, Nikunj Saunshi, and Jiacheng Zhuo.
\newblock Predicting what you already know helps: Provable self-supervised
  learning.
\newblock \emph{Advances in Neural Information Processing Systems},
  34:\penalty0 309--323, 2021.

\bibitem[Li et~al.(2020)Li, Zhou, Xiong, and Hoi]{li2020prototypical}
Junnan Li, Pan Zhou, Caiming Xiong, and Steven~CH Hoi.
\newblock Prototypical contrastive learning of unsupervised representations.
\newblock \emph{arXiv preprint arXiv:2005.04966}, 2020.

\bibitem[McAllester and Stratos(2020)]{mcallester2020formal}
David McAllester and Karl Stratos.
\newblock Formal limitations on the measurement of mutual information.
\newblock In \emph{International Conference on Artificial Intelligence and
  Statistics}, pages 875--884. PMLR, 2020.

\bibitem[Misra and Maaten(2020)]{misra2020self}
Ishan Misra and Laurens van~der Maaten.
\newblock Self-supervised learning of pretext-invariant representations.
\newblock In \emph{Proceedings of the IEEE/CVF Conference on Computer Vision
  and Pattern Recognition}, pages 6707--6717, 2020.

\bibitem[Nozawa and Sato(2021)]{nozawa2021understanding}
Kento Nozawa and Issei Sato.
\newblock Understanding negative samples in instance discriminative
  self-supervised representation learning.
\newblock \emph{Advances in Neural Information Processing Systems},
  34:\penalty0 5784--5797, 2021.

\bibitem[Oord et~al.(2018)Oord, Li, and Vinyals]{oord2018representation}
Aaron van~den Oord, Yazhe Li, and Oriol Vinyals.
\newblock Representation learning with contrastive predictive coding.
\newblock \emph{arXiv preprint arXiv:1807.03748}, 2018.

\bibitem[Peters et~al.(2018)Peters, Neumann, Iyyer, Gardner, Clark, Lee, and
  Zettlemoyer]{DBLP:journals/corr/abs-1802-05365}
Matthew~E. Peters, Mark Neumann, Mohit Iyyer, Matt Gardner, Christopher Clark,
  Kenton Lee, and Luke Zettlemoyer.
\newblock Deep contextualized word representations.
\newblock \emph{CoRR}, abs/1802.05365, 2018.
\newblock URL \url{http://arxiv.org/abs/1802.05365}.

\bibitem[Radford et~al.(2019)Radford, Wu, Child, Luan, Amodei, Sutskever,
  et~al.]{radford2019language}
Alec Radford, Jeffrey Wu, Rewon Child, David Luan, Dario Amodei, Ilya
  Sutskever, et~al.
\newblock Language models are unsupervised multitask learners.
\newblock \emph{OpenAI blog}, 1\penalty0 (8):\penalty0 9, 2019.

\bibitem[Radford et~al.(2021)Radford, Kim, Hallacy, Ramesh, Goh, Agarwal,
  Sastry, Askell, Mishkin, Clark, et~al.]{radford2021learning}
Alec Radford, Jong~Wook Kim, Chris Hallacy, Aditya Ramesh, Gabriel Goh,
  Sandhini Agarwal, Girish Sastry, Amanda Askell, Pamela Mishkin, Jack Clark,
  et~al.
\newblock Learning transferable visual models from natural language
  supervision.
\newblock In \emph{International conference on machine learning}, pages
  8748--8763. PMLR, 2021.

\bibitem[Robinson et~al.(2020)Robinson, Chuang, Sra, and
  Jegelka]{robinson2020contrastive}
Joshua Robinson, Ching-Yao Chuang, Suvrit Sra, and Stefanie Jegelka.
\newblock Contrastive learning with hard negative samples.
\newblock \emph{arXiv preprint arXiv:2010.04592}, 2020.

\bibitem[Saunshi et~al.(2022)Saunshi, Ash, Goel, Misra, Zhang, Arora, Kakade,
  and Krishnamurthy]{saunshi2022understanding}
Nikunj Saunshi, Jordan Ash, Surbhi Goel, Dipendra Misra, Cyril Zhang, Sanjeev
  Arora, Sham Kakade, and Akshay Krishnamurthy.
\newblock Understanding contrastive learning requires incorporating inductive
  biases.
\newblock \emph{arXiv preprint arXiv:2202.14037}, 2022.

\bibitem[Shen et~al.(2022)Shen, Jones, Kumar, Xie, HaoChen, Ma, and
  Liang]{shen2022connect}
Kendrick Shen, Robbie~M Jones, Ananya Kumar, Sang~Michael Xie, Jeff~Z HaoChen,
  Tengyu Ma, and Percy Liang.
\newblock Connect, not collapse: Explaining contrastive learning for
  unsupervised domain adaptation.
\newblock In \emph{International Conference on Machine Learning}, pages
  19847--19878. PMLR, 2022.

\bibitem[Su et~al.(2021)Su, Liu, Meng, Lan, Shu, Shareghi, and
  Collier]{su2021tacl}
Yixuan Su, Fangyu Liu, Zaiqiao Meng, Tian Lan, Lei Shu, Ehsan Shareghi, and
  Nigel Collier.
\newblock Tacl: Improving bert pre-training with token-aware contrastive
  learning.
\newblock \emph{arXiv preprint arXiv:2111.04198}, 2021.

\bibitem[Thomson(1904)]{thomson1904xxiv}
Joseph~John Thomson.
\newblock Xxiv. on the structure of the atom: an investigation of the stability
  and periods of oscillation of a number of corpuscles arranged at equal
  intervals around the circumference of a circle; with application of the
  results to the theory of atomic structure.
\newblock \emph{The London, Edinburgh, and Dublin Philosophical Magazine and
  Journal of Science}, 7\penalty0 (39):\penalty0 237--265, 1904.

\bibitem[Tian et~al.(2020{\natexlab{a}})Tian, Krishnan, and
  Isola]{tian2020contrastive}
Yonglong Tian, Dilip Krishnan, and Phillip Isola.
\newblock Contrastive multiview coding.
\newblock In \emph{European conference on computer vision}, pages 776--794.
  Springer, 2020{\natexlab{a}}.

\bibitem[Tian et~al.(2020{\natexlab{b}})Tian, Sun, Poole, Krishnan, Schmid, and
  Isola]{tian2020makes}
Yonglong Tian, Chen Sun, Ben Poole, Dilip Krishnan, Cordelia Schmid, and
  Phillip Isola.
\newblock What makes for good views for contrastive learning?
\newblock \emph{Advances in Neural Information Processing Systems},
  33:\penalty0 6827--6839, 2020{\natexlab{b}}.

\bibitem[Tian(2022{\natexlab{a}})]{tian2022deep}
Yuandong Tian.
\newblock Deep contrastive learning is provably (almost) principal component
  analysis.
\newblock \emph{arXiv preprint arXiv:2201.12680}, 2022{\natexlab{a}}.

\bibitem[Tian(2022{\natexlab{b}})]{tian2022understanding}
Yuandong Tian.
\newblock Understanding the role of nonlinearity in training dynamics of
  contrastive learning.
\newblock \emph{arXiv preprint arXiv:2206.01342}, 2022{\natexlab{b}}.

\bibitem[Tosh et~al.(2021)Tosh, Krishnamurthy, and Hsu]{tosh2021contrastive}
Christopher Tosh, Akshay Krishnamurthy, and Daniel Hsu.
\newblock Contrastive learning, multi-view redundancy, and linear models.
\newblock In \emph{Algorithmic Learning Theory}, pages 1179--1206. PMLR, 2021.

\bibitem[Tschannen et~al.(2019)Tschannen, Djolonga, Rubenstein, Gelly, and
  Lucic]{tschannen2019mutual}
Michael Tschannen, Josip Djolonga, Paul~K Rubenstein, Sylvain Gelly, and Mario
  Lucic.
\newblock On mutual information maximization for representation learning.
\newblock \emph{arXiv preprint arXiv:1907.13625}, 2019.

\bibitem[Von~K{\"u}gelgen et~al.(2021)Von~K{\"u}gelgen, Sharma, Gresele,
  Brendel, Sch{\"o}lkopf, Besserve, and Locatello]{von2021self}
Julius Von~K{\"u}gelgen, Yash Sharma, Luigi Gresele, Wieland Brendel, Bernhard
  Sch{\"o}lkopf, Michel Besserve, and Francesco Locatello.
\newblock Self-supervised learning with data augmentations provably isolates
  content from style.
\newblock \emph{Advances in neural information processing systems},
  34:\penalty0 16451--16467, 2021.

\bibitem[Wang and Liu(2021)]{wang2021understanding}
Feng Wang and Huaping Liu.
\newblock Understanding the behaviour of contrastive loss.
\newblock In \emph{Proceedings of the IEEE/CVF conference on computer vision
  and pattern recognition}, pages 2495--2504, 2021.

\bibitem[Wang and Isola(2020)]{wang2020understanding}
Tongzhou Wang and Phillip Isola.
\newblock Understanding contrastive representation learning through alignment
  and uniformity on the hypersphere.
\newblock In \emph{International Conference on Machine Learning}, pages
  9929--9939. PMLR, 2020.

\bibitem[Wang et~al.(2022)Wang, Zhang, Wang, Yang, and Lin]{wang2022chaos}
Yifei Wang, Qi~Zhang, Yisen Wang, Jiansheng Yang, and Zhouchen Lin.
\newblock Chaos is a ladder: A new theoretical understanding of contrastive
  learning via augmentation overlap.
\newblock \emph{arXiv preprint arXiv:2203.13457}, 2022.

\bibitem[Wei et~al.(2020)Wei, Shen, Chen, and Ma]{wei2020theoretical}
Colin Wei, Kendrick Shen, Yining Chen, and Tengyu Ma.
\newblock Theoretical analysis of self-training with deep networks on unlabeled
  data.
\newblock \emph{arXiv preprint arXiv:2010.03622}, 2020.

\bibitem[Wen and Li(2021)]{wen2021toward}
Zixin Wen and Yuanzhi Li.
\newblock Toward understanding the feature learning process of self-supervised
  contrastive learning.
\newblock In \emph{International Conference on Machine Learning}, pages
  11112--11122. PMLR, 2021.

\bibitem[Wen and Li(2022)]{wen2022mechanism}
Zixin Wen and Yuanzhi Li.
\newblock The mechanism of prediction head in non-contrastive self-supervised
  learning.
\newblock \emph{arXiv preprint arXiv:2205.06226}, 2022.

\bibitem[Zheng et~al.(2021)Zheng, Wang, You, Qian, Zhang, Wang, and
  Xu]{zheng2021weakly}
Mingkai Zheng, Fei Wang, Shan You, Chen Qian, Changshui Zhang, Xiaogang Wang,
  and Chang Xu.
\newblock Weakly supervised contrastive learning.
\newblock In \emph{Proceedings of the IEEE/CVF International Conference on
  Computer Vision}, pages 10042--10051, 2021.

\bibitem[Zimmermann et~al.(2021)Zimmermann, Sharma, Schneider, Bethge, and
  Brendel]{zimmermann2021contrastive}
Roland~S Zimmermann, Yash Sharma, Steffen Schneider, Matthias Bethge, and
  Wieland Brendel.
\newblock Contrastive learning inverts the data generating process.
\newblock In \emph{International Conference on Machine Learning}, pages
  12979--12990. PMLR, 2021.

\end{thebibliography}
\bibliographystyle{plainnat}
\newpage
\appendix

\newpage

\section{Proof of Theorem \ref{thm:uniformandclean}}
\label{appendix:realizable}

\noindent We prove Theorem~\ref{thm:uniformandclean}. To do so, we first prove Lemma~\ref{thm:uniform}, and use this to prove our main result.

\subsection{Proof of Lemma \ref{thm:uniform}}
\label{sec:lemma-prf}

\begin{lemma}[Lemma \ref{thm:uniform} Restated] \label{thm:uniform-app}
     If Assumptions \ref{assump:spikyaug} and \ref{assump:realizability} hold, $\beta > c \log d$ for an absolute constant $c$, $d>3$, and $\ell\geq 1$, then $g^*\in \underset{g\in \mathcal{G}_c}{\arg\min}  \ {\mathcal{L}}(g) $ if and only if $g^*$ is uniform.
\end{lemma}

\begin{proof} 
To prove this claim we first note that since we are optimizing over $\mathcal{G}_c$ and $g(x)=g(x^+)$ for all $x\in D_\circ,x^+\in A(x) $ and $g \in \mathcal{G}_c$, the optimization problem $\min_{g\in\mathcal{G}_c} \mathcal{L}(g)$ is equivalent to
\begin{equation}\min_{g \in \mathcal{G}_c}  \hat{\mathcal{L}}(g) = \mathbb{E}_{x,\{x_i^-\}_{\ell}}\left[  \log\left(1+ \sum_{i= 1}^\ell e^{\beta g(x)^\top g(x_{i}^-) - \beta d} \right) \right].
\end{equation}
Note that we can think of this as an optimization over distributions over the hypercube induced by $g$. That is, consider the random variable $Y$ supported on the hypercube that is given by $Y = g(X)$ for $X\sim D_\circ$, and denote its distribution as $\mathcal{D}_g$. Using this notation, the optimization above can be rewritten in terms of distributions
\begin{align*}
\min_{\mathcal{D}\in \{\mathcal{D}_g: g\in \G\}} \tilde{\mathcal{L}}(\mathcal{D}) = \mathbb{E}_{y,\{\bar{y}_i\}_{\ell} \sim \mathcal{D}}\left[  \log(1+ \sum_{i= 1}^\ell e^{\beta y^\top \bar{y}_{i} - \beta d} ) \right].
\end{align*}
where $\bar{y}_{i}$ corresponds to the representation of the $i$-th negative sample, and here we overload notation by using $ \sim \mathcal{D}$ to denote an i.i.d. draw from the distribution $\mathcal{D}$.

Next, we define a Markov chain as follows. We begin with a fresh set of samples denoted by $y^0, \bar{y}_1^0, \dots, \bar{y}_\ell^0$ that are drawn i.i.d. from the distribution $\mathcal{D}^0 = \mathcal{D}$. At each step, for each sample, we either with probability $\frac{1}{2}$ flip one bit uniformly at random, or with probability $\frac{1}{2}$ we do not change it. Concretely, we take for all $i$ a random variable $j_{i, t}\in [d]$ (both uniformly random and independent of each other and every other such sample) and set $(\bar{y}^{t}_i)_{j_{i, t}} = -(\bar{y}^{t-1}_i)_{j_{i, t}}$. After this operation, each $y_i^t$ (and $y$) can be considered to be an i.i.d. drawn from $\mathcal{D}^t$, where $\mathcal{D}^t$ is another distribution over $\mathcal{H}_d$. We show that $\mathcal{L}(\mathcal{D}^{t-1})>\mathcal{L}(\mathcal{D}^{t})$ if $\mathcal{D}^{t-1}$ is not uniform. Since $\{\mathcal{D}^{t}\}_t$ converges to the uniform distribution by Lemma~\ref{lem:stationary}, these two arguments together imply the claim of Lemma~\ref{thm:uniform}.
% . Since the joint distribution of $y^t, \{\bar{y}^t_i\}_\ell$ converges to the joint uniform distribution, we are done. 

To show $\mathcal{L}(\mathcal{D}^{t-1})>\mathcal{L}(\mathcal{D}^{t})$ for the case that $\mathcal{D}^{t-1}$ is not the uniform distribution, considering the definition of  $\mathcal{L}$ we need to study the variation in the inner products between the vectors $(y^t, \bar{y}^t_i)$ when we move from one distribution to another. Note that as these vectors are binary vectors, their inner product can be written as a function of their Hamming distances. More precisely, for any pair $(y, y')$ we have $y^\top y'=d- 2  h(y,y')$, where the Hamming distance between them is defined as $h(y,y')\coloneqq \sum_{j=1}^\ell \chi\{y_j\neq y'_j\}$ or the number of bits that are different in the two points $y,y'$ (note that $\chi\{U\}$ is the indicator variable for the event $U$).

%we will argue in terms of Hamming distances between negative pairs $(y^t, \bar{y}^t_i)$. For two points $y,y'\in\mathcal{H}_d$, . . Note that $2h(y,y') = {d-y^\top y'}{}$.
% consider the induced evolution of each of the exponents in the loss: $(y^{t-1})^\top \bar{y}^{t-1}_i =  d-2h(y^{t-1}, \bar{y}_i^{t-1}) $ to $(y^t)^\top \bar{y}^{t}_i = d-h(y^{t}, \bar{y}_i^{t})$, where here , or the number of bits that are different in the two points $y,y'$. 
For ease of notation we let $h_i^t\coloneqq h(y^t,\bar{y}_i^t)$ for all $i,t$. 
Due to the fact that each of the $\bar{y}_i^t$'s are independent and identically distributed and are evolving according to a Markov chain, the $h_i^t$'s also evolve according to a Markov chain. In particular, for every distribution $\mathcal{D}^t$ over $\uH_d$ that describes the distribution of each $\bar{y}_i^t$, there is induced a distribution $\tilde{\mathcal{D}}_h^t$ over $[d]$ that specifies the distribution for $h_i^t$. By direct computation, one can check that $h_i^t$ has the following transition kernels which differ for different values of $h_i$:
% the Markov chain over $h_i$ induced by the Markov chain over $y, \{\bar{y}_i\}_\ell$ has the following transition kernel:

\vspace{2mm}
For $2\le h_i^{t-1} \le d-2:$
$$h_i^{t}\to 
\begin{cases}
h_i^{t-1}-2 & \text{w.p.}  ~\frac{(h_i^{t-1})(h_i^{t-1}-1)}{4d^2}\\
h_i^{t-1}-1 & \text{w.p.}  ~\frac{h_i^{t-1}}{2d}\\
h_i^{t-1} & \text{w.p.}  ~\frac{1}{4}+\frac{(h_i^{t-1})(d-h_i^{t-1}+1)+(d-h_i^{t-1})(h_i^{t-1}+1)}{4d^2}\\
h_i^{t-1}+1 & \text{w.p.}  ~\frac{d-h_i^{t-1}}{2d}\\
h_i^{t-1}+2 & \text{w.p.} ~\frac{(d-h_i^{t-1})(d-h_i^{t-1}+1)}{4d^2}
\end{cases}$$

For $h_i^{t-1} = 1:$
\begin{equation}\label{h_one}
1\to 
\begin{cases}
0 & \text{w.p.}  ~\frac{1}{2d}\\
1 & \text{w.p.}  ~\frac{1}{4}+\frac{1}{4d}+\frac{2(d-1)}{4d^2}\\
2 & \text{w.p.}  ~\frac{d-1}{2d}\\
3 & \text{w.p.}  ~\frac{(d-1)(d-2)}{4d^2}\end{cases}
\end{equation}

For $h_i^{t-1} = d-1:$
$$1\to 
\begin{cases}
d & \text{w.p.}  ~\frac{1}{2d}\\
d-1 & \text{w.p.}  ~\frac{1}{4}+\frac{1}{4d}+\frac{2(d-1)}{4d^2}\\
d-2 & \text{w.p.}  ~\frac{d-1}{2d}\\
d-3 & \text{w.p.}  ~\frac{(d-1)(d-2)}{4d^2}\end{cases}$$

For $h_i^{t-1} = 0:$
\begin{equation}\label{h_zero}
0\to 
\begin{cases}
0 & \text{w.p.}  \frac{1}{4}+\frac{1}{4d}\\
1 & \text{w.p.}  \frac{1}{2}\\
2 & \text{w.p.}  \frac{d-1}{4d}
\end{cases}
\end{equation}

For $h_i^{t-1} = d:$
$$d\to 
\begin{cases}
d & \text{w.p.}  \frac{1}{4}+\frac{1}{4d}\\
d-1 & \text{w.p.}  \frac{1}{2}\\
d-2 & \text{w.p.}  \frac{d-1}{4d}
\end{cases}
$$

For ease of notation we drop the $t$ superscripts and refer to quantities at time $t-1$ without any superscript, and quantities at times $t$ with a $'$ superscript, e.g. $\mathcal{D}^{t-1}$  as $\mathcal{D}$ and  $\mathcal{D}^{t}$  as $\mathcal{D}'$, for the remainder of the proof.

% -------------------

% Next, consider any

% A distribution $\mathcal{D}$ over $y$ induces a distribution over the Hamming distance between two randomly sampled $y_1, y_2\sim \mathcal{D}$. We will denote this distribution as $\tilde{\mathcal{D}}_h$. 
Next, let 
$\textbf{h} = [h_1,\dots,h_\ell]$ denote the vector of concatenated Hamming distances between $y$ and $\{\bar{y}_i\}_{i=1}^\ell$. Using the above definitions, and the fact that $d-y^\top \bar{y}_i$ is twice of the hamming distance between $y$ and $ \bar{y}_i$ the loss $\Tilde{\mathcal{L}}$ can be written as 
% For any distribution $\mathcal{D}$ on $\mathcal{H}_d$ from which $y$ and $\{\bar{y}_i\}_{i=1}^\ell$ are drawn i.i.d., we ,
\begin{align*}
\Tilde{\mathcal{L}}(\mathcal{D}) &= \E_{y, \{\bar{y}_i\}_\ell \sim \mathcal{D}} \left[ \log \left(1+\sum_{i=1}^\ell e^{\beta (y^\top \bar{y}_i-d)}\right)\right] = \E_{\textbf{h}\sim \tilde{\mathcal{D}}_h} \left[ \log \left(1+\sum_{i=1}^l e^{-2\beta {h}_i}\right)\right]
\end{align*}
Now to characterize the difference between $\Tilde{\mathcal{L}}(\mathcal{D})$ and $\Tilde{\mathcal{L}}(\mathcal{D}')$ we need to study the evolution of the distribution of the Hamming distance $\textbf{h}$ from $\tilde{\mathcal{D}}_h$ to $\tilde{\mathcal{D}}'_h$, i.e., 
$$
\Tilde{\mathcal{L}}(\mathcal{D}')-\Tilde{\mathcal{L}}(\mathcal{D})
=\E_{\textbf{h}'\sim \tilde{\mathcal{D}}'_h} \left[ \log \left(1+\sum_{i=1}^l e^{-2\beta {h}'_i}\right)\right]
-\E_{\textbf{h}\sim \tilde{\mathcal{D}}_h} \left[ \log \left(1+\sum_{i=1}^l e^{-2\beta {h}_i}\right)\right]
$$

For each $i$, consider the random variable $s_i= h'_i-h_i$ that indicates which of the transitions is undertaken by ${h}_i$, and let $\textbf{s}:=[s_1,\dots,s_\ell]$ be its concatenation. Note that each $s_i$ takes values in $\{-2,-1,0,1,2\}$ and its distribution depends on the value of $h_i$, defined according to the transition kernel of $h_i$ defined above. Given this, we can express $ \Tilde{\mathcal{L}}(\mathcal{D}') $ as
\begin{equation*}
 \Tilde{\mathcal{L}}(\mathcal{D}') = \E_{\textbf{h}\sim \tilde{\mathcal{D}}_h , \textbf{s}} \left[ \log 1+ \sum_{i=1}^\ell e^{-2\beta ({h}_i+s_i))}\right]
\end{equation*}

Now we consider the difference $\Tilde{\mathcal{L}}(\mathcal{D}')-\Tilde{\mathcal{L}}(\mathcal{D})$. According to the above definitions, this difference can be written as 
\begin{align}
&\Tilde{\mathcal{L}}(\mathcal{D}')-\Tilde{\mathcal{L}}(\mathcal{D}) \nonumber \\
&\quad= \E_{\textbf{h}\sim \tilde{\mathcal{D}}_h , \textbf{s}} \left[ \log \left(1+\sum_{i=1}^l e^{-2\beta({h}_i+{s}_i)}\right) -  \log \left(1+\sum_{i=1}^l e^{-2\beta{h}_i}\right)\right] \nonumber \\
&\quad=\sum_{i=1}^l \E_{\textbf{h}\sim \tilde{\mathcal{D}}_h , \textbf{s}} \bigg[ \log \bigg(1+\sum_{j=1}^i e^{-\beta ({h}_j+{s}_j)}+\sum_{j=i+1}^l e^{-2\beta {h}_j}\bigg) \nonumber \\
&\quad \quad \quad \quad \quad \quad  \quad \quad \quad \quad -  \log \bigg(1+\sum_{j=1}^{i-1} e^{-2\beta ({h}_j+{s}_j)}+\sum_{j=i}^l e^{-2\beta {h}_j}\bigg)\bigg] \label{fgf}
\end{align}
where \eqref{fgf} follows from telescoping over each negative sample indexed by $i$. Now if we take out the $i$-th term of each of the above two expressions the difference can be written as 
\begin{align*}
&f(\mathcal{D}')-f(\mathcal{D})\\
&\quad=\sum_{i=1}^l \E_{\textbf{h}\sim \tilde{\mathcal{D}}_h , \textbf{s}} \bigg[ \log \bigg(1+ \sum_{j=1}^{i-1} e^{-2\beta ({h}_j+{s}_j)} + \sum_{j=i+1}^{\ell} e^{-2\beta {h}_j} +e^{-2\beta ({h}_i+{s}_i)}\bigg) \\
&\quad \quad \quad \quad\quad \quad \quad \quad \quad  -  \log \bigg(1+ \sum_{j=1}^{i-1} e^{-2\beta {h}_j} + \sum_{j=i+1}^{\ell} e^{-2\beta {h}_j} +e^{-2\beta {h}_i}\bigg)\bigg]\\
% &\quad=\sum_{i=1}^l \E_{\textbf{h}\sim \tilde{\mathcal{D}}_h} \big[ \log \big(\sum_{j\ne i} e^{-2\beta (\textbf{h}_i+\mathbbm{1}\{j > i\}\textbf{s}_i)}+e^{-2\beta (\textbf{h}_i+\textbf{s}_i)}\big) -  \log \big(\sum_{j\ne i} e^{-2\beta (\textbf{h}_i+\mathbbm{1}\{j > i\}\textbf{s}_i)}+e^{-2\beta (\textbf{h}_i)}\big)\big]\\
&\quad=\sum_{i=1}^l \E_{{h}_{-i}, {s}_i}\E_{{h}_i\sim \tilde{\mathcal{D}}_h, {s}_i}  \bigg[ \log \bigg(1+ \sum_{j=1}^{i-1} e^{-2\beta ({h}_j+{s}_j)} + \sum_{j=i+1}^{\ell} e^{-2\beta {h}_j} +e^{-2\beta ({h}_i+\textbf{s}_i)}\bigg) \\
&\quad \quad \quad \quad\quad \quad \quad \quad \quad  -  \log \bigg(1+ \sum_{j=1}^{i-1} e^{-2\beta ({h}_j+{s}_j)} + \sum_{j=i+1}^{\ell} e^{-2\beta {h}_j} +e^{-2\beta {h}_i}\bigg)\bigg] 
\end{align*}
To show that the RHS above is strictly less than 0 when $\mathcal{D}$ is not uniform, it is sufficient to show that each term of the sum is strictly less than zero. To do this, we show that the inner expectation is strictly negative.
% for all instances of 
% where the last inequality follows by applying Lemma \ref{lem:15} with 
% 
% % We will show that each of the terms in the summation are negative individually. It would suffice to show the following, where below $C$ is a proxy for $1+ \sum_{j=1}^{i-1} e^{-2\beta ({h}_j+{s}_j)} + \sum_{j=i+1}^{\ell} e^{-2\beta {h}_j}$. 
% \begin{lemma}\label{lem:15} Suppose $\Tilde{\mathcal{D}}_h$ is induced by some distribution on $\mathcal{H}_d$ that is not uniform. Then if $\beta \ge 2\log d$ and $ d > 3$, 
% $\E_{{h}_i\sim \tilde{\mathcal{D}}_h, {s}_i} \big[ \log \big(C+e^{-2\beta ({h}_i+{s}_i)}\big) -  \log \big(C+e^{-2\beta {h}_i}\big)\big] < 0$ for all $C>1$. 
% \end{lemma}
% \begin{proof}
In other words, all that remains to prove is that for all instances of  $C_{-i} \coloneqq 1+ \sum_{j=1}^{i-1} e^{-2\beta ({h}_j+{s}_j)} + \sum_{j=i+1}^{\ell} e^{-2\beta {h}_j}$, each of the terms satisfies
\begin{align}
&\E_{{h}_i\sim \tilde{\mathcal{D}}_h, {s}_i}  \bigg[ \log \bigg(C_{-i}+e^{-2\beta ({h}_i+\textbf{s}_i)}\bigg)   -  \log \bigg(C_{-i}+e^{-2\beta {h}_i}\bigg)\bigg] < 0 \label{defc}
\end{align}
when $\mathcal{D}$ is not uniform. Once we have this result, the claim that $f(\mathcal{D}') - f(\mathcal{D})< 0$
% $\E_{{h}_i\sim \tilde{\mathcal{D}}_h, {s}_i}<0$ 
holds. 

From now on, for ease of notation, we replace $C_{-i}$ by $C$. To prove the claim in \eqref{defc}, we first introduce the function $\Delta_C(h)$ defined as 
% consider the expression 
\begin{align*}\Delta_C(h)&:=\E_{{s}_i}\big[\log \big(C+e^{-2\beta (h+{s}_i)}\big) -  \log \big(C+e^{-2\beta h}\big)\big]\\
&=\E_{{s}_i}\big[\log \big(1+\frac{e^{-2\beta (h+{s}_i)}}{C}\big) -  \log \big(1+\frac{e^{-2\beta h}}{C}\big)\big]. \end{align*}
% as a function of $h$. 
Considering this definition the claim in \eqref{defc} can be translated into 
\begin{equation}
    \sum_{k=0}^d\Pr_{h\sim \tilde{\mathcal{D}}_h}[h = k] \Delta_C(h) < 0.
\end{equation}

To show this, it is easiest to compare this expression with the analogous expression in the case that $\mathcal{D}$ is uniform. That is $y,\{\bar{y}_i\}_{i=1}^\ell$ are drawn from the uniform distribution $\mathcal{U}$. We let $\tilde{\mathcal{U}}_h$ denote the distribution on $[d]$ of Hamming distances induced by $\mathcal{U}$. By the stationarity of the uniform distribution (Lemma \ref{lem:stationary}), the distribution of $h_i'= h_i+s_i$ is identical to that of $h_i$ if $h_i$ is drawn from a uniform distribution. Thus we have the following result:
\begin{align}
&\E_{{h}_i\sim \tilde{\mathcal{U}}_h, {s}_i}  \bigg[ \log \bigg(C_{-i}+e^{-2\beta ({h}_i+\textbf{s}_i)}\bigg)   -  \log \bigg(C_{-i}+e^{-2\beta {h}_i}\bigg)\bigg] \nonumber \\
&=\E_{{h_i}'\sim \tilde{\mathcal{U}}_h'}  \bigg[ \log \bigg(C_{-i}+e^{-2\beta ({h}'_i)}\bigg) \bigg]  -\E_{{h_i}\sim \tilde{\mathcal{U}}_h} \bigg[ \log \bigg(C_{-i}+e^{-2\beta {h}_i}\bigg)\bigg]\nonumber \\
&=0, \label{eq:ww}
\end{align}

Next, we show that  $\sum_{k=0}^d\Pr_{h\sim \tilde{\mathcal{D}}_h}[h = k] \Delta_C(h) -\sum_{k=0}^d\Pr_{h\sim \tilde{\mathcal{U}}_h}[h = k] \Delta_C(h) < 0$, which by \eqref{eq:ww} immediately implies $\sum_{k=0}^d\Pr_{h\sim \tilde{\mathcal{D}}_h}[h = k] \Delta_C(h)  < 0$. To achieve this we 
invoke Lemmas \ref{lem:16} and \ref{lem:17}, which describe the behavior of $\Pr_{h\sim \tilde{\mathcal{D}}_h}[h = k] \Delta_C(h) -\sum_{k=0}^d\Pr_{h\sim \tilde{\mathcal{U}}_h}[h = k] \Delta_C(h)$ and $\Delta_C(k)$, respectively. Using these lemmas we obtain:
% \begin{align*}
% &\sum_{h=k}\Pr_{\tilde{\mathcal{D}}_h}[h=k]\Delta(k) - \sum_{h=k}\Pr_{\U}[h=k]\Delta(k)\\
% &\quad\le \big(\Pr_{\tilde{\mathcal{D}}_h}[h = 0]-\Pr_{\U}[h=0]\big)\Delta(0)+\sum_{k>0}\binom{d}{k}\big(\Pr_{\tilde{\mathcal{D}}_h}[h=0]-\Pr_{\U}[h=0]\big)\Delta(k)\\
% &\quad\le \big(\Pr_{\tilde{\mathcal{D}}_h}[h = 0]-\Pr_{\U}[h=0]\big)\sum_{k\ge0}\binom{d}{k}\Delta(k)
% \end{align*}
% Now consider the terms $\Pr_{\tilde{\mathcal{D}}_h}[h=k]-\Pr_{\U}[h=k]$. For $v\in \uH_d$, let $p_v = \Pr_{D}(x = i)$. We have 
% \begin{align*}
% \Pr_{\tilde{\mathcal{D}}_h}[h=0]-\Pr_{\U}[h=0] &= \sum_{v\in \uH_d} (p_v-\frac{1}{2^d})^2\ge 0\\
% \Pr_{\tilde{\mathcal{D}}_h}[h=k]-\Pr_{\U}[h=k] &= \big(\Pr_{\tilde{\mathcal{D}}_h}[h=k]-\binom{d}{k}\Pr_{\tilde{\mathcal{D}}_h}[h=0]\big)\\
% &\quad-\big(\Pr_{\U}[h=k]-\binom{d}{k}\Pr_{\U}[h=0]\big)+\binom{d}{k}\big(\Pr_{\tilde{\mathcal{D}}_h}[h=0]-\Pr_{\U}[h=0]\big)\\
% &= \big(\Pr_{\tilde{\mathcal{D}}_h}[h=k]-\binom{d}{k}\Pr_{\tilde{\mathcal{D}}_h}[h=0]\big)+\binom{d}{k}\big(\Pr_{\tilde{\mathcal{D}}_h}[h=0]-\Pr_{\U}[h=0]\big)\\
% &= -\sum_{i, j\in \uH_d, h(i, j) = k} (p_v-p_j)^2+\binom{d}{k}\sum_{v\in \uH_d} (p_v-\frac{1}{2^d})^2\le 0
% \end{align*}
% \end{proof}
{
\begin{align}
&\sum_{k=0}^d\Pr_{h\sim \tilde{\mathcal{D}}_h}[h=k]\Delta_C(k) - \sum_{k=0}^d\Pr_{h\sim \tilde{\mathcal{\U}}_h} [h=k]\Delta_C(k) \nonumber \\ %\sum_{k=0}^d\Pr_{Y_t\sim \U}[h(y_0,Y_t)=k]\Delta_C(k) \nonumber \\
&\quad\le \big(\Pr_{h\sim\tilde{\mathcal{D}}_h}[h = 0]-\Pr_{h\sim \tilde{\mathcal{\U}}_h}[h=0]\big)\Delta_C(0)\nonumber \\
&\quad \quad +\sum_{k>0}\binom{d}{k}\big(\Pr_{h\sim \tilde{\mathcal{D}}_h}[h=0]-\Pr_{h\sim \tilde{\mathcal{\U}}_h}[h=0]\big)\Delta_C(k) \label{eq:eq}\\
&\quad= \big(\Pr_{h\sim\tilde{\mathcal{D}}_h}[h = 0]-\Pr_{h\sim \tilde{\mathcal{\U}}_h}[h=0]\big)\sum_{k\ge0}\binom{d}{k}\Delta_C(k) \label{eq:333}\\
&\quad < 0 \label{eq:0}
\end{align}
where \eqref{eq:eq} holds by Lemma \ref{lem:16}-3, \eqref{eq:333} follows by simply combining terms, and \eqref{eq:0} holds by Lemma \ref{lem:16}-2, which states that $\Pr_{h\sim\tilde{\mathcal{D}}_h}[h = 0]-\Pr_{h\sim \tilde{\mathcal{\U}}_h}[h=0] > 0$, and Lemma
% (strict for $\tilde{\mathcal{D}}_h$ induced by a non-uniform distribution) 
\ref{lem:17}, which states that $\sum_{k\ge0}\binom{d}{k}\Delta_C(k) < 0$.
% (strict for $\tilde{\mathcal{D}}_h$ induced by a non-uniform distribution), 
This completes  the proof.
}
\end{proof}
% \end{proof}

% We now need the following properties of distributions on the hypercube:
\begin{lemma} \label{lem:16}
For any the distribution $\tilde{\mathcal{D}}_h$ on $[d]$ induced by any non-uniform distribution $\mathcal{D}$ on $\mathcal{H}_d$, the following are true:
\begin{enumerate}
\item $\Pr_{h\sim \tilde{\mathcal{\U}}_h}[h=k] = \binom{d}{k}\Pr_{h\sim \tilde{\mathcal{\U}}_h}[h=0]$
\item $\Pr_{h\sim \tilde{\mathcal{D}}_h}[h = 0] - \Pr_{h\sim \tilde{\mathcal{\U}}_h}[h=0] > 0$ 
% with equality iff ${\mathcal{D}}= \U$ %{\color{blue}$\tilde{\mathcal{D}}_h = \U$ uniform on $[d]$? or ${\mathcal{D}}= \U$}.
% \item $\Pr_{\tilde{\mathcal{D}}_h}[h = k] - \Pr_{\U}[h = k]\ge -\binom{d}{k}\big(\Pr_{\tilde{\mathcal{D}}_h}[h=0]-\Pr_{\U}[h=0]\big)$
\item $\Pr_{h\sim \tilde{\mathcal{D}}_h}[h = k] - \Pr_{Y\sim \U}[h=k]< \binom{d}{k}\big(\Pr_{h\sim \tilde{\mathcal{D}}_h}[h=0]-\Pr_{Y\sim \U}[h(y,Y)=k]\big)$
\end{enumerate}
\end{lemma}
\begin{proof}
 For any vertex $v\in \mathcal{H}_d$, we let $p_v := \mathbb{P}_{y\sim \mathcal{D}}[y = v]$ for ease of notation.
 
\begin{enumerate}
\item Note that for any $k\in [d]$, 
\begin{align}
    \Pr_{h\sim \tilde{\mathcal{\U}}_h}[h=k] = \Pr_{y,y'\sim {\mathcal{\U}}}[h(y,y')=k] = \Pr_{y\sim {\mathcal{\U}}}[h(w,y)=k] 
\end{align}
for any fixed vertex $w\in \mathcal{H}_d$ by the symmetry of the uniform distribution.
Then 
% --------------------
% \begin{align*}
% \Pr_{Y\sim \U}[h(y,Y)=k] &= \sum_{v, u\in \uH: h(v, u) = k} p_vp_u\\
% & = \sum_{v\in \uH_d} \quad \sum_{u\in \uH_d : h(v, u)= k}p_vp_u = d\binom{d}{k}\frac{1}{2^{2d}}
% \end{align*} while $\Pr_{Y\sim \U}[h(y,Y)=0] = \sum_{v\in \uH_d} p_v^2 = d\frac{1}{2^{2d}}$.
{
\begin{align*}
\Pr_{h\sim \tilde{\mathcal{\U}}_h}[h=k] &= \Pr_{y\sim {\mathcal{\U}}}[h(w,y)=k]
\\                
&= \sum_{v\in \uH_d: h(w, v) = k} \Pr_{y\sim {\mathcal{\U}}}[y=v] \\
&= \sum_{v\in \uH_d: h(w, v) = k} \frac{1}{2^d} \\
% & = \sum_{v\in \uH_d} \quad \sum_{u\in \uH_d : h(v, u)= k}  p_u \\
&= \binom{d}{k}\frac{1}{2^{d}},
\end{align*} which implies that $\Pr_{h\sim \tilde{\U}_h}[h=0] =  \frac{1}{2^{d}}$.
}

\item Using the above observation that $\Pr_{h\sim \tilde{\U}_h}[h=0] =  \frac{1}{2^{d}}$, we have
\begin{align*}
\Pr_{h\sim \tilde{\mathcal{D}}_h}[h=0]-\Pr_{h\sim \tilde{\U}_h}[h=0] %&= \sum_{v\in \uH_d}\big(p_v^2-\frac{1}{2^{2d}}\big)\\
&= \sum_{v\in \uH_d}\big(p_v^2-\frac{1}{2^{2d}}\big)\\
&= \sum_{v\in \uH_d}\big(p_v^2-\frac{2p_v}{2^d}+\frac{1}{2^{2d}}\big)+\sum_{v\in \uH_d}\big(-\frac{2}{2^{2d}}+\frac{2p_v}{2^d}\big)\\
&= \sum_{v\in \uH_d}\big(p_v-\frac{1}{2^{d}}\big)^2\\
&> 0
\end{align*}
where the strict inequality holds since $\mathcal{D}$ is not uniform.

\item We argue similarly as in the proofs of the previous two statements. We have
\begin{align}
&2\left(\Pr_{h\sim \tilde{\mathcal{D}}_h}[h=k]-\Pr_{h\sim \tilde{\mathcal{\U}}_h}[h=k]\right)\\
&\quad= 2\sum_{v, u\in \uH_d, h(v, u) = k}\big(p_v p_u-\frac{1}{2^{2d}}\big)\nonumber\\
&\quad= \sum_{v, u\in \uH_d, h(v, u) = k}\big(-p_v^2+2p_vp_u-p_u^2\big)+\sum_{v, u\in \uH_d, h(v, u) = k}\big(p_v^2+p_u^2-\frac{2}{2^{2d}}\big)\label{eq:2}\\
&\quad= -\sum_{v, u\in \uH_d, h(v, u) = k}\big(p_v-p_u\big)^2+\sum_{v, u\in \uH_d, h(v, u) = k}\big(p_v^2-\frac{1}{2^{2d}}+p_u^2-\frac{1}{2^{2d}}\big)\nonumber\\
&\quad= -\sum_{v, u\in \uH_d, h(v, u) = k}\big(p_v-p_u\big)^2+2\binom{d}{k}\sum_{v\in \uH_d}\big(p_v^2-\frac{1}{2^{2d}}\big)\label{eq:22}\\
% &\quad= -\sum_{v, u\in \uH_d, h(v, u) = k}\big(p_v-p_u\big)^2+\sum_{v, u\in \uH_d, h(v, u) = k}\big(p_v^2-\frac{1}{2^{2d}}+p_u^2-\frac{1}{2^{2d}}\big)\\
% &\quad= -\sum_{v, u\in \uH_d, h(v, u) = k}\big(p_v-p_u\big)^2+2\binom{d}{k}\sum_{v\in \uH_d}\big(p_v^2-\frac{1}{2^{2d}}\big)\\
&\quad= -\sum_{v, u\in \uH_d, h(v, u) = k}\big(p_v-p_u\big)^2+2\binom{d}{k}\left(\Pr_{h\sim \tilde{\mathcal{D}}_h}[h=0]-\Pr_{h
\sim \tilde{\mathcal{\U}}_h}[h=0]\right) \nonumber\\%\label{eq:222} \\
&\quad< 2\binom{d}{k}\left(\Pr_{h\sim \tilde{\mathcal{D}}_h}[h=0]-\Pr_{Y\sim \U}[h(y,Y)=0]\right) \label{eq:2222}
\end{align}
where \eqref{eq:2} is obtained by adding and subtracting $p_v^2+p_u^2$, \eqref{eq:22} follows by the symmetry of the hypercube and the fact that for every $v\in\mathcal{H}_d$, there are $ \binom{d}{k}$ vertices $u\in \uH_d$ satisfying $h(v,u)=k$, and \eqref{eq:2222} follows by the fact that $p_v \neq p_u$ for some $v,u\in \mathcal{H}_d$ for all non-uniform distributions $\mathcal{D}$.
\end{enumerate}
\end{proof}

% Finally we have the following
\begin{lemma}\label{lem:17}
If $\beta > c\log d$ for an absolute constant $c$,  $d>3$ and $C>1$, then $\sum_{k\ge 0}\binom{d}{k}\Delta_C(k)< 0$.
\end{lemma}
\begin{proof}
According to the transition matrix of $h$ for the case that $h=0$ we know that $s$ could be either 0, 1, or 2, with probabilities denoted in \eqref{h_zero}. Hence, we can simplify the expression for $\Delta_C(0) $ as
\begin{align*}
\Delta_C(0) 
&= \log\bigg(1+\frac{1}{C}\bigg)\frac{d+1}{4d}+\log\bigg(1+\frac{e^{-2\beta}}{C}\bigg)\frac{1}{2}+\log\bigg(1+\frac{e^{-4\beta}}{C}\bigg)\frac{d-1}{4d}-\log\bigg(1+\frac{1}{C}\bigg)\\
&=-\frac{3d-1}{4d}\log\bigg(1+\frac{1}{C}\bigg)+\log\bigg(1+\frac{e^{-2\beta}}{C}\bigg)\frac{1}{2}+\log\bigg(1+\frac{e^{-4\beta}}{C}\bigg)\frac{d-1}{4d}%\\
% &\le -\frac{3d-1}{4d}\log(1+\frac{1}{C})+\frac{e^{-2\beta}}{2C}+\frac{e^{-4\beta}}{4C}
\end{align*}
We can similarly compute $\Delta_C(1)$, $\Delta_C(2), \ldots$ to obtain:
    {
    \begin{align}
    &\sum_{k\ge 0}\binom{d}{k}\Delta_C(k)\nonumber\\
    &\quad=\log\bigg(1+\frac{1}{C}\bigg)\frac{-3d+1}{4d}+\log\bigg(1+\frac{e^{-2\beta}}{C}\bigg)\frac{1}{2}+\log\bigg(1+\frac{e^{-4\beta}}{C}\bigg)\frac{d-1}{4d}\nonumber\\
    &\quad \quad + \binom{d}{1}\bigg(\log(1+\frac{1}{C})\frac{1}{2d}+\log\bigg(1+\frac{e^{-2\beta}}{C}\bigg)\frac{-3d^2+3d-2}{4d^2} \nonumber \\
    &\quad \quad \quad \quad \quad \quad +\log\bigg(1+\frac{e^{-4\beta}}{C}\bigg)\frac{d-1}{2d}+\log(1+\frac{e^{-6\beta}}{C})\frac{(d-1)(d-2)}{4d^2}\bigg)\nonumber\\
    &\quad\quad+\binom{d}{2}\bigg(\log\left(1+\frac{1}{C}\right)\frac{1}{2d^{2}}\nonumber \\
    &\quad\quad\quad \quad \quad \quad +\log\left(1+\frac{e^{-2\beta}}{C}\right)\frac{1}{d}+\log\left(1+\frac{e^{-4\beta}}{C}\right)\left(-\frac{3}{4}+\frac{2\left(d-1\right)+3\left(d-2\right)}{4d^{2}}\right)\nonumber\\
    &\quad\quad\quad \quad \quad \quad +\log\left(1+\frac{e^{-6\beta}}{C}\right)\frac{\left(d-2\right)}{2d}+\log\left(1+\frac{e^{-8\beta}}{C}\right)\frac{\left(d-2\right)\left(d-3\right)}{4d^{2}}\bigg) \nonumber\\
    &\quad \quad +\sum_{k> 2}\binom{d}{k}\Delta(k)\nonumber\\
    &\quad=\log\bigg(1+\frac{1}{C}\bigg)\bigg(\frac{-3d+1}{4d}+\frac{\binom{d}{1}}{2d}+\frac{\binom{d}{2}}{2d^2}\bigg)+\log\bigg(1+\frac{e^{-2\beta}}{C}\bigg)\bigg(\frac{1}{2}+\binom{d}{1}\frac{-3d^2+3d-2}{4d^2}+\frac{\binom{d}{2}}{d}\bigg)\nonumber\\
    &\quad\quad+\log\bigg(1+\frac{e^{-4\beta}}{C}\bigg)\bigg(\frac{d-1}{4d}+\binom{d}{1}\frac{d-1}{2d}+\binom{d}{2}\left(-\frac{3}{4}+\frac{2\left(d-1\right)+3\left(d-2\right)}{4d^{2}}\right)\bigg)\nonumber\\
    &\quad\quad + \log\bigg(1+\frac{e^{-6\beta}}{C}\bigg)\bigg(\binom{d}{1}\frac{(d-1)(d-2)}{4d^2}+\binom{d}{2}\frac{\left(d-2\right)}{2d}\bigg)\nonumber\\ 
    &\quad\quad+\log\left(1+\frac{e^{-8\beta}}{C}\right)\binom{d}{2}\frac{\left(d-2\right)\left(d-3\right)}{4d^{2}}+\sum_{k> 2}\binom{d}{k}\Delta(k)\nonumber\\
    &\quad=-\log\bigg(1+\frac{e^{-2\beta}}{C}\bigg)\bigg(\frac{d^2 -3d +2}{4d}\bigg)\nonumber\\
    &\quad\quad+\log\bigg(1+\frac{e^{-4\beta}}{C}\bigg)\bigg(\frac{d-1}{4d}+\binom{d}{1}\frac{d-1}{2d}+\binom{d}{2}\left(-\frac{3}{4}+\frac{2\left(d-1\right)+3\left(d-2\right)}{4d^{2}}\right)\bigg)\nonumber\\
    &\quad\quad + \log\bigg(1+\frac{e^{-6\beta}}{C}\bigg)\bigg(\binom{d}{1}\frac{(d-1)(d-2)}{4d^2}+\binom{d}{2}\frac{\left(d-2\right)}{2d}\bigg)\nonumber\\ 
    &\quad\quad+\log\left(1+\frac{e^{-8\beta}}{C}\right)\binom{d}{2}\frac{\left(d-2\right)\left(d-3\right)}{4d^{2}}+\sum_{k> 2}\binom{d}{k}\Delta(k)\nonumber\\
    &\quad \leq  - \frac{e^{-2\beta}}{C} \frac{d^2 -3d +2}{8d} + \frac{c'}{C} d^2 e^{-4\beta} +  \sum_{k> 2}\binom{d}{k}\Delta(k) \label{eq:mk}
    % &\quad\le-\log\bigg(1+\frac{e^{-\beta}}{C}\bigg)\frac{d^2-3d+2}{4}-\log\bigg(1+\frac{e^{-2\beta}}{C}\bigg)\frac{4d^2-6d+3}{12}+\log\bigg(1+\frac{e^{-3\beta}}{C}\bigg)\frac{d^2+d}{4}\nonumber\\ &\quad\quad+\log\left(1+\frac{e^{-4\beta}}{C}\right)\frac{d^2}{8}+\sum_{k> 2}\binom{d}{k}\Delta(k)\nonumber\\
    % % &\quad\le-\log(1+\frac{e^{-\beta}}{C})\frac{d^2-3d+2}{4}-\log(1+\frac{e^{-2\beta}}{C})\frac{4d^2-6d+3}{12}+\log(1+\frac{e^{-3\beta}}{C})\frac{d^2+d}{4}\\ &\quad\quad+\log\left(1+\frac{e^{-4\beta}}{C}\right)\frac{d^2}{8}+\sum_{k> 2}\binom{d}{k}\Delta(k)\\
    % &\quad\le-\frac{e^{-\beta}}{2C}\frac{d^2-3d+2}{4}-\frac{e^{-2\beta}}{2C}\frac{4d^2-6d+3}{12}+\frac{e^{-3\beta}}{C}\frac{d^2+d}{4}+\frac{e^{-4\beta}}{C}\frac{d^2}{8}\nonumber\\
    % &\quad\quad+\sum_{k> 2}d^k \left(\frac{e^{-\beta\cdot\left(k-2\right)}}{C}\frac{k^2}{4d^{2}}+\frac{e^{-\beta\left(k-1\right)}}{C}\frac{k}{2d}+\frac{e^{-\beta\left(k+1\right)}}{2C}+\frac{e^{-\beta\left(k+2\right)}}{4C}\right) 
    \end{align}}
    where in \eqref{eq:mk} we have used the numerical inequalities $-\log(1+x)\leq -\frac{x}{2}$ for $x \in [0,1]$ and $\log(1+x)\leq x$, and  $\beta > c \log d$, and $c'$ is a sufficiently large constant.
    
    % the terms with $e^{-\beta}$ are dominant for sufficiently large constant $c$. In particular, 
    % first term in \eqref{last1} is dominant for sufficiently large constant $c$, and this term is negative since $d>3$ ensures $d^2 - 3d + 2$ is positive.

    For $k>2$, we again use $\log(1+x)\leq x$ to obtain
\begin{align}
\Delta(k)
&\le\frac{e^{-2\beta\cdot\left(k-2\right)}}{C}\frac{\left(k-1\right)k}{4d^{2}}+\frac{e^{-2\beta\left(k-1\right)}}{C}\frac{k}{2d}-\frac{1}{2}\log\left(1+\frac{e^{-2\beta k}}{C}\right)+\frac{e^{-2\beta\left(k+1\right)}}{2C}+\frac{e^{-2\beta\left(k+2\right)}}{4C} \nonumber\\
&\leq \frac{e^{-2\beta\cdot\left(k-2\right)}}{C}\frac{\left(k-1\right)k}{4d^{2}}+ \frac{c''}{C} e^{-2\beta(k-1)}\label{k}
% &\le  \frac{e^{-2\beta\left(k-2\right)}}{C}\frac{k^2}{4d^{2}}+\frac{e^{-2\beta\left(k-1\right)}}{C}\frac{k}{2d}+\frac{e^{-2\beta\left(k+1\right)}}{2C}+\frac{e^{-2\beta\left(k+2\right)}}{4C}\\
% &\le  \frac{e^{-2\beta(k-2)}}{C}\left(\frac{k^2}{4d^{2}}+\frac{ke^{-2\beta}}{2d}+\frac{e^{-4\beta}}{2C}+\frac{e^{-6\beta}}{4C}\right)\\
% &\le  \frac{e^{-2\beta(k-2)}}{C}\left(\frac{k^2}{4d^{2}}+\frac{k}{2d^5}+\frac{1}{2Cd^8}+\frac{1}{4Cd^{12}}\right)
\end{align}
for an absolute constant $c''$. Combining this bound with \eqref{eq:mk} yields
\begin{align}
  \sum_{k\ge 0}\binom{d}{k}\Delta_C(k)&\le  - \frac{e^{-2\beta}}{C} \frac{d^2 -3d +2}{8d} + \frac{c'}{C} d^2 e^{-4\beta} +    \sum_{k=3}^d\frac{e^{-2\beta\cdot\left(k-2\right)}}{C}\frac{\left(k-1\right)k}{4d^{2}}+ \frac{c''}{C} e^{-2\beta(k-1)} \nonumber \\
  &\leq  \frac{e^{-2\beta}}{C} \left(-\frac{d^3 -3d^2 +2d}{8d^2} + \frac{12}{8d^2} \right) + \frac{c'''}{C} d^2 e^{-4\beta}  \label{last2}\\
  &< 0 \label{last3}
\end{align}
%     \begin{align}
%     \sum_{k\ge 0}\binom{d}{k}\Delta_C(k)&\le -\frac{e^{-\beta}}{2C}\frac{d^2-3d+2}{4} + \frac{9}{4C} d {e^{-\beta}} + \frac{c'}{C} d^3 e^{-2\beta} \label{last2}  \\
%     &= \frac{e^{-\beta}}{C}\left( -\frac{d^2-21d+2}{8} \right) + \frac{c'}{C} d^3 e^{-2\beta}\nonumber \\
%     &< 0 \label{last3}
%     % \leq \frac{c'}{C} d^2 e^{-\beta} +\frac{c'}{C}\sum_{k> 2} k^2 d^{k-2} {e^{-\beta\cdot\left(k-2\right)}} \label{last2} \\
%     % &\leq -\frac{c'}{C} d^2 e^{-\beta} +\frac{c''}{C} d {e^{-\beta}} + \frac{c''}{C} d^3 {e^{-2\beta}} \label{last3} \\
% \end{align}
where \eqref{last2} holds for an absolute constant $c'''$,  and \eqref{last3} holds for a sufficiently large constant $c$ and $d>3$, where throughout we have used $\beta > c \log d$.
\end{proof}

\begin{lemma} \label{lem:stationary}
    $\{\mathcal{D}_t\}_t $ converges to  $\mathcal{U}_d$.
\end{lemma}

\begin{proof}
    The transition kernel of  $\mathcal{D}_t $ is aperiodic and irreducible over a finite state space, and has a symmetric transition kernel, so it must converge to the uniform distribution \citep{bremaud}.
\end{proof}

\subsection{Proof of Theorem \ref{thm:uniformandclean}} \label{app:uniform-b}

\noindent Now using the above results, we prove the main claim of Theorem \ref{thm:uniformandclean}.
\vspace{0.1in}

\begin{proof} 
%Recall the definitions of the functions ${{\mathcal{L}}}(g)$ in \eqref{sth} and ${{\tilde{\mathcal{L}}}}(g)$ in \eqref{tilde_loss}. Note that since we always have $g(x)^\top g(x^+) \le d$, we can argue that for any $g$ we have $\mathcal{L}(g)\ge \tilde{\mathcal{L}}(g)$. Indeed, the equality holds when $g$ is a clean representation. Moreover, for any non-clean representation $g$, we know that  there exists at least one image $x$ for which its representation $g(x)$ is not exactly aligned with the representation of one of its augmented images. Therefore, the first term of the InfoNCE loss, i.e., $-\underbrace{ \beta \mathbb{E}_{x,x^+} [g(x)^\top g(x^+) ]}_{\text{alignment}} $, for that representation would be strictly larger than $- \beta d$. Hence, for non-clean $g$ we have $\mathcal{L}(g)> \tilde{\mathcal{L}}(g)$.
%
% {\color{red} change clean $g$ to cluster-preserving in this proof. Or use the language `$g$ composed of clean functions'.} 
Note that the InfoNCE loss can be written as 
\begin{align}\label{sth-supp}
{{\mathcal{L}}}(g) &=\underset{x,x^+,\{x_i^-\}_{\ell}}{\mathbb{E}} \bigg[  \log\bigg(1+ \sum_{i= 1}^\ell e^{\beta g(x)^\top \left(g(x_{i}^-) - g(x^+)\right)} \bigg)\bigg] 
% \nonumber \\
% & = - \beta \mathbb{E}_{x,x^{+}} [g(x)^\top g(x^{+})] + \underset{x,x^+,\{x_i^-\}_{\ell}}{\mathbb{E}} \bigg[  \log\bigg( e^{\beta g(x)^\top g(x^{+})}+ \sum_{i= 1}^\ell e^{\beta g(x)^\top g(x_{i}^-) } \bigg)\bigg]
\end{align}
Considering that we search over representations composed of clean functions, we know that for all $g\in \mathcal{G}_c$, the term $g(x)^\top g(x^+)$ is exactly equal to $ d$. Hence, the optimizing $\mathcal{L}(g)$ over $\mathcal{G}_c$ simplifies to minimizing 
\begin{align}\label{tilde_loss-supp}
{{\hat{\mathcal{L}}}}(g) & \coloneqq \underset{x,x^+,\{x_i^-\}_{\ell}}{\mathbb{E}} \bigg[  \log\bigg(1+ \sum_{i= 1}^\ell e^{\beta g(x)^\top g(x_{i}^-) - \beta d } \bigg)\bigg] 
% \nonumber \\
% &= - \beta d + \underset{x,x^+,\{x_i^-\}_{\ell}}{\mathbb{E}} \bigg[  \log\bigg(e^{\beta d}+ \sum_{i= 1}^\ell e^{\beta g(x)^\top g(x_{i}^-) } \bigg)\bigg]
\end{align}

{\color{black} Below, we use the term `clean representation' to indicate that the the representation is composed of clean functions, and a non-clean representation if at least one of the functions in the representation is not clean.}
Recall the definitions of the functions ${{\mathcal{L}}}(g)$ in \eqref{sth-supp} and ${{\hat{\mathcal{L}}}}(g)$ in \eqref{tilde_loss-supp}. Note that since we always have $g(x)^\top g(x^+) \le d$, we can argue that for any $g$ we have $\mathcal{L}(g)\ge \hat{\mathcal{L}}(g)$. Indeed, the equality holds when $g$ is a clean representation. Moreover, for any non-clean representation $g$, we know that  there exists at least one image $x$ for which its representation $g(x)$ is not exactly aligned with the representation of one of its augmented images $x^{+}$. Therefore, for that pair $(x,x^{+})$,  $- \beta g(x)^\top g(x^+) > - \beta d$. Hence, for some sample $x,x^{+},\{x_i^-\}_{\ell}$ with positive mass, we have: $\beta g(x)^\top  ( g(x_i^-) -  g(x^{+}))  > \beta g(x)^\top g(x_i^-) - \beta d $. Therefore, for non-clean $g$ we have $\mathcal{L}(g)> \hat{\mathcal{L}}(g)$ from \eqref{sth-supp} and \eqref{tilde_loss-supp}.

Moreover, according to the result of Lemma \ref{thm:uniform}, we know that the minimizer of the loss function $\hat{\mathcal{L}}$ is a uniform representation, thus for any non-uniform representation $g'$ and uniform representation $g''$ we have $\hat{\mathcal{L}}(g')> \hat{\mathcal{L}}(g'')$.

%Consider the related loss given below, where the positive inner product is replaced by its upper bound: 
%\begin{equation}\label{eq:Ltilde}
%\tilde{\mathcal{L}}(g) \coloneqq \underset{x,\{x_i^-\}_{\ell}}{\mathbb{E}}[  \log(1+ \sum_{i= 1}^\ell e^{\beta \left(g(x)^\top g(x_{i}^-) - d\right)} ) ]
%\end{equation}
   
Considering these two observations, we show that a uniform representation composed of clean functions, denoted by $g^*$, is an optimal solution of the loss ${\mathcal{L}}$. We consider the following two cases:

\textbf{Case 1:} If the representation $g'$ is not composed of clean functions, then we have
 \begin{align}
       {\mathcal{L}}(g') &\stackrel{(a)}{>} \hat{\mathcal{L}}(g') \stackrel{(b)}{\geq} \hat{\mathcal{L}}(g^*) \stackrel{(c)}{=}{\mathcal{L}}(g^*)  \nonumber
   \end{align}
 where $(a)$ holds with strict inequality since $g'$ is not clean (discussion above), $(b)$ holds as $g^*$ with a uniform distribution is an optimal solution of  $\hat{\mathcal{L}}$ (Lemma \ref{thm:uniform}), and $(c)$ holds because $g^*$ is composed of clean functions.
%% (Theorem \ref{thm:positive-clean}).
    
\textbf{Case 2:} If the representation $g'$ is composed of clean functions, but is not uniform, then we have 
 \begin{align}
       {\mathcal{L}}(g') &\stackrel{(a)}{\geq } \hat{\mathcal{L}}(g') \stackrel{(b)}{>} \hat{\mathcal{L}}(g^*) \stackrel{(c)}{=}{\mathcal{L}}(g^*)  \nonumber
   \end{align}
      where $(a)$ holds based on the definitions of ${\mathcal{L}}$ and $\hat{\mathcal{L}}$,  $(b)$~holds since $g'$ is not uniform  and $g^*$ is uniform (discussion above), and $(c)$~holds because $g^*$ is  composed of clean functions.

%% \textbf{Sanjay: The theorems above refer to commented parts. After the reorg of Lemma~5, we need to fix these. Case 1 -- (a from Assumption; b, c from Lemma, also clean vs cluster-preserving language}

    Combining these two cases, we obtain that the representation $g^*$ minimizes $\mathcal{L}(\cdot)$ if and only if $g^*$ is composed of clean functions and uniform. Furthermore, by Assumption~\ref{assump:spikyaug}, $g^*$ being composed of clean functions implies that is it cluster-preserving.
% Suppose that $g'\in \arg \min_{g \in \mathcal{G}} \mathcal{L}(g)$ is not clean.
   % Theorem \ref{thm:uniform} implies that all $\tilde{g} \in \arg \min_{g \in \mathcal{G}} \tilde{\mathcal{L}}(g)$ are uniform.  Further,
   % Assumption \ref{assump:spiky} implies that for any representation $g \in \mathcal{G}$ that is not clean,  ${\mathcal{L}}(g) > \tilde{\mathcal{L}}(g)$, whereas ${\mathcal{L}}(g) = \tilde{\mathcal{L}}(g)$ for all clean $g\in \mathcal{G}$. 
 %  Let $g*$ be clean and uniform and suppose $g'\in \arg \min_{g \in \mathcal{G}} \mathcal{L}(g)$ is  not clean and/or not uniform. Then 
 %  \begin{align}
  %     {\mathcal{L}}(g') &\stackrel{(a)}{\geq} \tilde{\mathcal{L}}(g') \stackrel{(b)}{\geq} \tilde{\mathcal{L}}(g^*) \stackrel{(c)}{=}{\mathcal{L}}(g^*)  \nonumber
 %  \end{align}
\end{proof}

\section{Agnostic Case}\label{appendix:agnostic}

In this section we prove Theorem \ref{thm:agnostic}.

\begin{theorem}[Theorem \ref{thm:agnostic} Restated] \label{thm:agnostic-app}
% Let $\ell = \frac{C}{\epsilon} d 2^d $ for an absolute constant $C$ and $\beta \geq c_3 2^d$. 
Suppose Assumptions \ref{assump:exist} and \ref{assump-b} hold and $g= [f_1,...,f_d]$ is \textbf{not} cluster-preserving with 
$\min_{j\in [d]} \min_{\mathbf{c} \in \Sigma_{f_j}} \mathbb{P}_{x,x'\sim D}[x, x' \in \Gamma_{\mathbf{c}},f_j(x)\neq f_j(x')] \geq \epsilon >0$. Let $\ell \geq \frac{ c  }{\epsilon} d 2^d $, $\beta \geq {c}\log(\frac{c}{\epsilon}) 2^d$ for a sufficiently large constant $c>1$.
% $\min_{j\in [d]} \min_{\mathbf{c} \in \Sigma_{f_j}} \|\{x^+\in A(x): x\in \Gamma_{\mathbf{c},\circ}, \; f_j(x)\ne f_j(x^+)\}\| \geq \epsilon >0$. 
Moreover, suppose $g$
 is close to a uniform
representation in the sense that $\mathbb{P}_{x\sim D_\circ}[g(x)=v] \geq \frac{10}{ c d2^d}$ 
or $\mathbb{P}_{x\sim D_\circ}[g(x)=v]  \leq \frac{ \epsilon}{ 100 c d 2^{2d}} $
% $\|Q_v\|_\circ > \frac{\log(\eta)}{c 2^d}$ or $\|Q_v\|_\circ \leq \frac{c'}{ \eta d 2^{2d}}$ 
for all $v\in \mathcal{H}_d$.
Then $g$ is \textbf{not} a minimizer of the InfoNCE loss. 
% $\mathcal{L}(g)> \mathcal{L}(g')$ for some nearby representation $g'$. 
% Consider the InfoNCE objective \ref{eq:InfoNCE} when $l \ge m \log m$. 
% The optimal representation is clean and such that it maximizes the entropy of the downstream representations to within an additive error.
\end{theorem}

\begin{proof}
First we recall notations: for a set of images $B\subseteq D$, we employ the notations  $\|B\|_\circ \coloneqq \mathbb{P}_{x \sim D_{\circ}}[x \in B]$ and $\|B\| \coloneqq \mathbb{P}_{x \sim D\setminus D_{\circ}}[x \in B]$. Also, we let $\Gamma_{\mathbf{c},\circ}\coloneqq \Gamma_{\mathbf{c}}\cap D_\circ$.

{\color{black}
As discussed in the proof sketch,  we construct a representation $g'$ that is close to  $g$ by changing one coordinate of $g$ such that it preserves one additional cluster, and show that the resulting $g'$ achieves smaller InfoNCE loss than $g$. 
% We construct $g'$ by changing only one coordinate of $g$ such that it preserves one cluster that it intersects. 

Suppose WLOG that $f_1$ does not preserve the cluster $\Gamma_{\mathbf{c}}$ for some $\mathbf{c}\in C$. That is, $\exists x,x'\in \Gamma_{\mathbf{c}}$ such that $f_1(x)\neq f_1(x')$.
% has largest misclassification of positive samples on its most misclassified cluster among all $d$ classifiers in $g$, specifically that $1 \in \arg\max_{j\in [d]} \max_{\mathbf{c}' \in \Sigma_{f_j}} \|\{x^+\in Q(x): x\in \Gamma_{\mathbf{c}',\circ}, \; f_j(x)\ne f_j(x^+)\}\|$. 
% meaning Suppose WLOG  that $1 \in \arg\max_{j\in [d]} \max_{\mathbf{c} \in \Sigma_{f_j}} \|\{x^+\in Q(x): x\in \Gamma_{\mathbf{c},\circ}, \; f_j(x)\ne f_j(x^+)\}\|$, meaning 
Further, let $f_1^{(\mathbf{c},\sigma)}$ be the smallest perturbation of $f_1$ that preserves $\Gamma_{\mathbf{c}}$. Specifically, $f^{({\mathbf{c}},\sigma)}(x) \coloneqq \begin{cases} f(x) & x \notin \Gamma_{\mathbf{c}} \\
    \sigma & x \in \Gamma_{\mathbf{c}} \\
    \end{cases}$, where $\sigma\in \argmin_{\sigma' \in \{-1,1\}} \Vert \{x \in \Gamma_{\mathbf{c},\circ} \mid f_1(x) \ne f_1^{({\mathbf{c}},\sigma')} (x)\}\Vert_\circ$.

% $ \sigma \in \arg \min_{\sigma' \in \{-1,1\}} \Vert \{x \in D_\circ \mid f_1(x) \ne f_{1}^{(\mathbf{c},\sigma')} (x)\}\Vert_\circ $, using the notation defined in Assumption \ref{assump-b}, and 
Denote $f_1'= f_1^{(\mathbf{c},\sigma)}$.  {By Assumption \ref{assump:exist}, $f_1' \in \mathcal{F}$.} 
Construct $g' =[f_1',f_2,\dots,f_d] \in \mathcal{G}$, and note that $g'$ is equivalent to $g$ on all but one coordinate, and the one differing coordinate differs only on one cluster, for which this coordinate preserves the cluster in $g'$ but does not preserve it in $g$.
}

% Since $g= [f_1,\dots,f_d]$ is not clean, there exists $j\in [d]$ such that $f_j$ is not clean. WLOG let this $j=1$. Let $S$ denote one of the clusters that $f_1$ intersects. 
% Suppose WLOG that $1 \in \arg\max_{j\in [d]} \Pr_{x\sim \mathcal{D}_\circ,x^+\sim  \kappa(x)}[f_j(x)\ne f_j(x^+)]$, and let $ \mathbf{c}',\sigma' \in  \min_{\mathbf{c}  \in \Sigma_{f_1}} \min_{\sigma \in \{-1,1\}} \Vert \{x \in \mathcal{D}_\circ \mid f_1(x) \ne f_{1}^{(\mathbf{c},\sigma)} (x)\}\Vert_\circ $, using the notation defined in Assumption \ref{assump-b}.  {By Assumption \ref{assump:exist}, $f_1^{(\mathbf{c}',\sigma')}\in \mathcal{F}$.} 
% Construct $g' =[f_1^{(\mathbf{c}',\sigma')},f_2,\dots,f_d]$. %where $f_1'(x) = f_1(x)$ for all $x \notin S$. 
% % % Define the set $B := \{x \in \mathcal{D}_\circ \cup \mathcal{D}: f_1(x) \neq f_1'(x) \} \subset S$. 
% % Let $f_1'(x) = 1$ or $f_1'(x) = -1$ for all $x \in S$, whichever induces smaller $\| \{x : f_1'(x) \neq f_1(x) \} \|$. 

% Now for the negative terms. Consider a cluster corresponding to $v \in \uH^d$. 

% ($B$ is deterministic, $n_2$ is random). 
% Also note that here $Q_v$  is totally unrelated  to the set $A$ in the picture.

We show that $\mathcal{L}(g)- \mathcal{L}(g')>0$, where
\begin{align}
    \mathcal{L}(g)- \mathcal{L}(g') = \mathcal{L}^+(g) - \mathcal{L}^+(g') + \mathcal{L}^-(g) - \mathcal{L}^-(g').
\end{align}
where 
\begin{align*}
\mathcal{L}^+(g) &:= -\beta \mathbb{E}_{x, x^+}\left[g(x)^\top g(x^+)\right] \nonumber \\
\mathcal{L}^-(g) 
&:= \mathbb{E}_{x,x^+, \{x^-_i\}_{\ell}} \bigg[\log\bigg({ e^{\beta g(x)^\top g(x^+)}\!+\!\sum_{i=1}^\ell e^{\beta g(x)^\top g(x^-_{i})}  } \bigg)\bigg] \nonumber
\end{align*}
where $\mathcal{L}^+$ and $\mathcal{L}^-$ respectively correspond to the alignment and uniformity losses in \ref{eq:InfoNCE}. We will refer to these losses as the positive and negative losses, respectively.
We show $\mathcal{L}(g)- \mathcal{L}(g')>0$ by showing 
\begin{align}
    \mathcal{L}^+(g) - \mathcal{L}^+(g') > \mathcal{L}^-(g') - \mathcal{L}^-(g)  \label{show}
\end{align}
by first computing the LHS explicitly, then upper bounding the RHS. To do this, we define a partitioning of the set of images $D$ as follows.
For all $v \in \mathcal{H}_d$, define the sets $Q_v := \{x \in  D : g(x)=v, g'(x)=v \}$ and $E_v := \{x \in  D : g(x)= v, g'(x)\neq v \}$. 
In other words $Q_v \cup E_v$ is the set of images that $g$ maps to $v$, and $Q_v$ is the subset of this set which $g'$ also maps to $v$, while $E_v$ is the subset which $g'$ does not map to $v$.  
% $\bar{A}_{v} = \{j\mid g(x_j^-) = v\}$ with $n_{1,v} := |\bar{A}_v| $ (random), 
% Note that $\|Q_v\|_\circ = D_g(v)$. 
Let  $Q:= \cup_{v \in \mathcal{H}_d} Q_v$ and $E:= \cup_{v \in \mathcal{H}_d} E_v$. Observe that  $Q\cup E= D$, and each of the $Q_v$ and $E_v$'s are disjoint, so they form a partition of $D$. 
Further note that $\|E\|_\circ = \Delta_{f_1,\mathbf{c}}\coloneqq \min_{\sigma' \in \{-1,1\}} \Vert \{x \in \Gamma_{\mathbf{c},\circ} \mid f_1(x) \ne f_1^{({\mathbf{c}},\sigma')} (x)\}\Vert_\circ $, as defined in Definition \ref{defn:reg}.

Now we consider the difference in  the positive losses. For all augmentations $x^+ \in {D}\setminus D_\circ$, define $\tilde{\mathcal{A}}^{-1}(x^+)$ as the natural image from which the augmentation was derived, i.e. \\
$ \tilde{\mathcal{A}}^{-1}(x^+) = x \iff \mathcal{A}(x) = x^+$ for some $\mathcal{A}\in \Lambda$.
Moreover, define $R := \{ x^+ \in E\setminus D_\circ : \tilde{\mathcal{A}}^{-1}(x^+) \in Q  \} \cup \{ x^+ \in Q\setminus D_\circ : \tilde{\mathcal{A}}^{-1}(x^+) \in E\} = \{x^+ \in \Gamma_{\mathbf{c}}:f_1(\tilde{\mathcal{A}}^{-1}(x^+)) \neq f_1(x^+)  \}$ as the set of augmentations in $\Gamma_{\mathbf{c}}$ that $f_1$ classifies incorrectly. For any event $B$, let $\chi B$ denote the indicator random variable for $B$, i.e. $\chi B = 1$ if $B$ occurs and $\chi B = 0$ otherwise. Using this notation and the construction of $g'$ we can write the difference in positive losses as:
\begin{align*}
\mathcal{L}^+(g) - \mathcal{L}^+(g') &= 
\beta \mathbb{E}_{x, x^+}\left[g'(x)^\top g'(x^+) - g(x)^\top g(x^+)\right] \nonumber \\
&= \beta \mathbb{E}_{x, x^+}\left[f_1'(x) f_1'(x^+) - f_1(x) f_1(x^+)\right] \nonumber \\
&= \beta \sum_{v \in \mathcal{H}_d}\mathbb{E}_{x, x^+}\big[\chi\{x \in Q_v\}(f_1'(x) f_1'(x^+) - f_1(x) f_1(x^+)) \nonumber \\
&\quad \quad \quad \quad \quad \quad \quad + \chi\{x \in E_v\}(f_1'(x) f_1'(x^+) - f_1(x) f_1(x^+))\big] \nonumber \\
&= 2\beta \sum_{v \in \mathcal{H}_d}\mathbb{E}_{x, x^+}\big[\chi\{x \in Q_v\}\chi\{x^+ \in E\} + \chi\{x \in E_v\}\chi\{x^+ \in Q\}\big] \nonumber \\
&= 2\beta \sum_{v \in \mathcal{H}_d}\big(\Pr_{x \sim D_\circ, x^+\sim A(x)}\left[x\in Q_v, x^+\in E\right]\nonumber \\
&\quad \quad \quad \quad \quad \quad +\Pr_{x \sim D_\circ, x^+\sim A(x)}\left[x\in E_v, x^+\in Q\right] \big)\nonumber \\
&= 2 \beta \|R\|
\end{align*}
where in the last equality we have used that all augmentation sets are of equal size.
% By Assumption \ref{assump-b}, we have $\|R\| \geq \delta \|E\|_\circ$.
% Meanwhile 
% \begin{align*}
% \underset{x, x^+}{E}[\beta g'(x)g'(x^+)] &\le  \beta d \Pr[g(x)= g(x^+)] + \beta(d-2)\Pr[g(x)\ne g(x^+)] \\
% &= \beta(d-2) + 2\beta\left(1-\left(\Pr\left[x\in Q, x^+\in E\right]+\Pr\left[x\in E, x^+\in Q\right]\right)\right)
% \end{align*}
% The difference is
% \begin{align}
% \underset{x, x^+}{E}[\beta \left(g'(x)g'(x^+) - g(x)g(x^+)\right)] \le 2\beta\left(\Pr\left[x\in Q, x^+\in E\right]+\Pr\left[x\in E, x^+\in Q\right]\right) 
% &= 2\beta \Delta_f^{new}
% \end{align}{\color{red}(A and B are as in the picture), and $\Delta_f^{new}$ defined as the fraction of natural images whose augmentation sets are split by $f$.}

Now we consider the negative losses. 
We first decompose the negative loss as
\begin{align*}
\mathcal{L}^-(g) 
&= \mathbb{E}_{x,x^+, \{x^-_i\}_{\ell}}\bigg[ \log\bigg({ e^{\beta g(x)^\top g(x^+)}\!+\!\sum_{i=1}^\ell e^{\beta g(x)^\top g(x^-_{i})}  } \bigg)\bigg] \nonumber \\
&= \sum_{v \in \uH_{d}} \mathbb{E}_{x,x^+, \{x^-_i\}_{\ell}} \bigg[\chi \{x \in Q_v\}  \log\bigg({ e^{\beta g(x)^\top g(x^+)}\!+\!\sum_{i=1}^\ell e^{\beta g(x)^\top g(x^-_{i})}  } \bigg) \bigg] \nonumber \\
&\quad + \sum_{v \in \uH_{d}} \mathbb{E}_{x,x^+, \{x^-_i\}_{\ell}} \bigg[\chi \{x \in E_v\}  \log\bigg({ e^{\beta g(x)^\top g(x^+)}\!+\!\sum_{i=1}^\ell e^{\beta g(x)^\top g(x^-_{i})}  } \bigg) \bigg] \nonumber \\
&= \sum_{v \in \uH_{d}} \mathcal{L}^-_{Q_v}(g)  +  \mathcal{L}^-_{E_v}(g) 
\end{align*}
where $\mathcal{L}^-_{Q_v}(g)  := \mathbb{E}_{x,x^+, \{x^-_i\}_{\ell}} \left[\chi \{x \in Q_v\}  \log\big({ e^{\beta g(x)^\top g(x^+)}\!+\!\sum_{i=1}^\ell e^{\beta g(x)^\top g(x^-_{i})}  } \big)\right] $ and \\
$\mathcal{L}^-_{E_v}(g)  := \mathbb{E}_{x,x^+, \{x^-_i\}_{\ell}} \left[\chi \{x \in E_v\}  \log\big({ e^{\beta g(x)^\top g(x^+)}\!+\!\sum_{i=1}^\ell e^{\beta g(x)^\top g(x^-_{i})}  } \big)\right] $.
Note that we need to upper bound \begin{align}\mathcal{L}^-(g') - \mathcal{L}^-(g) = \sum_{v \in \uH_{d}} \mathcal{L}^-_{Q_v}(g') - \mathcal{L}^-_{Q_v}(g)  +  \mathcal{L}^-_{E_v}(g') -\mathcal{L}^-_{E_v}(g).\label{gn}\end{align} 
% \begin{enumerate}
% \item In the event that $\{x\in Q_v\}$, we have 
We  analyze $\mathcal{L}^-_{Q_v}(g') - \mathcal{L}^-_{Q_v}(g) $ and $\mathcal{L}^-_{E_v}(g') -\mathcal{L}^-_{E_v}(g)$ separately for every $v \in \mathcal{H}_d$. 
To do so, we define additional notations. For a batch of negative samples $\{x_{i}^-\}_{i=1}^\ell$ and a vertex $v\in \mathcal{H}_d$, let $n_{1,v}:= \sum_{i=1}^\ell \chi\{ x_{i}^- \in Q_v \}$, $n_{2}:= \sum_{i=1}^\ell \chi\{ x_{i}^- \in E \}$, $n_{3,v}:= \sum_{i=1}^\ell \chi\{ x_{i}^- \in E_v \}$ and $n_{4}:= \sum_{i=1}^\ell \chi\{ x_{i}^- \in Q \}$.
Using the fact that $g(x)^\top g(x_i^-) = d $ for all $x,x_i^- \in Q_v$, we have
\begin{align*}
\mathcal{L}^-_{Q_v}(g) 
&= \mathbb{E}_{x,x^+, \{x^-_i\}_{\ell}} \bigg[ \chi \{x \in Q_v\}  \log\bigg({ e^{\beta g(x)^\top g(x^+)}\!+\!n_{1,v}e^{\beta d}+\sum_{x_i^{-}\notin {Q}_v} e^{\beta g(x)^\top g(x^-_{i})}  } \bigg) \bigg]\\
&= \mathbb{E}_{x,x^+, \{x^-_i\}_{\ell}} \bigg[ \chi \{x \in Q_v\} \nonumber \\
&\quad \quad \log\bigg({ e^{\beta g(x)^\top g(x^+)}+n_{1,v}e^{\beta d}+ \sum_{x_i^{-} \in E}e^{\beta g(x)^\top g(x^-_{i})} 
 +\sum_{x_i^{-}\notin {Q}_v\cup E} e^{\beta g(x)^\top g(x^-_{i})}  } \bigg)\bigg] 
\end{align*}
% For all $x_{i}^-\in B$ and $x\in Q$, we know that $g'(x)^\top g'(x_i^-) = g(x)^\top g(x_i^-) + 2\beta$
Next, using the fact that %$g'(x)^\top g'(x_i^-) = g(x)^\top g(x_i^-)$ for all $x, x_i^- \in Q$ and 
$g'(x)=g(x)$ for all $x \notin  E$, we have
\begin{align*}
\mathcal{L}^-_{Q_v}(g') & = \mathbb{E}_{x,x^+, \{x^-_i\}_{\ell}}  \bigg[\chi \{x \in Q_v\}  \log\bigg({ e^{\beta g'(x)^\top g'(x^+)}+ n_{1,v}e^{\beta d}+\sum_{x_i^{-}\notin {Q}_v} e^{\beta g'(x)^\top g'(x^-_{i})}  } \bigg)\bigg] \nonumber \\
& =  \mathbb{E}_{x,x^+, \{x^-_i\}_{\ell}} \bigg[ \chi \{x \in Q_v\} \nonumber \\
&\quad \quad \log\bigg({ e^{\beta g'(x)^\top g'(x^+)}+ n_{1,v}e^{\beta d}+ \sum_{x_i^{-}\in E} e^{\beta g'(x)^\top g'(x^-_{i})}  +\sum_{x_i^{-}\notin {Q}_v\cup E} e^{\beta g(x)^\top g(x^-_{i})}  } \bigg)\bigg] \nonumber 
% &=\mathbb{E}_{x,x^+, \{x^-_i\}_{\ell}}  \chi \{x \in Q_v\} \log\bigg({ e^{\beta g'(x)^\top g'(x^+)}+ (n_{1,v} + n_2)e^{\beta d}+\sum_{x_i^{-}\not\in {Q}_v\cup E} e^{\beta g(x)^\top g(x^-_{i})}  } \bigg)
\end{align*}
Before analyzing $\mathcal{L}^-_{E_v} $, we first prove the following claims. 
\begin{claim} \label{clm:1} For  all $x,x^-\in E_v$, $g'(x)=g'(x^-)$.
\end{claim}
\begin{proof}
    By construction of $g'$, $g'(x)$ agrees with $g(x)$ on all but the first coordinate. Thus, $g(x)\neq g'(x) \implies f_1(x) = -f_1'(x)$. Consider any $x,x^-\in E_v$.  By definition of $E_v$, $g(x)=g(x^-)=v$, and $g'(x)\neq g(x)$. Let $v=[v_1,v_2,\dots,v_d]$, then we have $g'(x)= g'(x^-)=[-v_1,v_2,\dots,v_d]$.
\end{proof}
\begin{claim} \label{clm:2} For all $v\in \mathcal{H}_d$ and all $x\in E_v$, $x^-\in E$, $g'(x)^\top g'(x^-) = g(x)^\top  g(x^-)$.
\end{claim}
\begin{proof}
Consider any $x\in E_v, x^-\in E$. As above, observe that the $j$-th coordinates of $g(x)$ and   $g'(x)$ are the same for all $j>1$ (and likewise for $g(x^-)$ and   $g'(x^-)$) by construction of $g$. Moreover, $f_1(x) = -f_1'(x)$ and $f_1(x^-) = -f_1'(x^-)$ by definition of $E$. Thus $g'(x)^\top g'(x^-) - g(x)^\top  g(x^-) = f_1'(x)f_1'(x^-) - (-f_1'(x))(-f_1'(x^-)) = 0$.
\end{proof}
\begin{claim} \label{clm:3} For all $v\in \mathcal{H}_d$ and all $x\in E_v$, $x^-\in Q$, $g'(x)^\top g'(x^-) \neq g(x)^\top  g(x^-)$.
\end{claim}
\begin{proof}
By definition of $E_v$, $g'(x) \neq g(x)$ for all $x \in E_v$, and by definition of $Q$, $g'(x^-) = g(x^-)$ for all $x^- \in Q$. Thus $g'(x)^\top g'(x^-) \neq g(x)^\top  g(x^-)$.
\end{proof}

Next we can decompose  $\mathcal{L}^-_{E_v}(g)$ as follows, using the fact that $g(x)^\top g(x_i^-) = d $ for all $x,x_i^- \in E_v$.
\begin{align*}
\mathcal{L}^-_{E_v}(g) 
&= \mathbb{E}_{x,x^+, \{x^-_i\}_{\ell}} \left[ \chi \{x \in E_v\}  \log\bigg({ e^{\beta g(x)^\top g(x^+)}\!+\!n_{3,v}e^{\beta d}+\sum_{x_i^{-}\not\in {E}_v} e^{\beta g(x)^\top g(x^-_{i})}  } \bigg)\right] \\
&= \mathbb{E}_{x,x^+, \{x^-_i\}_{\ell}}\bigg[  \chi \{x \in E_v\} \nonumber \\
&\quad \quad \quad \log\bigg({ e^{\beta g(x)^\top g(x^+)}+n_{3,v}e^{\beta d}+ \sum_{x_i^{-} \in Q}e^{\beta g(x)^\top g(x^-_{i})} 
 +\sum_{x_i^{-}\not\in {E}_v\cup Q} e^{\beta g(x)^\top g(x^-_{i})}  } \bigg)\bigg] 
\end{align*}
Next we use Claims \ref{clm:1}, \ref{clm:2} and \ref{clm:3} to obtain
\begin{align*}
\mathcal{L}^-_{E_v}(g') & = \mathbb{E}_{x,x^+, \{x^-_i\}_{\ell}} \bigg[ \chi \{x \in E_v\}  \log\bigg({ e^{\beta g'(x)^\top g'(x^+)}+ n_{3,v}e^{\beta d}+\sum_{x_i^{-}\not\in {E}_v} e^{\beta g'(x)^\top g'(x^-_{i})}  } \bigg)\bigg] \nonumber \\
& =  \mathbb{E}_{x,x^+, \{x^-_i\}_{\ell}}\bigg[  \chi \{x \in E_v\} \\
&\quad \quad \quad \log\bigg({ e^{\beta g'(x)^\top g'(x^+)}+ n_{3,v}e^{\beta d}+ \sum_{x_i^{-}\in  Q} e^{\beta g'(x)^\top g'(x^-_{i})}  +\sum_{x_i^{-}\not\in {E}_v\cup Q} e^{\beta g(x)^\top g(x^-_{i})}  } \bigg)\bigg] \nonumber 
% &=\mathbb{E}_{x,x^+, \{x^-_i\}_{\ell}}  \chi \{x \in Q_v\} \log\bigg({ e^{\beta g'(x)^\top g'(x^+)}+ (n_{1,v} + n_2)e^{\beta d}+\sum_{x_i^{-}\not\in {Q}_v\cup E} e^{\beta g(x)^\top g(x^-_{i})}  } \bigg)
\end{align*}
Next, define $Z := \{x\in Q: f_1'(x)\neq f_1'(x^-), \forall x^- \in E \} \subseteq Q$. 
% $Z \coloneqq \{Q_{v}: f'_{1}(x)\neq f'_1(x^-) \; \forall x \in Q_{v}, x^-\in E\}$
as the set of images that $f_1'$ labels differently than it does the samples in $E$ (note that $f_1'(x)=f_1'(x^-)$ for all $x,x^- \in E$, and $f_1'(x)=f_1'(x^-)$ for all $x,x^- \in Q_v$, so the $\forall$ condition in the definition of $Z$ could be replaced with `for some'). Also note that the definition of $Z$ here differs slightly from the one  in the proof sketch in Section \ref{section:agnostic} for ease of notation.
%Also note that $Z$ is defined a set of images here, wher

% $Z := \{x: f_1'(x) = f_1(x), f_1'(x)\neq f_1'(x'), x' \in E \} \subseteq Q$. 
% Note that for all $v$, $Q_v \cap C = Q_v$ or $Q_v \cap C = \emptyset$. 

Next we prove two claims regarding properties of $Z$.
\begin{claim}\label{clm:4}
    For all $v \in \mathcal{H}_d$ exactly one of the following holds: (i) $Q_v \cap Z = Q_v$ or (ii) $Q_v \setminus Z= Q_v$. 
\end{claim}
\begin{proof}
Suppose $x \in Q_v \cap Z$. Then, for all $x'\in Q_v$, $f_1'(x')=f_1'(x)$ by definition of $Q_v$. Thus $f_1'(x')=f_1'(x) \neq f_1'(x'')$ for any $x''\in E$ since $x\in Z$. This implies $x'\in Z$, therefore $Q_v \cap Z=Q_v$.

Likewise, suppose $x \in Q_v \setminus Z$. Then, for all $x'\in Q_v$, $f_1'(x')=f_1'(x)$ by definition of $Q_v$. Thus $f_1'(x')=f_1'(x) = f_1'(x'')$ for any $x''\in E$ since $x\notin Z$. This implies $x'\notin Z$, therefore $Q_v \setminus Z=Q_v$.
\end{proof}

\begin{claim}\label{clm:5}
    For all $v \in \mathcal{H}_d$, $x\in Q_v \cap Z$, $x^- \in E$, $g'(x)^\top g'(x^-) = g(x)^\top g(x^-) - 2$.
\end{claim}
\begin{proof}
    Suppose $x \in Q_v \in Z$. For any $x^-\in E$, then $f_1(x_i^-) = -f_1'(x_i^-)$ by definition of $E$, $f_1(x) = f_1'(x)$ by definition of $Q_v$ and $f_1'(x) = -f_1'(x_i^-)$ by definition of $Z$. Therefore, $f_1(x) = f_1(x_i^-) $  and $g'(x)^\top g'(x^-) = g(x)^\top g(x^-) - 2$, noting that $g$ and $g'$ agree on all but the first coordinate.
\end{proof}

\begin{claim}\label{clm:6}
    For all $v \in \mathcal{H}_d$, $Q_v \cap Z= \emptyset $ and $Q_v\neq \emptyset \implies E_v = \emptyset$. 
\end{claim}

\begin{proof}
From Claim \ref{clm:4}, $Q_v \cap Z= \emptyset \implies Q_v \setminus Z= Q_v$.
    Suppose $x \in Q_v \setminus Z$ and $x'\in E_v$. Then $f_1(x)=f_1(x')$ by definition of $Q_v$ and $E_v$. Also, $f_1'(x)=f_1'(x')$ since $x \notin Z$ and $x'\in E$, and $f_1(x')\neq f_1'(x')$ by definition of $E$. Therefore $f_1(x)\neq f_1'(x)$, but this contradicts the definition of $Q_v$.
\end{proof}

We use Claim \ref{clm:6} later in the proof. For now we use Claims \ref{clm:4} and \ref{clm:5} to bound $\mathcal{L}^-_{Q_v}(g') - \mathcal{L}^-_{Q_v}(g)$
for all $v$ such that $Q_v \subseteq Z$:
\begin{align}
   & \mathcal{L}^-_{Q_v}(g') - \mathcal{L}^-_{Q_v}(g) \nonumber \\
    &= \mathbb{E}_{x,x^+, \{x^-_i\}_{\ell}} \bigg[ \chi \{x \in Q_v\} \nonumber \\
&\quad \quad \bigg(\log\bigg({ e^{\beta g'(x)^\top g'(x^+)}+ n_{1,v}e^{\beta d}+ \sum_{x_i^{-}\in  E} e^{\beta g'(x)^\top g'(x^-_{i})}  +\sum_{x_i^{-}\not\in {Q}_v\cup E} e^{\beta g(x)^\top g(x^-_{i})}  } \bigg)\nonumber \\
&\quad \quad - \log\bigg({ e^{\beta g(x)^\top g(x^+)}+n_{1,v}e^{\beta d}+ \sum_{x_i^{-} \in E}e^{\beta g(x)^\top g(x^-_{i})} 
 +\sum_{x_i^{-}\not\in {Q}_v\cup E} e^{\beta g(x)^\top g(x^-_{i})}  } \bigg)\bigg)\bigg] \nonumber \\
 &=  \mathbb{E}_{x,x^+, \{x^-_i\}_{\ell}} \bigg[ \chi \{x \in Q_v\} \nonumber \\
&\quad \quad \bigg(\log\bigg({ e^{\beta g(x)^\top g(x^+)}+ n_{1,v}e^{\beta d}+ e^{-2\beta}\sum_{x_i^{-}\in  E} e^{\beta g(x)^\top g(x^-_{i})}  +\sum_{x_i^{-}\not\in {Q}_v\cup E} e^{\beta g(x)^\top g(x^-_{i})}  } \bigg)\nonumber \\
&\quad \quad - \log\bigg({ e^{\beta g(x)^\top g(x^+)}+n_{1,v}e^{\beta d}+ \sum_{x_i^{-} \in E}e^{\beta g(x)^\top g(x^-_{i})} 
 +\sum_{x_i^{-}\not\in {Q}_v\cup E} e^{\beta g(x)^\top g(x^-_{i})}  } \bigg)\bigg)\bigg] \label{lb} \\
  &\leq 0
\end{align}
% where \eqref{lb} follows by definition of $C$. So, we have
Thus, we have
\begin{align}
    \mathcal{L}^-(g') - \mathcal{L}^-(g) &= \sum_{v \in \uH_{d}} \mathcal{L}^-_{Q_v}(g') - \mathcal{L}^-_{Q_v}(g)  +  \mathcal{L}^-_{E_v}(g') -\mathcal{L}^-_{E_v}(g) \nonumber \\
    &\leq \sum_{v\in \mathcal{H}_d: Q_v \setminus Z = Q_v} \mathcal{L}^-_{Q_v}(g') - \mathcal{L}^-_{Q_v}(g)  + \sum_{v \in \uH_{d}}  \mathcal{L}^-_{E_v}(g') -\mathcal{L}^-_{E_v}(g) 
\end{align}

For each $v \in \mathcal{H}_d: Q_v \setminus Z = Q_v$, we consider three cases: (1) $n_2=0$, (2) $n_2>0,n_{1,v=0}$, and (3) $n_2>0,n_{1,v}>0$. In particular we decompose $\mathcal{L}^-_{Q_v}(g') - \mathcal{L}^-_{Q_v}(g)$ as follows:
\begin{align}
    &\mathcal{L}^-_{Q_v}(g') - \mathcal{L}^-_{Q_v}(g) \nonumber \\
    &=  \mathbb{E}_{x,x^+, \{x^-_i\}_{\ell}}\bigg[  \chi \{x \in Q_v\setminus Z\}\chi\{n_2=0\} \nonumber \\
    &\quad \quad \bigg(\log\bigg({ e^{\beta g'(x)^\top g'(x^+)}+ n_{1,v}e^{\beta d}+ \sum_{x_i^{-}\in  E} e^{\beta g'(x)^\top g'(x^-_{i})}  +\sum_{x_i^{-}\not\in {Q}_v\cup E} e^{\beta g(x)^\top g(x^-_{i})}  } \bigg) \nonumber \\
    &\quad \quad  - \log\bigg({ e^{\beta g(x)^\top g(x^+)}+n_{1,v}e^{\beta d}+ \sum_{x_i^{-} \in E}e^{\beta g(x)^\top g(x^-_{i})} 
 +\sum_{x_i^{-}\not\in {Q}_v\cup E} e^{\beta g(x)^\top g(x^-_{i})}  } \bigg) \bigg)\bigg] \nonumber \\
 &\quad +  \mathbb{E}_{x,x^+, \{x^-_i\}_{\ell}} \bigg[ \chi \{x \in Q_v\setminus Z\}\chi\{n_2>0,n_{1,v}=0\} \nonumber \\
 &\quad\quad \bigg(\log\bigg({ e^{\beta g'(x)^\top g'(x^+)}+ n_{1,v}e^{\beta d}+ \sum_{x_i^{-}\in  E} e^{\beta g'(x)^\top g'(x^-_{i})}  +\sum_{x_i^{-}\not\in {Q}_v\cup E} e^{\beta g(x)^\top g(x^-_{i})}  } \bigg) \nonumber \\
    % &\quad \quad  - \mathbb{E}_{x,x^+, \{x^-_i\}_{\ell}}  \chi \{x \in Q_v\} \chi\{n_2>0,n_{1,v}=0\} \nonumber \\
    &\quad \quad - \log\bigg({ e^{\beta g(x)^\top g(x^+)}+n_{1,v}e^{\beta d}+ \sum_{x_i^{-} \in E}e^{\beta g(x)^\top g(x^-_{i})} 
 +\sum_{x_i^{-}\not\in {Q}_v\cup E} e^{\beta g(x)^\top g(x^-_{i})}  } \bigg)\bigg)\bigg] \nonumber \\
 &\quad  +  \mathbb{E}_{x,x^+, \{x^-_i\}_{\ell}}\bigg[  \chi \{x \in Q_v\setminus Z\}\chi\{n_2>0,n_{1,v}>0\} \nonumber \\
 &\quad\quad \quad \bigg( \log\bigg({ e^{\beta g'(x)^\top g'(x^+)}+ n_{1,v}e^{\beta d}+ \sum_{x_i^{-}\in  E} e^{\beta g'(x)^\top g'(x^-_{i})}  +\sum_{x_i^{-}\not\in {Q}_v\cup E} e^{\beta g(x)^\top g(x^-_{i})}  } \bigg) \nonumber \\
    % &\quad \quad - \mathbb{E}_{x,x^+, \{x^-_i\}_{\ell}}  \chi \{x \in Q_v\} \chi\{n_2>0,n_{1,v}>0\} \nonumber \\
    &\quad \quad \quad - \log\bigg({ e^{\beta g(x)^\top g(x^+)}+n_{1,v}e^{\beta d}+ \sum_{x_i^{-} \in E}e^{\beta g(x)^\top g(x^-_{i})} 
 +\sum_{x_i^{-}\not\in {Q}_v\cup E} e^{\beta g(x)^\top g(x^-_{i})}  } \bigg)\bigg)\bigg] \nonumber 
 \end{align}
Likewise, for each $v\in \mathcal{H}_d$, we decompose $\mathcal{L}^-_{E_v}(g') - \mathcal{L}^-_{E_v}(g)$ as:
 \begin{align}
     &\mathcal{L}^-_{E_v}(g') - \mathcal{L}^-_{E_v}(g) \nonumber \\
 &=  \mathbb{E}_{x,x^+, \{x^-_i\}_{\ell}} \bigg[ \chi \{x \in E_v\}\chi\{n_4=0\} \nonumber\\
 &\quad \quad  \bigg(\log\bigg({ e^{\beta g'(x)^\top g'(x^+)}+ n_{3,v}e^{\beta d}+ \sum_{x_i^{-}\in Q} e^{\beta g'(x)^\top g'(x^-_{i})}  +\sum_{x_i^{-}\not\in {E}_v\cup Q} e^{\beta g(x)^\top g(x^-_{i})}  } \bigg) \nonumber \\
    &\quad \quad  -  \log\bigg({ e^{\beta g(x)^\top g(x^+)}+n_{3,v}e^{\beta d}+ \sum_{x_i^{-} \in Q}e^{\beta g(x)^\top g(x^-_{i})} 
 +\sum_{x_i^{-}\not\in {E}_v\cup Q} e^{\beta g(x)^\top g(x^-_{i})}  } \bigg)\bigg)\bigg] \nonumber \\
 &\quad +  \mathbb{E}_{x,x^+, \{x^-_i\}_{\ell}} \bigg[ \chi \{x \in E_v\}\chi\{n_4>0,n_{3,v}=0\} \nonumber \\
 &\quad\quad \bigg(\log\bigg({ e^{\beta g'(x)^\top g'(x^+)}+ n_{3,v}e^{\beta d}+ \sum_{x_i^{-}\in Q} e^{\beta g'(x)^\top g'(x^-_{i})}  +\sum_{x_i^{-}\not\in {E}_v\cup Q} e^{\beta g(x)^\top g(x^-_{i})}  } \bigg) \nonumber \\
    % &\quad \quad  - \mathbb{E}_{x,x^+, \{x^-_i\}_{\ell}}  \chi \{x \in E_v\} \chi\{n_4>0,n_{3,v}=0\} \nonumber \\
    &\quad \quad \quad - \log\bigg({ e^{\beta g(x)^\top g(x^+)}+n_{3,v}e^{\beta d}+ \sum_{x_i^{-} \in Q}e^{\beta g(x)^\top g(x^-_{i})} 
 +\sum_{x_i^{-}\not\in {E}_v\cup Q} e^{\beta g(x)^\top g(x^-_{i})}  } \bigg)\bigg)\bigg] \nonumber \\
 &\quad +  \mathbb{E}_{x,x^+, \{x^-_i\}_{\ell}} \bigg[ \chi \{x \in E_v\}\chi\{n_4>0,n_{3,v}>0\} \nonumber \\
 &\quad\quad \bigg( \log\bigg({ e^{\beta g'(x)^\top g'(x^+)}+ n_{3,v}e^{\beta d}+ \sum_{x_i^{-}\in Q} e^{\beta g'(x)^\top g'(x^-_{i})}  +\sum_{x_i^{-}\not\in {E}_v\cup Q} e^{\beta g(x)^\top g(x^-_{i})}  } \bigg) \nonumber \\
    % &\quad \quad - \mathbb{E}_{x,x^+, \{x^-_i\}_{\ell}}  \chi \{x \in E_v\} \chi\{n_2>0,n_{1,v}>0\} \nonumber \\
    &\quad \quad \quad - \log\bigg({ e^{\beta g(x)^\top g(x^+)}+n_{3,v}e^{\beta d}+ \sum_{x_i^{-} \in Q}e^{\beta g(x)^\top g(x^-_{i})} 
 +\sum_{x_i^{-}\not\in {E}_v\cup Q} e^{\beta g(x)^\top g(x^-_{i})}  } \bigg)\bigg)\bigg] \nonumber
\end{align}
Thus, for each $v$ in \eqref{gn}, we need to upper bound a total of six terms. 
We consider each of these six terms individually, starting with the three terms with  $\chi\{x\in Q_v\setminus Z\}$ factors.

\begin{enumerate}
\item $x\in Q_v\setminus Z, n_2 = 0$. 

% increases by at most $2\beta $ due to positive pair, and that increase happens with total probability $\Delta_f^{new}$ after summing over all $v$. So we need to lower bound the other terms in the log. The bad case is if $n_{1,v}$ is zero, but this happens with small probability. Use:
In this case, we have
\begin{align}
    % \tilde{L}_{\text{neg}}(g') -\tilde{L}_{\text{neg}}(g) 
    ((1))&:=\mathbb{E}_{x,x^+, \{x^-_i\}_{\ell}} \bigg[ \chi \{x \in Q_v\setminus Z\}\chi \{n_2 = 0\} \nonumber \\
    &\quad \bigg(\log\bigg({ e^{\beta g'(x)^\top g'(x^+)}+ n_{1,v}e^{\beta d}  +\sum_{x_i^{-}\not\in {Q}_v\cup E} e^{\beta g(x)^\top g(x^-_{i})}  } \bigg) \nonumber \\
    &\quad -  \log\bigg({ e^{\beta g(x)^\top g(x^+)}+n_{1,v}e^{\beta d}
 +\sum_{x_i^{-}\not\in {Q}_v\cup E} e^{\beta g(x)^\top g(x^-_{i})}  } \bigg)\bigg) \bigg] \label{dif1} \\
 &\leq \mathbb{E}_{x,x^+, \{x^-_i\}_{\ell}} \bigg[ \chi \{x \in Q_v\setminus Z\} \chi \{n_2 = 0\} \nonumber \\
 &\quad \bigg(\log\bigg({ e^{\beta g'(x)^\top g'(x^+)}+ n_{1,v}e^{\beta d}   } \bigg) - \log\bigg({ e^{\beta g(x)^\top g(x^+)}+n_{1,v}e^{\beta d}  } \bigg)\bigg) \bigg] \label{subm} \\
  &= \mathbb{E}_{x,x^+, \{x^-_i\}_{\ell}} \bigg[ \chi \{x \in Q_v\setminus Z,x^+ \in E, n_2 = 0\}  \log\bigg(\frac{ e^{2\beta }e^{\beta g(x)^\top g(x^+)}+ n_{1,v}e^{\beta d}   }{ e^{\beta g(x)^\top g(x^+)}+n_{1,v}e^{\beta d}  } \bigg) \bigg] \label{vv} \\
 &\leq \mathbb{E}_{x,x^+, \{x^-_i\}_{\ell}} \bigg[ \chi \{x \in Q_v\setminus Z,x^+ \in E,n_2 = 0\}\log\bigg(\frac{ e^{\beta d}+ n_{1,v}e^{\beta d}   }{ e^{\beta (d-2)}+n_{1,v}e^{\beta d}  } \bigg) \bigg] \label{hh} \\
 &= \mathbb{E}_{x,x^+, \{x^-_i\}_{\ell}} \bigg[ \chi \{x \in Q_v\setminus Z,x^+ \in E,n_2 = 0\} \log\bigg(\frac{1+ n_{1,v}   }{e^{-2\beta }+n_{1,v}} \bigg) \bigg]\nonumber \\
 &=  \mathbb{E}_{x,x^+, \{x^-_i\}_{\ell}} \bigg[ \chi \{x \in Q_v\setminus Z,x^+ \in E,n_2 = 0, n_{1,v}=0 \} \log\bigg(\frac{1+ n_{1,v}   }{e^{-2\beta }+n_{1,v}} \bigg) \bigg] \nonumber \\
 &\quad + \mathbb{E}_{x,x^+, \{x^-_i\}_{\ell}} \bigg[ \chi \{x \in Q_v\setminus Z,x^+ \in E,n_2 = 0,n_{1,v}>0 \} \log\bigg(\frac{1+ n_{1,v}   }{e^{-2\beta }+n_{1,v}} \bigg) \bigg] \nonumber \\
 &\leq 2 \beta \mathbb{E}_{x,x^+, \{x^-_i\}_{\ell}} \bigg[ \chi \{x \in Q_v\setminus Z,x^+ \in E,n_2 = 0,n_{1,v}=0 \} \bigg]\nonumber \\
 &\quad + \log(2 ) \mathbb{E}_{x,x^+, \{x^-_i\}_{\ell}} \bigg[ \chi \{x \in Q_v\setminus Z,x^+ \in E,n_2 = 0,n_{1,v}>0 \} \bigg] \label{k}
\end{align}
where \eqref{subm} follows by the submodularity of the $\log(\cdot)$ function and the fact that $g'(x)^\top g'(x^+)\geq g(x)^\top g(x^+)$ by construction of $g'$, \eqref{vv} follows by the fact that if $x\in Q_v$, then $g'(x)^\top g'(x^+)= g(x)^\top g(x^+)$ for all $x^+\notin B$,  and \eqref{hh} follows since $h(x):= \frac{ax+ c}{x+c}$ is monotonically increasing for $a> 1$.
Next, by the independence of $x_i^-$ from $x$ and $x^+$, 
\begin{align}
    \mathbb{E}_{x,x^+, \{x^-_i\}_{\ell}}   &\big[\chi \{x \in Q_v\setminus Z,x^+ \in E,n_2 = 0, n_{1,v}=0 \}\big]  \nonumber\\
    &=  \mathbb{P}(x \in Q_v\setminus Z \cap x^+ \in B) \mathbb{P}(n_{1,v}=0,n_2=0) \nonumber \\
    &= \mathbb{P}(x \in Q_v\setminus Z \cap x^+ \in B) (1 - \|Q_v\setminus Z\|_\circ - \|E\|_\circ)^\ell \label{ttwo}
    \end{align}
Similarly, for the second term in \eqref{k}, we have
    \begin{align}
    \mathbb{E}_{x,x^+, \{x^-_i\}_{\ell}} &\big[ \chi \{x \in Q_v\setminus Z,x^+ \in E,n_2 = 0,n_{1,v}>0 \} \big] \nonumber \\
    &= \mathbb{P}(x \in Q_v\setminus Z \cap x^+ \in E) \mathbb{P}(n_{1,v}>0|n_2=0)\mathbb{P}(n_2=0) \\
     &\leq \mathbb{P}(x \in Q_v\setminus Z \cap x^+ \in E) \min\left(1, \frac{\ell \|Q_v\setminus Z\|_\circ}{1-\|E\|_\circ}\right) (1- \|E\|_\circ)^\ell \label{oone}
\end{align}
By combining \eqref{oone}, \eqref{ttwo}, and \eqref{k}, we obtain the following upper bound on \eqref{dif1}:
\begin{align}
     ((1))&\leq   \mathbb{P}(x \in Q_v\setminus Z \cap x^+ \in E) \nonumber \\
     &\quad \quad \quad \quad \left( 2 \beta (1 - \|Q_v\setminus Z\|_\circ - \|E\|_\circ)^\ell + \log(2) \min(1, \frac{\ell \|Q_v\setminus Z\|_\circ}{1-\|E\|_\circ}) (1- \|E\|_\circ)^\ell\right).
\end{align}
% which is fine.

\item $x \in Q_v\setminus Z, n_{1,v} = 0, n_2 > 0$ 

% then $\tilde{L}_{\text{neg}}$ increases by at most $2\beta$ (since each term in the argument of the $\log$ can increase by no more than a multiplicative factor of $2\beta$). This happens with probability $(1-\mathcal{D}_g(v))^l l\Vert E\Vert_\circ$, that is, there is a probability $(1-\mathcal{D}_g(v))^l$ that no negative sample is sampled from the same cluster, and a probability $l\delta$ that $n_2 > 0$.  

In this case we have:
    \begin{align}
    % \tilde{L}_{\text{neg}}(g') -\tilde{L}_{\text{neg}}(g) 
  ((2))  &:= \mathbb{E}_{x,x^+, \{x^-_i\}_{\ell}} \bigg[  \chi \{x \in Q_v\setminus Z,n_{1,v} = 0,n_2 > 0\} \nonumber \\
    &\quad \bigg(\log\bigg({ e^{\beta g'(x)^\top g'(x^+)}+ n_{1,v}e^{\beta d} + \sum_{x_i^{-} \in E}e^{\beta g'(x)^\top g'(x^-_{i})}  +\sum_{x_i^{-}\notin {Q}_v\cup E} e^{\beta g(x)^\top g(x^-_{i})}  } \bigg) \nonumber \\
    % &\quad - \mathbb{E}_{x,x^+, \{x^-_i\}_{\ell}}  \chi \{x \in Q_v,n_{1,v} = 0,n_2 > 0\}   \nonumber \\
    &\quad \quad - \log\bigg({ e^{\beta g(x)^\top g(x^+)}+n_{1,v}e^{\beta d}
+ \sum_{x_i^{-} \in E}e^{\beta g(x)^\top g(x^-_{i})}  +\sum_{x_i^{-}\not\in {Q}_v\cup E} e^{\beta g(x)^\top g(x^-_{i})}  } \bigg) \bigg) \bigg] \label{dif2} \\
 &=  \mathbb{E}_{x,x^+, \{x^-_i\}_{\ell}}\bigg[  \chi \{x \in Q_v\setminus Z,n_{1,v} = 0,n_2 > 0\} \nonumber\\
 &\quad \bigg(\log\bigg({ e^{\beta g'(x)^\top g'(x^+)} + \sum_{x_i^{-} \in E}e^{\beta g'(x)^\top g'(x^-_{i})} +\sum_{x_i^{-}\notin {Q}_v\cup E} e^{\beta g(x)^\top g(x^-_{i})}  } \bigg) \nonumber \\
    &\quad \quad -  \log\bigg({ e^{\beta g(x)^\top g(x^+)}
 + \sum_{x_i^{-} \in E}e^{\beta g(x)^\top g(x^-_{i})} +\sum_{x_i^{-}\notin {Q}_v\cup E} e^{\beta g(x)^\top g(x^-_{i})}  } \bigg)\bigg)\bigg]  \\
  &\leq  \mathbb{E}_{x,x^+, \{x^-_i\}_{\ell}} \bigg[ \chi \{x \in Q_v\setminus Z,n_{1,v} = 0,n_2 > 0\}\nonumber \\
  &\quad \quad \bigg(\log\bigg({ e^{\beta g'(x)^\top g'(x^+)} + \sum_{x_i^{-} \in E}e^{\beta g'(x)^\top g'(x^-_{i})} } \bigg) \nonumber \\
  &\quad \quad \quad \quad \quad -  \log\bigg({ e^{\beta g(x)^\top g(x^+)}
 + \sum_{x_i^{-} \in E}e^{\beta g(x)^\top g(x^-_{i})}  } \bigg)\bigg)\bigg] \label{nonn}  
 \end{align}
 where \eqref{nonn} follows by the submodularity of the $\log()$ function and the facts that $g'(x)^\top g'(x^+)\geq g(x)^\top g(x^+)$ and $g'(x)^\top g'(x^-_i)\geq g(x)^\top g(x^-_i)$ for all $x \in Q_v \setminus Z, x_{i}^-\in {E}$ by definition of $g'$, $C$ and $E$. Next, 
 \begin{align}
((2)) &\leq 2 \beta  \mathbb{P}(x \in Q_v\setminus Z) \mathbb{E}_{ \{x^-_i\}_{\ell}}  [\chi \{n_{1,v} = 0\}\chi \{n_2 > 0\}] \label{non} \\
 &= 2 \beta  \mathbb{P}(x \in Q_v\setminus Z) \mathbb{E}_{ \{x^-_i\}_{\ell}} [ \chi \{\cap_i \{x_{i}^- \notin Q_v\setminus Z\}\}\chi \{\cup_i \{x_{i}^- \in E\}\} ] \nonumber \\
 &=  2 \beta \mathbb{P}(x \in Q_v\setminus Z) \mathbb{P} ( \cup_i \{x_{i}^- \in E\} | \cap_i \{x_{i}^- \notin Q_v\setminus Z\})\mathbb{P}(\cap_i \{x_{i}^- \notin Q_v\setminus Z\}) \nonumber \\
 &\leq 2 \beta  \mathbb{P}(x \in Q_v\setminus Z) \mathbb{P}(\cap_i \{x_{i}^- \notin Q_v\setminus Z\}) \min\left( 1, \sum_{i=1}^\ell \mathbb{P} ( x_{i}^- \in E |  x_{i}^- \notin Q_v\setminus Z) \right) \label{ww} \\
 &= 2 \beta  \|Q_v\setminus Z\|_\circ (1 - \|Q_v\setminus Z\|_\circ)^\ell \min\left(1, \frac{\ell \|E\|_\circ}{1- \|Q_v\setminus Z\|_\circ}\right) \nonumber 
 % &= 2 \beta \ell  \log(\ell+1)  \|E\| \|Q_v\|_\circ (1 - \|Q_v\|)^{\ell-1} .
\end{align}
% where \eqref{}
 where \eqref{non} follows since $g'(x)^\top g'(x^+) - g(x)^\top g(x^+) \leq 2$ and $g'(x)^\top g'(x^-_i) - g(x)^\top g(x^-_i)\leq 2$ for all $x, x_{i}^-$, and \eqref{ww} follows by a union bound.

\item $x\in Q_v\setminus Z, n_1>0, n_2 > 0$.

In this case, we have
\begin{align}
((3))&:= \mathbb{E} \bigg[\chi \{x \in Q_v\setminus Z,n_{1,v} > 0,n_2 > 0\} \nonumber \\
&\quad \bigg(\log\bigg({ e^{\beta g'(x)^\top g'(x^+)}+ n_{1,v}e^{\beta d} + \sum_{x_i^{-} \in E}e^{\beta g'(x)^\top g'(x^-_{i})}  +\sum_{x_i^{-}\not\in {Q}_v\cup E} e^{\beta g(x)^\top g(x^-_{i})}  } \bigg) \nonumber \\
    % &\quad - \mathbb{E} \bigg[\chi \{x \in Q_v,n_{1,v} > 0,n_2 > 0\} \nonumber \\
    &\quad  -\log\bigg({ e^{\beta g(x)^\top g(x^+)}+n_{1,v}e^{\beta d}
+ \sum_{x_i^{-} \in E}e^{\beta g(x)^\top g(x^-_{i})}  +\sum_{x_i^{-}\not\in {Q}_v\cup E} e^{\beta g(x)^\top g(x^-_{i})}  } \bigg) \bigg) \bigg] \label{dif3} \\
&\leq \mathbb{E} \bigg[ \chi \{x \in Q_v\setminus Z,n_{1,v} > 0,n_2 > 0\} \nonumber \\
&\quad \quad \bigg( \log\bigg({ e^{\beta g'(x)^\top g'(x^+)}+ n_{1,v}e^{\beta d} + e^{2\beta} \sum_{x_i^{-} \in E}e^{\beta g(x)^\top g(x^-_{i})}   } \bigg) \nonumber \\
    % &\quad - \mathbb{E} \left[\chi \{x \in Q_v,n_{1,v} > 0,n_2 > 0\}  
    &\quad \quad \quad  - \log\bigg({ e^{\beta g(x)^\top g(x^+)}+n_{1,v}e^{\beta d}
+ \sum_{x_i^{-} \in E}e^{\beta g(x)^\top g(x^-_{i})}    } \bigg) \bigg)\bigg] \label{leb} 
\end{align}
where \eqref{leb} follows by the submodularity of $\log()$ and the facts that $g'(x)^\top g'(x^+)\geq g(x)^\top g(x^+)$ and $g'(x)^\top g'(x^-_i)= g(x)^\top g(x^-_i) + 2$ for all $x \in Q_v \setminus Z, x_{i}^-\in {E}$ by definition of $g'$, $Z$ and $E$.
  Continuing, we obtain
\begin{align}
((3))&\leq \mathbb{E} \bigg[ \chi \{x \in Q_v\setminus Z,n_{1,v} > 0,n_2 > 0\} \nonumber \\
&\quad \quad \bigg(\log\bigg({ e^{\beta g'(x)^\top g'(x^+)}+ (n_{1,v}+n_2)e^{\beta d}  } \bigg)\nonumber \\
&\quad \quad \quad - \log\bigg({ e^{\beta g(x)^\top g(x^+)}+n_{1,v}e^{\beta d} + n_2e^{\beta (d-2)}
  } \bigg)\bigg) \bigg] \label{lebb}  \\
  &= \mathbb{E} \bigg[ \chi \{x \in Q_v\setminus Z,n_{1,v} > 0,n_2 > 0,  x^+ \notin E \} \nonumber \\
  &\quad \quad\quad  \bigg(\log\bigg({ e^{\beta g'(x)^\top g'(x^+)}+ (n_{1,v}+n_2)e^{\beta d}  } \bigg) \nonumber \\
  &\quad \quad \quad - \log\bigg({ e^{\beta g(x)^\top g(x^+)}+n_{1,v}e^{\beta d} + n_2e^{\beta (d-2)}
  } \bigg)\bigg)\bigg] \nonumber \\
  &\quad +  \mathbb{E} \bigg[ \chi \{x \in Q_v\setminus Z,n_{1,v} > 0,n_2 > 0,  x^+ \in E\} \nonumber \\
  &\quad \quad\quad  \bigg(\log\bigg({ e^{\beta g'(x)^\top g'(x^+)}+ (n_{1,v}+n_2)e^{\beta d}  } \bigg) \nonumber \\
   &\quad \quad\quad \quad - \log\bigg({ e^{\beta g(x)^\top g(x^+)}+n_{1,v}e^{\beta d} + n_2e^{\beta (d-2)}
  } \bigg)\bigg)\bigg] 
  \end{align}
  % Denote by $(*)$ the expression in \eqref{dif3}.
  where
  \eqref{lebb} follows since $h(x):= \frac{a + cx}{b+x}$ is an increasing function of $x$ for $x>0, c > \frac{a}{b}$ (here, $\frac{a}{b} \leq \frac{2}{1 + e^{-2\beta}}\leq 2$ and $c = e^{2\beta}> 2$). Note that $x\in Q_v, x^+\in E \implies g'(x)^\top g'(x^+) = g(x)^\top g(x^+) + 2$, and $x\in Q_v, x^+\notin E \implies g'(x)^\top g'(x^+) = g(x)^\top g(x^+) $. Using this we find
  \begin{align}
  ((3)) &\leq \mathbb{E}  \left[\chi \{x \in Q_v\setminus Z,n_{1,v} > 0,n_2 > 0, x^+ \notin E \}  \log\bigg({  \frac{n_{1,v}+n_2}{n_{1,v}} } \bigg) \right]\nonumber \\
    % &\quad - \mathbb{E} \left[\chi \{x \in Q_v,n_{1,v} > 0,n_2 > 0,  g'(x)^\top g'(x^+) =  g(x)^\top g(x^+)\}   \log\bigg({ n_{1,v}e^{\beta d} + n_2e^{\beta (d-2)}
  % } \bigg) \right] \\ 
  &\quad +  \mathbb{E}  \left[\chi \{x \in Q_v\setminus Z,n_{1,v} > 0,n_2 > 0,   x^+ \in B  \}  \log\bigg(\frac{  n_{1,v}+n_2+1  }{n_{1,v} } \bigg) \right]\nonumber \\
  &\leq  \mathbb{E}  \left[\chi \{x \in Q_v\setminus Z,n_{1,v} > 0,n_2 > 0,  x^+ \notin B\}  \frac{2n_2}{n_{1,v}+1} \right]\nonumber \\
    % &\quad - \mathbb{E} \left[\chi \{x \in Q_v,n_{1,v} > 0,n_2 > 0,  g'(x)^\top g'(x^+) =  g(x)^\top g(x^+)\}   \log\bigg({ n_{1,v}e^{\beta d} + n_2e^{\beta (d-2)}
  % } \bigg) \right] \\
  &\quad +  \mathbb{E}  \left[\chi \{x \in Q_v\setminus Z,n_{1,v} > 0,n_2 > 0,   x^+ \in E\}  \frac{  2n_2+2  }{n_{1,v}+1 } \right]\label{vv} \\
  &= \mathbb{E}  \left[\chi \{x \in Q_v\setminus Z,n_{1,v} > 0,n_2 > 0\}  \frac{2n_2}{n_{1,v}+1}  \right]\nonumber \\
    % &\quad - \mathbb{E} \left[\chi \{x \in Q_v,n_{1,v} > 0,n_2 > 0,  g'(x)^\top g'(x^+) =  g(x)^\top g(x^+)\}   \log\bigg({ n_{1,v}e^{\beta d} + n_2e^{\beta (d-2)}
  % } \bigg) \right] \\
  &\quad +  \mathbb{E}  \left[\chi \{x \in Q_v\setminus Z,n_{1,v} > 0,n_2 > 0,   x^+ \in B  \}  \frac{  2  }{n_{1,v}+1 } \right]\nonumber \\
  &= \mathbb{E}  \left[\chi \{x \in Q_v\setminus Z,n_{1,v} > 0,n_2 > 0\}  \frac{2n_2}{n_{1,v}+1}  \right]\nonumber \\
  % &\quad + \mathbb{E}  \left[\chi \{x \in Q_v,n_{1,v} > 0,n_2 > 0, Q_v\cap B\neq \emptyset \}  \frac{2n_2}{n_{1,v}+1}  \right]\nonumber \\
  &\quad +  \mathbb{E}  \left[\chi \{x \in Q_v\setminus Z,n_{1,v} > 0,n_2 > 0, x^+ \in E\}  \frac{  2  }{n_{1,v}+1 } \right]\label{4terms}
  % &\quad +  \mathbb{E}  \left[\chi \{x \in Q_v,n_{1,v} > 0,n_2 > 0,  g'(x)^\top g'(x^+) \neq  g(x)^\top g(x^+), Q_v\cap B\neq \emptyset \}  \frac{  2  }{n_{1,v}+1 } \right]\label{4terms}
\end{align}
where \eqref{vv} follows using the inequality $\log(1+x)\leq x$. 
Thus we are left with two terms in \eqref{4terms}.
For the first term  we have (ignoring notation overload, as after the first line, $n_{1,v}$ and $n_2$ change from random variables to dummy variables):
% For the first term ,
\begin{align}
    &\mathbb{E}  \left[\chi \{x \in Q_v\setminus Z,n_{1,v} > 0,n_2 > 0,Q_v\cap E= \emptyset \}  \frac{2n_2}{n_{1,v}+1} \right] \nonumber \\
    &= \mathbb{P}(x \in Q_v\setminus Z) \sum_{n_{1,v},n_2: n_{1,v}+n_2\leq \ell }  \chi \{n_{1,v} > 0,n_2 > 0\} \binom{\ell}{n_{1,v}~n_2~\ell\!-\!n_{1,v}\!-\!n_2} \nonumber \\
    &\quad \quad \quad \quad  \Vert Q_v\setminus Z\Vert_\circ^{n_{1,v}}\Vert E\Vert_\circ^{n_2}(1-\|Q_v\setminus Z\|_\circ-\|E\Vert_\circ)^{\ell-n_{1,v}-n_2}\frac{2n_2}{n_{1,v}+1} \\
    &= 2\Vert Q_v\setminus Z\Vert_\circ\sum_{n_{1,v},n_2: n_{1,v},n_2>0, n_{1,v}+n_2\leq \ell}  \binom{\ell}{n_{1,v}~n_2~\ell\!-\!n_{1,v}\!-\!n_2} \nonumber \\
    &\quad \quad \quad \quad \Vert Q_v\setminus Z\Vert_\circ^{n_{1,v}}\Vert E\Vert_\circ^{n_2}( 1-\|Q_v \setminus Z\|_\circ-\|E\Vert_\circ)^{\ell-n_{1,v}-n_2}\frac{n_2}{n_{1,v}+1} \\
    &= 2 \Vert Q_v\setminus Z\Vert_\circ \frac{\|E\|_\circ}{\| Q_v\setminus Z\|_\circ} \sum_{n_{1,v},n_2: n_{1,v},n_2>0, n_{1,v}+n_2\leq \ell}    \binom{\ell}{n_{1,v}+1~n_2-1~\ell\!-\!n_{1,v}\!-\!n_2} \nonumber \\
    &\quad \quad \quad \quad \Vert Q_v\setminus Z\Vert_\circ^{n_{1,v}+1}\Vert E\Vert_\circ^{n_2-1}( 1-\|Q_v\setminus Z\|_\circ-\|E\Vert_\circ)^{\ell- n_{1,v}-n_2} \\
    &\leq 2 \Vert Q_v\setminus Z\Vert_\circ \frac{\|E\|_\circ}{\|Q_v \setminus Z\|_\circ} \nonumber \\
    &\quad \sum_{n_{1,v}+1,n_2-1: n_{1,v}+1,n_2-1\geq0, n_{1,v}+1+n_2-1\leq \ell}   \binom{\ell}{n_{1,v}\!+\!1~n_2\!-\!1~\ell\!-(n_{1,v}\!+\!1) \!-\! (n_2\!-\!1)}\nonumber \\
    &\quad \quad \quad \quad \quad \quad   \quad \quad \quad \Vert Q_v\setminus Z\Vert_\circ^{n_{1,v}+1}\Vert E\Vert_\circ^{n_2-1}( 1\!-\!\|Q_v\setminus Z\|_\circ\!-\!\|E\Vert_\circ)^{\ell - (n_{1,v}+1) - (n_2-1)}  \label{sq}\\
    &= 2{\|E\|_\circ} ( \Vert Q_v\setminus Z\Vert_\circ + \Vert E\Vert_\circ + 1-\|Q_v\setminus Z\|_\circ-\|E\Vert_\circ)^\ell  \nonumber \\
    &= {2 {\|E\|_\circ}}   \nonumber 
\end{align}
where in \eqref{sq} we have added terms to the sum to complete the trinomial expansion, and the last equality follows since $Q_v$ and $E$ are disjoint.

Now we need to consider the last term in \eqref{4terms}, which corresponds to the case wherein the positive inner products are not equal for $g$ and $g'$. For this term, we simply have
{\begin{align}
&\mathbb{E}  \left[\chi \{x \in Q_v\setminus Z,n_{1,v} > 0,n_2 > 0,  x^+\in E\}  \frac{  2  }{n_{1,v}+1 } \right]\nonumber \\
  &\leq \mathbb{P}  \left[x \in Q_v\setminus Z, x^+\in E\right]
\end{align}
}

In total, for the case $n_{1,v}>0,n_2>0$ and $x\in Q_v\setminus Z$, we have
\begin{align}
%     &\mathbb{E}  \bigg[\chi \{x \in Q_v\setminus Z,n_{1,v} > 0,n_2 > 0\} \nonumber \\
%     &\quad \quad \bigg( \log\bigg({ e^{\beta g'(x)^\top g'(x^+)}+ n_{1,v}e^{\beta d} + \sum_{x_i^{-} \in E}e^{\beta g'(x)^\top g'(x^-_{i})}  +\sum_{x_i^{-}\notin {Q}_v\cup E} e^{\beta g(x)^\top g(x^-_{i})}  } \bigg) \nonumber \\
%     &\quad \quad - \log\bigg({ e^{\beta g(x)^\top g(x^+)}+n_{1,v}e^{\beta d}
% + \sum_{x_i^{-} \in E}e^{\beta g(x)^\top g(x^-_{i})}  +\sum_{x_i^{-}\notin {Q}_v\cup E} e^{\beta g(x)^\top g(x^-_{i})}  } \bigg)\bigg) \bigg] \nonumber \\
((3))&\leq  2 \|E\| +   \mathbb{P}  \left[x \in Q_v\setminus Z, x^+\in E\right]
% 2 \|E\|_\circ \chi\{ Q_v\cap B = \emptyset \} + 2 \chi\{\|Q_v\!\setminus\! B\|_\circ \geq \|Q_v\|_\circ/2\}({\Vert Q_v\cap E\Vert_\circ+ 2\Vert B\setminus Q_v\Vert_\circ }) \nonumber \\
%     &\quad +  4 \chi\{\|Q_v\!\setminus\! B\|_\circ \geq \|Q_v\|_\circ/2\}\frac{\Vert Q_v\cap E\Vert_\circ^2 + 2 \Vert Q_v\cap E\Vert_\circ\Vert B\setminus Q_v\Vert_\circ}{\Vert Q_v\Vert_\circ} \nonumber \\
%     &\quad + 8 \ell \Vert  Q_v\cap E\Vert_\circ {\Vert B\setminus Q_v\Vert_\circ}  (1 - \Vert  Q_v\Vert_\circ/2)^{\ell} \chi \{ \|Q_v\!\cap\! B\|_\circ \geq \|Q_v\|_\circ/2\} \nonumber \\
%     &\quad +4 \|Q_v\cap B\|_\circ \left(2\frac{\Vert B\setminus Q_v\Vert_\circ}{\Vert Q_v\Vert_\circ}+1\right)\chi \{\|Q_v\!\cap\! B\|_\circ \geq \|Q_v\|_\circ/2\} \nonumber \\
%     &\quad + \chi\{Q_v\cap E= \emptyset \} \mathbb{P}  \left[x \in Q_v, x^+\in E\right] + \chi\{Q_v\cap B\neq \emptyset \} \mathbb{P}  \left[x \in Q_v \cap B, x^+ \notin E \right] 
\end{align}
% All of these terms are $O(\|E\|_\circ)$ and do not depend on $\beta$ so we can make $\beta$ large enough to dominate them.

\item $x \in E_v, n_4 = 0$.

This case is symmetric to  Case 1, so we argue similarly.
\begin{align}
   ((4))&:=   \mathbb{E}_{x,x^+, \{x^-_i\}_{\ell}} \bigg[ \chi \{x \in E_v\}\chi\{n_4=0\} \nonumber \\
   &\quad \bigg(\log\bigg({ e^{\beta g'(x)^\top g'(x^+)}+ n_{3,v}e^{\beta d}+ \sum_{x_i^{-}\in Q} e^{\beta g'(x)^\top g'(x^-_{i})}  +\sum_{x_i^{-}\not\in {E}_v\cup Q} e^{\beta g(x)^\top g(x^-_{i})}  } \bigg) \nonumber \\
    % &\quad  - \mathbb{E}_{x,x^+, \{x^-_i\}_{\ell}}  \chi \{x \in E_v\} \chi\{n_4=0\} \nonumber \\
    &\quad \quad - \log\bigg({ e^{\beta g(x)^\top g(x^+)}+n_{3,v}e^{\beta d}+ \sum_{x_i^{-} \in Q}e^{\beta g(x)^\top g(x^-_{i})} 
 +\sum_{x_i^{-}\not\in {E}_v\cup Q} e^{\beta g(x)^\top g(x^-_{i})}  } \bigg)\bigg) \bigg] \nonumber \\
    &= \mathbb{E}_{x,x^+, \{x^-_i\}_{\ell}} \bigg[ \chi \{x \in E_v\}\chi\{n_4=0\} \nonumber \\
    &\quad\quad \quad\quad \quad\quad \quad \bigg(\log\bigg({ e^{\beta g'(x)^\top g'(x^+)}+ n_{3,v}e^{\beta d} +\sum_{x_i^{-}\not\in {E}_v\cup Q} e^{\beta g(x)^\top g(x^-_{i})}  } \bigg) \nonumber \\
    &\quad \quad\quad \quad \quad \quad\quad \quad  - \log\bigg({ e^{\beta g(x)^\top g(x^+)}+n_{3,v}e^{\beta d}+\sum_{x_i^{-}\not\in {E}_v\cup Q} e^{\beta g(x)^\top g(x^-_{i})}  } \bigg)\bigg)\bigg] \nonumber \\
    &\leq \mathbb{E}_{x,x^+, \{x^-_i\}_{\ell}} \bigg[ \chi \{x \in E_v\}\chi\{n_4=0\} \log\bigg(\frac{ e^{\beta g'(x)^\top g'(x^+)}+ n_{3,v}e^{\beta d}   }{e^{\beta g(x)^\top g(x^+)}+n_{3,v}e^{\beta d}}\bigg) \bigg] \label{bb}
\end{align}
where \eqref{bb} follows since $ {g'(x)^\top g'(x^+)} \geq { g(x)^\top g(x^+)}$ and $\log()$ is submodular. Next we intersect with the events $\{x^+\in Q\}$ and $\{x^+\notin Q\}$, obtaining
\begin{align}
    ((4)) &\leq  \mathbb{E}_{x,x^+, \{x^-_i\}_{\ell}} \bigg[ \chi \{x \in E_v, n_4=0, x^+ \in Q\} \log\bigg(\frac{ e^{\beta g'(x)^\top g'(x^+)}+ n_{3,v}e^{\beta d}   } { e^{\beta g(x)^\top g(x^+)}+n_{3,v}e^{\beta d} }\bigg)\bigg] \nonumber \\
    &\quad +  \mathbb{E}_{x,x^+, \{x^-_i\}_{\ell}} \bigg[ \chi \{x \in E_v, n_4=0, x^+ \notin Q\} \log\bigg(\frac{ e^{\beta g'(x)^\top g'(x^+)}+ n_{3,v}e^{\beta d}   }{e^{\beta g(x)^\top g(x^+)}+n_{3,v}e^{\beta d}} \bigg)\bigg] \nonumber \\
    &= \mathbb{E}_{x,x^+, \{x^-_i\}_{\ell}} \bigg[ \chi \{x \in E_v, n_4=0, x^+ \in Q\} \log\bigg(\frac{ e^{2 \beta} e^{\beta g(x)^\top g(x^+)}+ n_{3,v}e^{\beta d}   }{ e^{\beta g(x)^\top g(x^+)}+n_{3,v}e^{\beta d} } \bigg)\bigg] \nonumber \\
    &\leq \mathbb{E}_{x,x^+, \{x^-_i\}_{\ell}}  \bigg[\chi \{x \in E_v, n_4=0, x^+ \in Q\} \log\bigg(\frac{ e^{\beta d} + n_{3,v}e^{\beta d}   }{e^{\beta (d-2)}+n_{3,v}e^{\beta d}} \bigg)\bigg]  \nonumber \\
    &= \mathbb{E}_{x,x^+, \{x^-_i\}_{\ell}} \bigg[ \chi \{x \in E_v, n_4=0, x^+ \in Q\} \log\bigg(\frac{ 1 + n_{3,v}  }{e^{-2\beta} + n_{3,v}} \bigg)\bigg] \nonumber \\
    &\leq 2\beta \mathbb{E}_{x,x^+, \{x^-_i\}_{\ell}} \big[ \chi \{x \in E_v, n_4=0, x^+ \in Q, n_{3,v}=0\} \big] \nonumber \\
    &+ \log( 2) \mathbb{E}_{x,x^+, \{x^-_i\}_{\ell}} \big[ \chi \{x \in E_v,n_4=0, x^+\in Q,n_{3,v}>0\} \big]
\end{align}
where
\begin{align}
    &\mathbb{E}_{x,x^+, \{x^-_i\}_{\ell}} \big[ \chi \{x \in E_v, n_4=0, x^+ \in Q, n_{3,v}=0\} \big] \nonumber \\
    &=  \mathbb{P}(x \in E_v \cap x^+ \in Q) \mathbb{P}(n_{3,v}=0, n_4=0) \\
    &= \mathbb{P}(x \in E_v \cap x^+ \in Q) (1 - \|E_v\|_\circ - \|Q\|_\circ)^\ell 
    % \\
    % & \mathbb{E}_{x,x^+, \{x^-_i\}_{\ell}}  \chi \{x \in Q_v\}\chi \{x^+ \in E\} \chi \{n_2 = 0\}  \chi\{ n_{1,v}>0 \} 
    % \nonumber \\
    % &= \mathbb{P}(x \in Q_v \cap x^+ \in B) \mathbb{P}(n_{3,v}>0) \\
    %  &\leq \mathbb{P}(x \in Q_v \cap x^+ \in B) \mathcal{D}_g(v)
\end{align}
\begin{align}
    &\mathbb{E}_{x,x^+, \{x^-_i\}_{\ell}} \big[ \chi \{x \in E_v,n_4=0, x^+\in Q,n_{3,v}>0\} \big] \nonumber \\
    &= \mathbb{P}(x \in E_v \cap x^+ \in Q)  \mathbb{P}(n_{3,v} >0, n_4=0)  \nonumber \\
    &\leq \mathbb{P}(x \in E_v \cap x^+ \in Q) \min\left( 1, \frac{\ell \|E_v\|_\circ}{1 - \|Q\|_\circ}\right)  (1 - \|Q\|_\circ)^\ell
\end{align}
so in total for this case we have 
\begin{align}
    ((4)) &\leq \mathbb{P}(x \in E_v \cap x^+ \in Q) \nonumber \\
    &\quad \quad \quad\quad \quad \left(2\beta (1 - \|E_v\|_\circ - \|Q\|_\circ)^\ell  + \log(2)\min( 1, \frac{\ell \|E_v\|_\circ}{1 - \|Q\|_\circ})  (1 - \|Q\|_\circ)^\ell\right) \nonumber
\end{align}

\item $x \in E_v, n_4>0, n_{3,v} = 0$.

Define $n_{5}:= \sum_{i=1}^\ell \chi\{ x_{i}^- \in Q\setminus Z\}$. %and $n_{6}:= \sum_{i=1}^\ell \chi\{ x_{i}^- \in C\}$. 
We have 
\begin{align}
    ((5)) &=  \mathbb{E}_{x,x^+, \{x^-_i\}_{\ell}} \bigg[ \chi \{x \in E_v,n_4>0,n_{3,v}=0\}\nonumber \\
    &\quad \bigg(\log\bigg({ e^{\beta g'(x)^\top g'(x^+)}+ \sum_{x_i^{-}\in Q} e^{\beta g'(x)^\top g'(x^-_{i})}  +\sum_{x_i^{-}\not\in {E}_v\cup Q} e^{\beta g(x)^\top g(x^-_{i})}  } \bigg) \nonumber \\
    % &\quad  - \mathbb{E}_{x,x^+, \{x^-_i\}_{\ell}}  \chi \{x \in E_v,n_4>0,n_{3,v}=0\} \nonumber \\
    &\quad \quad - \log\bigg({ e^{\beta g(x)^\top g(x^+)}+\sum_{x_i^{-} \in Q}e^{\beta g(x)^\top g(x^-_{i})} 
 +\sum_{x_i^{-}\not\in {E}_v\cup Q} e^{\beta g(x)^\top g(x^-_{i})}  } \bigg)\bigg)\bigg] \nonumber
 \end{align}
 since we are intersecting with the event $\{n_{3,v}=0\}$. Next we  split the negative samples in $Q$ into those in $Q\setminus Z$ and those in $Z$, noting that $g'(x)^\top g'(x_i^-) = g(x)^\top g(x_i^-) - 2$ for $x_i^-\in Z$ and $g'(x)^\top g'(x_i^-) = g(x)^\top g(x_i^-) + 2$ for $x_i^-\in Q\setminus Z$. 
 \begin{align}
((5)) &= \mathbb{E}_{x,x^+, \{x^-_i\}_{\ell}} \bigg[ \chi \{x \in E_v,n_4>0,n_{3,v}=0\}\nonumber \\
 & \quad\bigg(\log\bigg( e^{\beta g'(x)^\top g'(x^+)}+  e^{2\beta }\sum_{x_i^{-} \in Q\setminus Z}e^{\beta g(x)^\top g(x^-_{i})} \nonumber \\
 &\quad \quad \quad \quad \quad +e^{-2\beta }\sum_{x_i^{-} \in  C}e^{\beta g(x)^\top g(x^-_{i})} +\sum_{x_i^{-}\not\in {E}_v\cup Q} e^{\beta g(x)^\top g(x^-_{i})}   \bigg) \nonumber \\
    % &\quad \quad  - \mathbb{E}_{x,x^+, \{x^-_i\}_{\ell}}  \chi \{x \in E_v,n_4>0,n_{3,v}=0\} \nonumber \\
    &\quad -\log\bigg( e^{\beta g(x)^\top g(x^+)}+\sum_{x_i^{-} \in Q\setminus Z}e^{\beta g(x)^\top g(x^-_{i})} +\sum_{x_i^{-} \in  C}e^{\beta g(x)^\top g(x^-_{i})} \nonumber \\
 &\quad \quad \quad \quad \quad
 +\sum_{x_i^{-}\not\in {E}_v\cup Q} e^{\beta g(x)^\top g(x^-_{i})}   \bigg)\bigg) \bigg] \nonumber \\
  &\leq \mathbb{E}_{x,x^+, \{x^-_i\}_{\ell}} \bigg[ \chi \{x \in E_v,n_4>0,n_{3,v}=0\}\nonumber \\
 &\quad \bigg(\log\bigg({ e^{\beta g'(x)^\top g'(x^+)}+  e^{2\beta }\sum_{x_i^{-} \in Q\setminus Z}e^{\beta g(x)^\top g(x^-_{i})} +\sum_{x_i^{-}\not\in {E}_v\cup Q} e^{\beta g(x)^\top g(x^-_{i})}  } \bigg) \nonumber \\ %---
    % &\quad \quad  - \mathbb{E}_{x,x^+, \{x^-_i\}_{\ell}}  \chi \{x \in E_v,n_4>0,n_{3,v}=0\} \nonumber \\
    &\quad -\log\bigg({ e^{\beta g(x)^\top g(x^+)}+\sum_{x_i^{-} \in Q\setminus Z}e^{\beta g(x)^\top g(x^-_{i})} 
 +\sum_{x_i^{-}\not\in {E}_v\cup Q} e^{\beta g(x)^\top g(x^-_{i})}  } \bigg)\bigg) \bigg] \nonumber \\
 &\leq \mathbb{E}_{x,x^+, \{x^-_i\}_{\ell}} \bigg[ \chi \{x \in E_v,n_4>0,n_{3,v}=0\}\nonumber \\
 &\quad \bigg(\log\bigg({ e^{\beta g'(x)^\top g'(x^+)}+  n_{5}e^{\beta d } +\sum_{x_i^{-}\not\in {E}_v\cup Q} e^{\beta g(x)^\top g(x^-_{i})}  } \bigg) \nonumber \\
    % &\quad \quad  - \mathbb{E}_{x,x^+, \{x^-_i\}_{\ell}}  \chi \{x \in E_v,n_4>0,n_{3,v}=0\} \nonumber \\
    &\quad -\log\bigg({ e^{\beta g(x)^\top g(x^+)}+n_5 e^{\beta(d-2)} 
 +\sum_{x_i^{-}\not\in {E}_v\cup Q} e^{\beta g(x)^\top g(x^-_{i})}  } \bigg)\bigg)\bigg] \label{lll} 
 \end{align}
 where \eqref{lll} follows since $h(x)\coloneqq \frac{a+e^{2\beta}x}{b+x}$ is an increasing function of $x$ for $a \leq  e^{2 \beta} b $. Next we  intersect with $\{x^+ \in Q\}$ and $\{x^+ \notin Q\}$ to obtain
 \begin{align}
 ((5)) &\leq   \mathbb{E}_{x,x^+, \{x^-_i\}_{\ell}} \bigg[ \chi \{x \in E_v,n_4>0,n_{3,v}=0, x^+ \in Q\}\nonumber \\
 &\quad \quad \bigg(\log\bigg({  (n_{5}+1)e^{\beta d } +\sum_{x_i^{-}\not\in {E}_v\cup Q} e^{\beta g(x)^\top g(x^-_{i})}  } \bigg) \nonumber \\
    % &\quad \quad  - \mathbb{E}_{x,x^+, \{x^-_i\}_{\ell}}  \chi \{x \in E_v,n_4>0,n_{3,v}=0, x^+ \in Q\} \nonumber \\
    &\quad\quad \quad - \log\bigg({  (n_5+1) e^{\beta(d-2)} 
 +\sum_{x_i^{-}\not\in {E}_v\cup Q} e^{\beta g(x)^\top g(x^-_{i})}  } \bigg)\bigg) \bigg] \nonumber \\
 &\quad + \mathbb{E}_{x,x^+, \{x^-_i\}_{\ell}} \bigg[ \chi \{x \in E_v,n_4>0,n_{3,v}=0, x^+ \notin Q\}\nonumber \\
 &\quad \quad \bigg( \log\bigg({  n_{5}e^{\beta d } +\sum_{x_i^{-}\not\in {E}_v\cup Q} e^{\beta g(x)^\top g(x^-_{i})}  } \bigg)  \nonumber \\
 &\quad \quad \quad - \log\bigg({  n_5 e^{\beta(d-2)} 
 +\sum_{x_i^{-}\not\in {E}_v\cup Q} e^{\beta g(x)^\top g(x^-_{i})}  } \bigg)\bigg)\bigg] \label{ff} 
 % \\
 % &\leq
\end{align}
We have two terms above. %For the first term, we have
% Let $\Lambda:= \{(n_{5}+1)e^{\beta d }+ n_{6}e^{-\beta d } \geq  (n_5+1) e^{\beta(d-2)} +n_6 e^{-\beta(d-2)}\}$. Note that $\Lambda$ holds unless $n_6 \gg n_{5,v}$.
% We split the first term in \eqref{ff} on $\Lambda$.
For the first term,
\begin{align}
   &  \mathbb{E}_{x,x^+, \{x^-_i\}_{\ell}} \bigg[ \chi \{x \in E_v,n_4>0,n_{3,v}=0, x^+ \in Q\}\nonumber \\
 &\quad \quad \bigg(\log\bigg({  (n_{5}+1)e^{\beta d } +\sum_{x_i^{-}\not\in {E}_v\cup Q} e^{\beta g(x)^\top g(x^-_{i})}  } \bigg) \nonumber \\
 &\quad \quad \quad - \log\bigg({  (n_5+1) e^{\beta(d-2)} 
 +\sum_{x_i^{-}\not\in {E}_v\cup Q} e^{\beta g(x)^\top g(x^-_{i})}  } \bigg)\bigg)\bigg] \nonumber \\
 &\leq  \mathbb{E}_{x,x^+, \{x^-_i\}_{\ell}} \bigg[ \chi \{x \in E_v,n_4>0,n_{3,v}=0, x^+ \in Q\} \bigg(\log\bigg(\frac{  (n_{5}+1)e^{\beta d }  }{  (n_5+1) e^{\beta(d-2)}  } \bigg)\bigg)\bigg] \label{fg} \\
  &\leq  2\beta  \mathbb{E}_{x,x^+, \{x^-_i\}_{\ell}} \big[ \chi \{x \in E_v,n_4>0,n_{3,v}=0, x^+ \in Q\}\big]  \nonumber 
\end{align}
% where \eqref{ew} holds by definition of $\Lambda$ and \eqref{fg} follows by the submodularity of the $\log$ and the definition of $\Lambda$. 
Similarly, for the second term in \eqref{ff}, 
% first define $\Lambda':= \{  n_{5}e^{\beta d }+ (n_{6}+1)e^{-\beta d } \geq  n_5 e^{\beta(d-2)} +(n_6+1) e^{-\beta(d-2)} \}$. Note that $\Lambda' \implies n_5 > 0$. 
we have
\begin{align}
   & \mathbb{E}_{x,x^+, \{x^-_i\}_{\ell}} \bigg[ \chi \{x \in E_v,n_4>0,n_{3,v}=0, x^+ \notin Q\}\nonumber \\
 &\quad \quad \bigg( \log\bigg({  n_{5}e^{\beta d }+\sum_{x_i^{-}\not\in {E}_v\cup Q} e^{\beta g(x)^\top g(x^-_{i})}  } \bigg) - \log\bigg({  n_5 e^{\beta(d-2)} 
 +\sum_{x_i^{-}\not\in {E}_v\cup Q} e^{\beta g(x)^\top g(x^-_{i})}  } \bigg)\bigg)\bigg] \nonumber \\
 % &\leq \mathbb{E}_{x,x^+, \{x^-_i\}_{\ell}} \bigg[ \chi \{x \in E_v,n_4>0,n_{3,v}=0, x^+ \notin Q,\Lambda'\}\nonumber \\
 % &\quad \quad \bigg( \log\bigg({  n_{5}e^{\beta d }+ (n_{6}+1)e^{-\beta d } +\sum_{x_i^{-}\not\in {E}_v\cup Q} e^{\beta g(x)^\top g(x^-_{i})}  } \bigg) \nonumber \\
 %    % &\quad \quad  - \mathbb{E}_{x,x^+, \{x^-_i\}_{\ell}}  \chi \{x \in E_v,n_4>0,n_{3,v}=0, x^+ \notin Q\} \nonumber \\
 %    &\quad \quad \quad - \log\bigg({  n_5 e^{\beta(d-2)} +(n_6+1) e^{-\beta(d-2)}
 % +\sum_{x_i^{-}\not\in {E}_v\cup Q} e^{\beta g(x)^\top g(x^-_{i})}  } \bigg)\bigg)\bigg] \nonumber \\
  &\leq \mathbb{E}_{x,x^+, \{x^-_i\}_{\ell}} \bigg[ \chi \{x \in E_v,n_5>0,n_{3,v}=0, x^+ \notin Q\}\log\bigg(\frac{n_5e^{\beta d}}{n_5 e^{\beta (d-2)}}\bigg)\bigg]\nonumber \\
 &= 2\beta  \mathbb{E}_{x,x^+, \{x^-_i\}_{\ell}} \big[ \chi \{x \in E_v,n_5>0,n_{3,v}=0, x^+ \notin Q \}\big]\nonumber % \\ %\\
 % &\quad \quad \bigg( \log\bigg({  n_{5}e^{\beta d }} \bigg) \nonumber \\
 %    % &\quad \quad  - \mathbb{E}_{x,x^+, \{x^-_i\}_{\ell}}  \chi \{x \in E_v,n_4>0,n_{3,v}=0, x^+ \notin Q\} \nonumber \\
 %    &\quad \quad \quad - \log\bigg({  n_5 e^{\beta(d-2)}  } \bigg)\bigg) \nonumber \\
  % &\leq 2\beta  \mathbb{E}_{x,x^+, \{x^-_i\}_{\ell}} \big[ \chi \{x \in E_v,n_4>0,n_{3,v}=0, x^+ \notin Q\} \big] \nonumber %\\
\end{align}
% where the last inequality follows since $\mathcal{C'}\implies n_5>0$.  
By summing the upper bounds on the two terms in \eqref{ff}, we obtain
\begin{align}
    ((5)) &\leq  2\beta  \mathbb{E}_{x,x^+, \{x^-_i\}_{\ell}} \big[ \chi \{x \in E_v,n_4>0,n_{3,v}=0, x^+ \in Q\}\big] \nonumber \\
    &\quad + 2\beta  \mathbb{E}_{x,x^+, \{x^-_i\}_{\ell}} \big[ \chi \{x \in E_v,n_5>0,n_{3,v}=0, x^+ \notin Q\} \big] \nonumber \\
    &\leq 2 \beta \mathbb{P}(x \in E_v)\mathbb{P}(n_4>0, n_{3,v}=0) \nonumber \\
    &= 2 \beta \|E_v\|_\circ (1 - \|E_v\|_\circ)^\ell \min\left(1, \frac{\ell \|Q\|_\circ}{1 - \|E_v\|_\circ}\right).
\end{align}

\item $x \in E_v, n_4>0, n_{3,v} > 0$.

In this case, we argue similarly as in Case 3 to obtain
\begin{align}
     ((6))&= \mathbb{E}_{x,x^+, \{x^-_i\}_{\ell}} \bigg[ \chi \{x \in E_v,n_4>0, n_{3,v}>0\} \nonumber \\
     &\quad \bigg( \log\bigg({ e^{\beta g'(x)^\top g'(x^+)}+ n_{3,v}e^{\beta d}+ \sum_{x_i^{-}\in Q} e^{\beta g'(x)^\top g'(x^-_{i})}  +\sum_{x_i^{-}\not\in {E}_v\cup Q} e^{\beta g(x)^\top g(x^-_{i})}  } \bigg) \nonumber \\
    &\quad \quad  - \log\bigg({ e^{\beta g(x)^\top g(x^+)}+n_{3,v}e^{\beta d}+ \sum_{x_i^{-} \in Q}e^{\beta g(x)^\top g(x^-_{i})} 
 +\sum_{x_i^{-}\not\in {E}_v\cup Q} e^{\beta g(x)^\top g(x^-_{i})}  } \bigg)\bigg)\bigg] \nonumber \\
 &\leq \mathbb{E}_{x,x^+, \{x^-_i\}_{\ell}}  \bigg[ \chi \{x \in E_v,n_4>0, n_{3,v}>0\}  \nonumber \\
 &\quad \bigg(\log\bigg({ e^{\beta g'(x)^\top g'(x^+)}+ n_{3,v}e^{\beta d}+ e^{2\beta }\sum_{x_i^{-}\in Q \setminus Z} e^{\beta g(x)^\top g(x^-_{i})}  +\sum_{x_i^{-}\not\in {E}_v\cup Q} e^{\beta g(x)^\top g(x^-_{i})}  } \bigg) \nonumber \\
    &\quad \quad  -  \log\bigg({ e^{\beta g(x)^\top g(x^+)}+n_{3,v}e^{\beta d}+ \sum_{x_i^{-} \in Q\setminus Z}e^{\beta g(x)^\top g(x^-_{i})} 
 +\sum_{x_i^{-}\not\in {E}_v\cup Q} e^{\beta g(x)^\top g(x^-_{i})}  } \bigg) \bigg)\bigg] \nonumber \\
 &\leq  \mathbb{E}_{x,x^+, \{x^-_i\}_{\ell}} \bigg[ \chi \{x \in E_v,n_4>0, n_{3,v}>0\} \nonumber \\
 &\quad \quad  \bigg(\log\bigg( e^{\beta g'(x)^\top g'(x^+)}+ (n_{3,v}+n_5)e^{\beta d}\bigg) \nonumber \\
 &\quad \quad \quad -\log\bigg({ e^{\beta g(x)^\top g(x^+)}+n_{3,v}e^{\beta d} + n_{5}e^{\beta (d-2)} } \bigg)\bigg)\bigg]\label{rr} %\nonumber 
 % \\
    % &\leq 
\end{align}
where \eqref{rr} follows by  the analogous argument as in \eqref{lebb}.
Next we intersect with  $\{x^+\in Q\}$ and $\{x^+\notin Q\}$. We have
\begin{align}
    ((6)) &\leq \mathbb{E}_{x,x^+, \{x^-_i\}_{\ell}} \bigg[ \chi \{x \in E_v,n_4>0, n_{3,v}>0\} \nonumber \\
 &\quad \quad  \bigg(\log\bigg( e^{\beta g'(x)^\top g'(x^+)}+ (n_{3,v}+n_5)e^{\beta d}\bigg) \nonumber \\
 &\quad \quad \quad -\log\bigg({ e^{\beta g(x)^\top g(x^+)}+n_{3,v}e^{\beta d} + n_{5}e^{\beta (d-2)} } \bigg)\bigg) \bigg]\nonumber \\
 &\leq \mathbb{E}_{x,x^+, \{x^-_i\}_{\ell}} \bigg[ \chi \{x \in E_v,n_4>0, n_{3,v}>0,x^+\in Q\}\log\left(\frac{n_{3,v}+n_5 + 1}{n_{3,v}}\right)\bigg] \nonumber \\
 % &\quad \quad  \bigg(\log\bigg({  (n_{3,v}+n_5 + 1)e^{\beta d}\bigg)  -\log\bigg({  (n_{5}+1)e^{\beta (d-2)} } }\bigg)\bigg) \nonumber \\
  &\quad +\mathbb{E}_{x,x^+, \{x^-_i\}_{\ell}} \bigg[ \chi \{x \in E_v,n_4>0, n_{3,v}>0,x^+\notin Q\} \log\left(\frac{n_{3,v}+n_5}{n_{3,v}}\right)\bigg]\label{tm}\\
 % &\quad \quad  \bigg(\log\bigg({  (n_{3,v}+n_5)e^{2\beta }\bigg)  -\log\bigg({  n_{5} } }\bigg)\bigg) \nonumber \\
 % &\leq 2\mathbb{E}_{x,x^+, \{x^-_i\}_{\ell}}  \chi \{x \in E_v,n_4>0, n_{3,v}>0,x^+\in Q\}\frac{n_{5}+ 1}{n_{3,v}+1} \nonumber \\
 %  &\quad +2\mathbb{E}_{x,x^+, \{x^-_i\}_{\ell}}  \chi \{x \in E_v,n_4>0, n_{3,v}>0,x^+\notin Q\} \frac{n_{5}}{n_{3,5}+1}\nonumber \\
 %  &= 2\mathbb{E}_{x,x^+, \{x^-_i\}_{\ell}}  \chi \{x \in E_v,n_4>0, n_{3,v}>0\}\frac{n_5}{n_{3,v}+1} \nonumber \\
 %  &\quad +2\mathbb{E}_{x,x^+, \{x^-_i\}_{\ell}}  \chi \{x \in E_v,n_4>0, n_{3,v}>0,x^+\in Q\} \frac{1}{n_{3,v}+1}
  &\leq \mathbb{E}_{x,x^+, \{x^-_i\}_{\ell}} \big[ \chi \{x \in E_v,n_{3,v}>0,x^+\in Q\}\log(n_5 + 2) \big] \nonumber \\
  &\quad +\mathbb{E}_{x,x^+, \{x^-_i\}_{\ell}} \big[ \chi \{x \in E_v, n_{3,v}>0,x^+\notin Q\} \log(n_5 + 1) \big] \nonumber \\
  &= \mathbb{P}(x \in E_v,x^+\in Q ) \mathbb{E}_{x,x^+, \{x^-_i\}_{\ell}} \big[ \chi \{ n_{3,v}>0\}\log(n_5 + 2) \big] \nonumber \\
  &\quad +\mathbb{P}(x \in E_v,x^+\notin Q )\mathbb{E}_{x,x^+, \{x^-_i\}_{\ell}} \big[ \chi \{ n_{3,v}>0\} \log({n_5} + 1) \big] \nonumber \\
  &\leq \mathbb{P}(x \in E_v,x^+\in Q )  \log(\mathbb{E} [n_5 + 2])  +\mathbb{P}(x \in E_v,x^+\notin Q )\log(\mathbb{E}[{n_5}+1])\label{gg} \\
  &\leq \mathbb{P}(x \in E_v,x^+\in Q )  \log(\ell \|Q\setminus Z\|_\circ + 2) +\|E_v\|_\circ \log(\ell \|Q\setminus Z\|_\circ+1) \nonumber 
  % \\
  % &\leq \mathbb{P}(x \in E_v,x^+\in Q )  \log(\ell \|Q\setminus Z \| + 2) +\|E_v\|\log(\ell \|Q\setminus Z \|+1) \nonumber 
  % \\
  %  &\leq 2\mathbb{E}_{x,x^+, \{x^-_i\}_{\ell}}  \chi \{x \in E_v, n_{3,v}>0\}\frac{n_{3,v}}{n_5+1} \nonumber \\
  % &\quad +2\mathbb{E}_{x,x^+, \{x^-_i\}_{\ell}}  \chi \{x \in E_v,n_{3,v}>0,x^+\in Q\} \frac{1}{n_5+1}\nonumber 
\end{align}
where \eqref{tm} follows since if $x\in E_v$, then $x^+\in Q \iff g'(x)^\top g'(x^+)=  g(x)^\top g(x^+) + 2$ and by the submodularity of $\log()$, and  \eqref{gg} follows by Jensen's Inequality and upper bounding $\chi\{n_{3,v}>0\}\leq 1$.

\end{enumerate}

Now we combine all six cases and sum over $v\in \mathcal{H}_d$.  We obtain
\begin{align}
&\mathcal{L}^-(g') -\mathcal{L}^-(g)\nonumber \\
&\le \sum_{v \in \mathcal{H}_d} \bigg(\mathbb{P}(x \in Q_v\setminus Z \cap x^+ \in E)\nonumber \\
&\quad\quad\quad\quad\quad\quad\left( 2 \beta (1 - \|Q_v\setminus Z\|_\circ - \|E\|_\circ)^\ell + \log(2)  \min\big(1, \tfrac{\ell \|Q_v\setminus Z\|_\circ}{1-\|E\|_\circ} \big) 
(1- \|E\|_\circ)^\ell\right)\nonumber\\
&\quad + 2 \beta   \|Q_v\setminus Z\|_\circ (1 - \|Q_v \setminus Z\|_\circ)^\ell \min\big(1, \tfrac{\ell \|E\|_\circ}{1- \|Q_v\setminus Z\|_\circ}\big) \nonumber \\
&\quad + 2 \|E\|_\circ +   \mathbb{P}  \left[x \in Q_v\setminus Z, x^+\in E\right] \nonumber \\
&\quad + \mathbb{P}(x \in E_v \cap x^+ \in Q)\big(2\beta (1 - \|E_v\|_\circ - \|Q\setminus Z\|_\circ)^\ell  \nonumber \\
&\quad \quad\quad\quad\quad \quad \quad\quad\quad \quad\quad\quad\quad \quad + \log(2)\min\left( 1, \tfrac{\ell \|E_v\|_\circ}{1 - \|Q\|_\circ}\right)  (1 - \|Q\setminus Z\|_\circ)^\ell\big)\nonumber \\
&\quad + 2 \beta \|E_v\|_\circ (1 - \|E_v\|_\circ)^\ell \min \left(1,\tfrac{\ell \|Q\|_\circ}{1 - \|E_v\|_\circ}\right)\nonumber \\
&\quad + \mathbb{P}(x \in E_v,x^+\in Q )  \log(\ell \|Q\setminus Z\|_\circ + 2) +\|E_v\|_\circ\log(\ell \|Q\setminus Z\|_\circ+1) \bigg) \\
&\leq  \|R\|\log(\ell + 2) + \|E\|_\circ (2^{d+1}+ \log(\ell +1 )  ) \nonumber \\
&\quad +  \sum_{v \in \mathcal{H}_d} \bigg(\mathbb{P}(x \in Q_v \cap x^+ \in E)\left( 2 \beta (1 - \|Q_v\|_\circ - \|E\|_\circ)^\ell + \log(2)  (1- \|E\|_\circ)^\ell\right)\nonumber\\
&\quad +  2 \beta \ell \|E\|_\circ \|Q_v\setminus Z\|_\circ (1 - \|Q_v\setminus Z\|_\circ)^{\ell-1}  \nonumber \\
&\quad + \mathbb{P}(x \in E_v \cap x^+ \in Q)\left(2\beta (1 - \|E_v\|_\circ - \|Q\|_\circ)^\ell  + \log(2) (1 - \|Q\|_\circ)^\ell\right)\nonumber \\
&\quad + 2 \beta \|E_v\|_\circ (1 - \|E_v\|_\circ)^\ell \bigg) \label{eee}\\
&\leq  \|R\|(\log(\ell + 2)+ \log(2) ) + \|E\|_\circ (2^{d+1}+ \log(\ell +1 ) ) \nonumber \\
&\quad + 2 \beta \sum_{v \in \mathcal{H}_d} \bigg(\mathbb{P}(x \in Q_v \cap x^+ \in E) (\|Q\|_\circ - \|Q_v\|_\circ)^\ell + \mathbb{P}(x \in E_v \cap x^+ \in Q)(\|E\|_\circ - \|E_v\|_\circ)^\ell \nonumber\\
&\quad +   \ell \|E\|_\circ \|Q_v\setminus Z\|_\circ (1 - \|Q_v\setminus Z\|_\circ)^{\ell-1} +  \|E_v\|_\circ (1 - \|E_v\|_\circ)^\ell \bigg) \nonumber \\
&\leq  \|R\|(\log(\ell + 2)+ \log(2) ) + \|E\|_\circ (2^{d+1}+ \log(\ell +1 ) ) + 2^{-(\ell-1)} \beta \|R\| \nonumber \\
&\quad + 2 \beta \sum_{v \in \mathcal{H}_d} \bigg(\mathbb{P}(x \in Q_v \cap x^+ \in E) (\|Q\|_\circ - \|Q_v\|_\circ)^\ell  \nonumber\\
&\quad +   \ell \|E\|_\circ \|Q_v\setminus Z\|_\circ (1 - \|Q_v\setminus Z\|_\circ)^{\ell-1} +  \|E_v\|_\circ (1 - \|E_v\|_\circ)^\ell \bigg) \label{eeee}
% &\leq  \|R\|(\log(\ell + 2)+ \log(2) ) + \|E\| (2^{d+1}+ \log(\ell +1 ) ) \nonumber \\
% &\quad + 2 \beta \sum_{v \in \mathcal{H}_d} \bigg( \|Q\|^{\ell+1}\frac{\mathbb{P}(x \in Q_v \cap x^+ \in B)}{\|Q\|} \left(1 - \frac{\mathbb{P}(x \in Q_v \cap x^+ \in B)}{\|Q\|}\right)^\ell\nonumber \\
% &\quad \quad + \|E\|^{\ell+1} \frac{\mathbb{P}(x \in E_v \cap x^+ \in Q)}{\|E\|}\left(1 - \frac{\mathbb{P}(x \in E_v \cap x^+ \in Q)}{\|E\|}\right)^\ell \nonumber\\
% &\quad \quad +   \ell \|E\| \|Q_v\|_\circ (1 - \|Q_v\|)^{\ell-1} +  \|E_v\|_\circ (1 - \|E_v\|)^\ell \bigg) \label{eeee}
% &\leq \|R\|(\log(\ell \|Q\setminus Z \| + 1)+1) + \|E\| (2^{d+1}+ \log(\ell \|Q\setminus Z \|)  ) \nonumber \\
% &\quad + 2 \beta \sum_{v \in \mathcal{H}_d} \bigg(\mathbb{P}(x \in Q_v \cap x^+ \in B) (1 - \|Q_v\| - \|E\|)^\ell+ \mathbb{P}(x \in E_v \cap x^+ \in Q) (1 - \|E_v\| - \|Q\|_\circ)^\ell \nonumber\\
% &\quad +  \|Q_v\|_\circ (1 - \|Q_v\|)^\ell \min\left(1, \frac{\ell \|E\|}{1- \|Q_v\|}\right) +  \|E_v\|_\circ (1 - \|E_v\|)^\ell \min(1,\ell \|Q\|/(1 - \|E_v\|))\bigg) \nonumber \\
% &\leq \|R\|(\log(\ell \|Q\setminus Z \| + 1)+1) + \|E\| (2^{d+1}+ \log(\ell \|Q\setminus Z \|)  ) \nonumber \\
% &\quad + 2 \beta \sum_{v \in \mathcal{H}_d} \bigg(\mathbb{P}(x \in Q_v \cap x^+ \in B) (1 - \|Q_v\| - \|E\|)^\ell+ \mathbb{P}(x \in E_v \cap x^+ \in Q) (1 - \|E_v\| - \|Q\|_\circ)^\ell \nonumber\\
% &\quad +  \ell {\|E\|}\|Q_v\|_\circ (1 - \|Q_v\|)^{\ell-1} +  \|E_v\|_\circ (1 - \|E_v\|)^\ell \bigg) \nonumber \\
\end{align}
where \eqref{eee} follows since $\min(x,y)\leq x$, $\sum_{v \in \mathcal{H}_d} \|E_v\|_\circ = \|E\|_\circ$, and  $\|R\| = \sum_{v\in \mathcal{H}_d} \mathbb{P}(x \in Q_v \cap x^+ \in E_v) +  \mathbb{P}(x \in E_v \cap x^+ \in Q) $, and \eqref{eeee} follows since $\|E\|_\circ\leq \frac{1}{2}$ by construction of $g'$ (since for $f_1'=f_1^{(\mathbf{c},\sigma)}$, $\sigma$ is chosen such that the induced $\|E\|_\circ$ cannot be larger than $\tfrac{1}{2}$).  It remains to bound the three terms in the sum in \eqref{eeee}. 

To do so, we first define $\Tilde{E}_v\coloneqq \{ x^+\in B: \tilde{\mathcal{A}}^{-1}(x^+) \in Q_v \}$ as the set of partial augmentation sets that are in $B$, corresponding to sets whose natural image is in $Q_v$ (where $\tilde{\mathcal{A}}^{-1}(x^+)$ is the natural image from which the augmented image $x^+$ was generated, i.e. $\tilde{\mathcal{A}}^{-1}(x^+)=x \iff \mathcal{A}(x)=x^+$ for some $\mathcal{A}\in\Lambda$). We have 
\begin{align}
    \sum_{v \in \mathcal{H}_d} \mathbb{P}(x \in Q_v \cap x^+ \in E) (\|Q\|_\circ - \|Q_v\|_\circ)^\ell &= \sum_{v \in \mathcal{H}_d} \|\Tilde{E}_v\| (\|Q\|_\circ - \|Q_v\|_\circ)^\ell \label{xx} \\
    &\leq  \sum_{v \in \mathcal{H}_d} \|\Tilde{E}_v\| (1 - \|Q_v\|_\circ)^\ell \nonumber \\
    &\leq  \sum_{v \in \mathcal{H}_d} \|\Tilde{E}_v\| \left(1 - {\|\Tilde{E}_v\|}\right)^\ell  \label{ba} %\\
    % &= \gamma  \sum_{v \in \mathcal{H}_d} \frac{\|\Tilde{E}_v\|}{\gamma} \left(1 - \frac{\|\Tilde{E}_v\|}{\gamma}\right)^\ell  %\nonumber \\
    % &= 
\end{align}
where \eqref{xx} and \eqref{ba}  follow since all augmentation sets are equal size.
% , and  follows since $\|\Tilde{E}_v\|\leq \gamma \|Q_v\|$  
% $\|\Tilde{E}_v\|\leq \gamma \|Q_v\|$ by Assumption \ref{assump-a} 
% (i.e. no classifier can separate a natural image from more than $\gamma$\% of its augmentations, where $\gamma<1$.... no, this is regularity condition on the size of the augmentation sets. If R is large but the connecting natural images are rare, this is an issue. just assume the augmentation sets are all equal size.). 

Note $\sum_v \|\Tilde{E}_v\|=\|R\|$, and $h(x_v)\coloneqq x_v(1-x_v)^\ell$ is maximized on $x_v\in[0,1]$ at $x_v=\frac{1}{\ell+1}$. 
% If $\frac{\|R\|}{2^d} \leq \frac{1}{\ell+1}$, then by the concavity of $h(x_v)\coloneqq x_v(1-x_v)^\ell$ in $x_v\in [0,1]$, the sum is maximized by setting  ${\|\Tilde{E}_v\|} = \frac{\|R\|}{2^d} $ for all $v$. 
Thus, $\sum_v \|\Tilde{E}_v\|\left(1 - {\|\Tilde{E}_v\|}\right)^\ell$ is upper bounded by setting ${\|\Tilde{E}_v\|} = \frac{1}{\ell+1} $ for all $v$.
% Recall that $\ell \geq \frac{c c_1 c_2}{\epsilon} d 2^d$ and $\beta \geq \frac{c}{\delta} \log\left(\frac{c_1 c_2}{\epsilon}\right)$.
% Next, we argue that $\|R\|\ll \frac{1}{\ell+1}$.
Thus we obtain
\begin{align}
    \sum_{v \in \mathcal{H}_d} \mathbb{P}(x \in Q_v \cap x^+ \in B) (\|Q\|_\circ - \|Q_v\|_\circ)^\ell 
    &\leq %\chi\{ {\|R\|} \leq \frac{2^d}{\ell+1} \} \sum_{v \in \mathcal{H}_d} \frac{\|R\|}{\  2^d} \left(1 -  \frac{\|R\| }{2^d} \right)^\ell  \nonumber \\
    % &\quad +\chi\{ {\|R\|} > \frac{ 2^d}{\ell+1} \} 
    \sum_{v \in \mathcal{H}_d} \frac{1}{\ell+1} \left(1 -  \frac{1}{\ell+1} \right)^\ell  \nonumber \\
    &= 
    % \chi\{ {\|R\|} \leq \frac{ 2^d}{\ell+1} \} {\|R\| } \left(1 -  \frac{\|R\| }{ 2^d} \right)^\ell  \nonumber \\
    % &\quad +\chi\{ {\|R\|} > \frac{ 2^d}{\ell+1} \}
    \frac{2^d}{\ell+1} \left(1 -  \frac{1}{\ell+1} \right)^\ell  \nonumber \\
    &\leq 
    % \chi\{ {\|R\|} \leq \frac{ 2^d}{\ell+1} \} \|R\| e^{-\frac{\|R\|\ell }{ 2^d}}  \nonumber \\
    % &\quad +
    % \chi\{ {\|R\|} > \frac{2^d}{\ell+1} \}
   \frac{\epsilon}{c  d} e^{-\frac{\ell}{\ell+1}} \nonumber 
\end{align}
where we have used $\ell \geq \frac{c}{\epsilon} d 2^d$ for a constant $c$ in the last line.
Next we consider $\sum_{v \in \mathcal{H}_d}\|E_v\|_\circ (1 - \|E_v\|_\circ)^\ell$, and use a tighter method of bounding this sum than above. Note that $\sum_{v \in \mathcal{H}_d} \|E_v\|_\circ = \|E\|_\circ$. 
If $\frac{\|E\|_\circ}{2^d} \leq \frac{1}{\ell+1}$, then by the concavity of $h(x_v)\coloneqq x_v(1-x_v)^\ell$ on the interval $x_v\in [0,\frac{1}{\ell+1}]$, the sum is maximized by setting
 $\|E_v\|_\circ = \frac{\|E\|_\circ}{2^d}$ for all $v$. Otherwise,  the sum is upper bounded  by setting $\|E_v\|_\circ = \frac{1}{\ell+1}$ for all $v$.  
 % Note that $\sum_{v \in \mathcal{H}_d} \|E_v\|_\circ = \|E\|_\circ$.  
 Thus we have 
\begin{align}
    \sum_{v \in \mathcal{H}_d}\|E_v\|_\circ (1 - \|E_v\|_\circ)^\ell &\leq \chi\left\{ \|E\|_\circ\leq \frac{2^d}{\ell+1}\right\}  {\|E\|_\circ}\left(1 - \frac{\|E\|_\circ}{2^d}\right)^\ell\nonumber \\
    &\quad + \chi\left\{ \|E\|_\circ> \frac{2^d}{\ell+1}\right\} \frac{2^d}{\ell+1} \left(1 - \frac{1}{\ell+1}\right)^\ell\nonumber \\
    &\leq \chi\left\{ \|E\|_\circ\leq \frac{2^d}{\ell+1}\right\}  {\|E\|_\circ}e^{-\frac{\|E\|_\circ\ell}{2^d}}\nonumber \\
    &\quad + \chi\left\{ \|E\|_\circ> \frac{2^d}{\ell+1}\right\} \|E\|_\circ e^{-\frac{\ell}{\ell+1}}\label{eq:bb}
\end{align}
% However, note that $\|E\|_\circ \leq \frac{1}{\delta}\|R\|$, and 
Finally, note that $ \mathcal{D}_g(v):= \mathbb{P}_{x\sim D_\circ}[g(x)=v] = \|Q_v\|_\circ + \|E_v\|_\circ$. We have that for all $v$, $ Q_v \setminus Z \neq \emptyset \implies Q_v \setminus Z = Q_v, E_v=\emptyset$ by Claims \ref{clm:4} and \ref{clm:6}.
Thus for all $v: Q_v \setminus Z \neq \emptyset$, $\mathcal{D}_g(v) = \|Q_v\|_\circ = \|Q_v\setminus Z\|_\circ$. This allows us to use that $g$ is near uniform, i.e. $\mathcal{D}_g(v) > \frac{1}{ c_1 d 2^d}$ or $\mathcal{D}_g(v) \leq \frac{ \epsilon}{ c_2  d 2^{2d}}$ for all $v\in \mathcal{H}_d$ for some constants $c_1,c_2$. We have $\ell \geq \frac{c }{\epsilon} d 2^d$ and choose $c_1<c$, such that $ \frac{1}{ c_1 d 2^d}>\frac{1}{\ell+1}$. Since $h(x_v)\coloneqq x_v(1-x_v)^\ell$ is a decreasing function of $x_v$ for $x_v\geq \frac{1}{\ell+1}$,  we can bound the last sum in \eqref{eeee} as  %$\|Q_v\|$,   
\begin{align}
     \ell \sum_{v \in \mathcal{H}_d}\|Q_v\setminus Z\|_\circ (1 - \|Q_v\setminus Z\|_\circ)^{\ell-1} &\leq \ell 2^d \max\left( \frac{ \epsilon}{ c_2 d 2^{2d}} , \frac{1}{c_1 d 2^d}(1-\frac{1}{c_1 d 2^d})^\ell \right) \nonumber \\
     &\leq \ell \max\left( \frac{\epsilon}{ c_2 d 2^{d}} , \frac{1}{c_1 d} e^{-\frac{\ell }{c_1 d 2^d}}\right) \nonumber \\
     &\leq  \max\left( {\frac{c}{c_2}  }, {\frac{c}{c_1\epsilon}  } e^{-(\frac{c  }{c_1\epsilon}-1)}\right)  \label{bnb}
     % &\leq  \ell \max\left( \frac{c_2}{  d 2^{d}} , {c_1}e^{-C c_1 d/\epsilon} \right) \nonumber \\
     % &\leq  \max\left( C c_2, \frac{C{c_1} d}{\epsilon} e^{-(C c_1-1) d} \right). \label{eq:unif}
\end{align}
Before we combining these bounds with \eqref{eeee}, we first show that  $\|R\|= \Omega( \epsilon)$.
\begin{claim} \label{clm:7}
    Let $\epsilon \leq  \min_{j\in [d]} \min_{\mathbf{c}' \in \Sigma_{f_j}} \mathbb{P}_{x,x'\sim D}[x, x' \in \Gamma_{\mathbf{c}'},f_j(x)\neq f_j(x')]$ as defined in the statement of Theorem \ref{thm:agnostic}. Then $\epsilon \leq \frac{6}{\delta } \|R\|$.
\end{claim}
\begin{proof}
From our choice of $f_1'$, we have
\begin{align}
   \epsilon &\leq  \min_{j\in [d]} \min_{\mathbf{c}' \in \Sigma_{f_j}} \mathbb{P}_{x,x'\sim D}[x, x' \in \Gamma_{\mathbf{c}'},f_j(x)\neq f_j(x')]\nonumber \\
   &\leq 2\mathbb{P}_{x,x'\sim D}[x\in  \Gamma_{\mathbf{c}}\setminus E, x' \in E]\nonumber \\
   &= 2\mathbb{P}_{x\sim D}[x\in  \Gamma_{\mathbf{c}}\setminus E]  \mathbb{P}_{x\sim D}[x \in E]\nonumber \\
   &\leq 2\mathbb{P}_{x\sim D}[x \in E]\nonumber \\
   &\leq 2(\mathbb{P}_{x\sim D_\circ}[x \in E]+\mathbb{P}_{x^+\sim D\setminus D_\circ}[x^+ \in E]) \nonumber \\
   &= 2 (\|E\|_\circ + \|E\|). \label{jgj} % \nonumber \\
   % &\leq \|E\|.
\end{align}
Note that $\|E\|_\circ\leq \frac{1}{\delta} \|R\|$ by Assumption \ref{assump-b}.
Observe that
\begin{align}
    \|E\| &= \mathbb{P}_{x^+\sim D\setminus D_{\circ}}[ x^+ \in E ] \nonumber \\
    &= \mathbb{P}_{x^+\sim D\setminus D_{\circ}}[ x^+ \in E, \mathcal{A}^{-1}(x^+) \in E  ] + \mathbb{P}_{x^+\sim D\setminus D_{\circ}}[ x^+ \in E, \mathcal{A}^{-1}(x^+) \notin E  ]  \nonumber \\
    &\leq \mathbb{P}_{x\sim D_{\circ}, x^+\sim A(x)}[  x \in E, x^+ \in E  ] + \|R\| \nonumber \\
    &\leq \|E\|_\circ + \|R\|\label{dvd}
\end{align}
Note that $\|E\|_\circ\leq \frac{1}{\delta} \|R\|$ by Assumption \ref{assump-b}. Combining this with \eqref{jgj} and \eqref{dvd} yields
\begin{align}
      \epsilon &\leq 2 (2\|E\|_\circ + \|R\|) \leq \frac{6}{\delta } \|R\| \nonumber
\end{align}
\end{proof}

Finally, using Claim \ref{clm:7} with $\beta \geq \frac{c'}{\delta} \log\left(\frac{c}{\epsilon}\right)2^d$,
we obtain from \eqref{eeee} that
\begin{align}
    &\mathcal{L}^-(g') -\mathcal{L}^-(g)\nonumber \\
    &\leq \|R\|(\log(\ell + 2)+ \log(2) ) + \frac{\|R\|}{\delta} (2^{d+1}+ \log(\ell +1 ) ) + 2^{-(\ell-1)} \beta \|R\| \nonumber \\
    &\quad  + \frac{2 \beta \epsilon}{c d} e^{-\frac{\ell}{\ell+1}}  +  2 \beta  \max\left( {\frac{c}{c_2}  }, {\frac{c}{c_1} } e^{-(\frac{c  }{c_1\epsilon}-1)}\right) \nonumber \\
    % \frac{2\beta}{\delta} \|R\|\max\left( C c_2, C{c_1} d e^{-(C c_1-1) d} \right)   \nonumber \\
    % + 2\beta \chi\left\{ {\|R\|} \leq \frac{ 2^d}{\ell+1} \right\} \|R\| e^{-\frac{\|R\|\ell }{2^d}}   +2 \beta \chi\left\{ {\|R\|} > \frac{ 2^d}{\ell+1} \right\}\|R\| e^{-\frac{\ell}{\ell+1}} \nonumber \\
    &\quad  + 2 \beta \chi\left\{ \|E\|_\circ\leq \frac{2^d}{\ell+1}\right\}  {\|E\|_\circ}e^{-\frac{\|E\|_\circ\ell}{2^d}} + 2 \beta \chi\left\{ \|E\|_\circ> \frac{2^d}{\ell+1}\right\} \|E\|_\circ e^{-\frac{\ell}{\ell+1}}\nonumber \\
    &\leq \beta \|R\| \bigg(\frac{\log(2c d2^{d}/\epsilon + 4) + \frac{1}{\delta} \log(c  d2^d/\epsilon + 1)}{c' \log(\frac{c}{\epsilon}) 2^d/\delta} +\frac{2 }{c' \log(\frac{c}{\epsilon}) 2^d} + \frac{2}{2^{c d 2^d/\epsilon}} \nonumber \\
    &\quad \quad \quad \quad + \frac{12}{c d \delta} e^{-\frac{\ell}{\ell+1}} + 2 \max\left( {\frac{c}{c_2}  }, {\frac{c}{c_1} } e^{-(\frac{c  }{c_1\epsilon}-1)}\right) \bigg) \nonumber \\
    &\quad + 2 \beta \|E\|_\circ \left(\chi\left\{ \|E\|_\circ\leq \frac{\epsilon}{cd}\right\}  e^{-\frac{\|E\|_\circ\ell}{2^d}} + \chi\left\{ \|E\|_\circ> \frac{2^d}{\ell+1}\right\}  e^{-\frac{\ell}{\ell+1}} \right)\nonumber \\
    &\leq \beta \|R\| \bigg(\frac{\log(4c /\epsilon + 4) + \frac{1}{\delta} \log(2c  /\epsilon + 1)}{2c' \log(\frac{c}{\epsilon}) /\delta} +\frac{1}{c' \log(\frac{c}{\epsilon}) } + \frac{1}{2^{c/\epsilon}} \nonumber \\
    &\quad \quad \quad \quad + \frac{12}{c d \delta} e^{-\frac{\ell}{\ell+1}} + 2 \max\left( {\frac{c}{c_2}  }, {\frac{c}{c_1}  } e^{-(\frac{c }{c_1\epsilon}-1)}\right) \bigg) \nonumber \\
    &\quad +  \beta \|R\| \left( \frac{12}{cd \delta} \chi\left\{ \|E\|_\circ\leq \frac{6 \|R\|}{cd\delta}\right\}   +  \frac{1}{\delta} \chi\left\{ \|E\|_\circ> \frac{2^d}{\ell+1}\right\}  e^{-\frac{\ell}{\ell+1}} \right)\nonumber \\
    &\leq  2\beta  \|R\|\left(\frac{2c}{c_2}+ \frac{1}{100}\right) + 2\beta \|R\| \frac{1.01}{\delta e} 
    % \max\left(\frac{25}{cd},\frac{1.01}{e}\right)  
    \label{61}  \\
    % &\leq   \beta \|R\| \left(\frac{\log(2Cd2^{d}/\epsilon + 4) + \log(Cd2^d/\epsilon + 1)}{c_3 \delta 2^d} +\frac{2 }{c_3\delta} + \frac{2}{2^{C d 2^d/\epsilon}}  + \frac{2}{\delta} \max\left( {C c_2}/\epsilon, \frac{C{c_1}}{\epsilon} d e^{-(C c_1/\epsilon-1) d} \right)  \right) \nonumber \\
    % &\quad + 2\beta \|R\| e^{-\frac{\ell}{\ell+1}}\nonumber \\
    % &\quad + 2 \beta \frac{\|R\|}{Cd}\chi\left\{ \|E\|_\circ\leq \frac{\epsilon}{Cd+\epsilon/2^d}\right\}  e^{-\frac{\|E\|_\circ\ell}{2^d}} + \frac{2\|R\|}{\delta} \chi\left\{ \|E\|_\circ> \frac{2^d}{\ell+1}\right\}  e^{-\frac{\ell}{\ell+1}} \nonumber \\
    % &\leq  \beta \|R\| \left(\frac{\log(2Cd2^{d}/\epsilon + 4) + \frac{1}{\delta} \log(Cd2^d/\epsilon + 1)}{c_3 2^d} +\frac{2}{c_3\delta} + \frac{2}{2^{C d 2^d/\epsilon}}  + \frac{2}{\delta} \max\left( {C c_2}/\epsilon, \frac{C{c_1}}{\epsilon} d e^{-(C c_1/\epsilon-1) d} \right)  \right) \nonumber \\
    % &\quad + 2\beta \|R\| \frac{1.01}{e}\left(1+\frac{1}{\delta}\right)\nonumber \\
    &< 2 \beta \|R\| = \mathcal{L}^+(g) - \mathcal{L}^+(g')
\end{align}
where \eqref{61} follows for a sufficiently large constants $c'$ and $c>c_1$. Use $\delta\geq 0.4$ and set $c'=c$, $c_1=\frac{c}{10}$, $c_2= 100c$ and make $c$ sufficiently large to obtain
\begin{align}
    &\mathcal{L}^-(g') -\mathcal{L}^-(g)< 2 \beta \|R\| = \mathcal{L}^+(g) - \mathcal{L}^+(g'). \nonumber
\end{align}
\end{proof}

\subsection{Remark: Modified version of Theorem \ref{thm:agnostic}} \label{app:weighted}
{
Theorem \ref{thm:agnostic} shows that if a minimizer of the InfoNCE loss is close to uniform, then it must be clean. Here we show that the InfoNCE loss can be interpreted as the Lagrangian which, under appropriate choice of hyperparameters, we can formally show to be minimized only by clean representations.

\textbf{Weighted InfoNCE loss. } 
% Minimizing the InfoNCE loss in \eqref{eq:InfoNCE} can also be viewed as minimizing the Lagrangian of the following constrained minimization problem, where we try 
Consider the following constrained optimization problem that tries to maximize uniformity while preserving alignment:
\begin{align}\label{eq:constrained}
    \min_{g \in \mathcal{G}} \; &\mathbb{E}_{x,x^+, \{x^-_i\}_{\ell}} \bigg[\log\bigg({ e^{\beta g(x)^\top g(x^+)}\!+\!\sum_{i=1}^\ell e^{\beta g(x)^\top g(x^-_{i})}  } \bigg)\bigg] \nonumber \\
    &\text{s.t.}  \quad \mathbb{E}_{x,x^+}[ g(x)^\top g(x^+)] = d  %\quad \forall x, x^+ 
\end{align}
The unconstrained penalized version of this problem is 
\begin{align}%_{\text{SimCLR}}
    &    \min_{g \in \mathcal{G}} -\lambda \underset{x,x^+}{\mathbb{E}} \left[g(x)^\top g(x^+) \right] \;  + \; \underset{x,x^+, \{x^-_i\}_{\ell}}{\mathbb{E}} \left[\log\bigg({ e^{\beta g(x)^\top g(x^+)}\!+\!\sum_{i=1}^\ell e^{\beta g(x)^\top g(x^-_{i})}  } \bigg)\right]
\end{align}
Note that the above objective with penalty coefficient $\lambda=\beta$ is equal to the InfoNCE objective. This formulation motivates alternate choices of $\lambda$ depending on how strictly we would like to enforce alignment. We refer to the loss 
\begin{align}\label{eq:weightedInfoNCE} %_{\text{SimCLR}}
    &\uL_\rho({g}) = -\rho\beta \underset{x,x^+}{\mathbb{E}} \left[g(x)^\top g(x^+) \right] \;  + \; \underset{x,x^+, \{x^-_i\}_{\ell}}{\mathbb{E}} \left[\log\bigg({ e^{\beta g(x)^\top g(x^+)}\!+\!\sum_{i=1}^\ell e^{\beta g(x)^\top g(x^-_{i})}  } \bigg)\right]
\end{align}
as the Weighted InfoNCE loss which is equivalent to the  penalized version of \eqref{eq:constrained} with $\lambda = \rho\beta$. We note that this loss is similar to the generalized InfoNCE loss proposed by \cite{chen2021intriguing}.
% who further observed empirically that up-weighting the positive loss (i.e. setting $\rho > 1$) can improve downstream performance.
}

\begin{corollary} Consider the same setting as Theorem \ref{thm:agnostic} but with any $\delta >0$ and with the Weighted InfoNCE loss with $\rho \geq  {2^{d+1}} + \frac{1}{\delta}$. Then for sufficiently large $\ell$ and $\beta$, all solutions of the Weighted InfoNCE objective are cluster-preserving.
\end{corollary}

\begin{proof}
    The result follows from the analysis in the proof of Theorem \ref{thm:agnostic}. The analysis for the difference in positive terms is identical except that they are scaled by $\rho$, so we have
    \begin{align}
        \mathcal{L}^+(g) - \mathcal{L}^+(g') = 2 \rho \beta \|R\|
    \end{align}
    For the difference in negative terms, the analysis is again identical except that we can no longer use that $g$ is close to uniform. The only place we have used this is to bound $\ell \sum_{v\in \mathcal{H}_d} \|Q_v \setminus Z \|_\circ (1 - \|Q_v \setminus Z \|_\circ)^\ell$  in \eqref{bnb}. Here, we bound this term using the fact that $h(x)\coloneqq x(1-x)^\ell$ is maximized on the interval $x\in [0,1]$ at $x = \frac{1}{\ell+1}$.
    \begin{align}
        \ell \sum_{v\in \mathcal{H}_d} \|Q_v \setminus Z \|_\circ (1 - \|Q_v \setminus Z \|_\circ)^\ell &\leq \ell \sum_{v\in \mathcal{H}_d} \frac{1}{\ell+1} \left(1 - \frac{1}{\ell+1}\right)^\ell \nonumber \\
        &=  2^d \left(\frac{\ell}{\ell+1}\right)^{\ell+1} \nonumber \\
        &\leq \frac{2^d}{e}
    \end{align}
    Replacing this bound and executing the same analysis as in \eqref{61} yields
    \begin{align}
    &\mathcal{L}^-(g') - \mathcal{L}^-(g)\nonumber \\
    &\leq 2\beta  \|R\|\left({2^{d}}+ \frac{1}{100}\right) + 2\beta \|R\| \frac{1.01}{\delta e} \nonumber \\
    &< 2 \rho\beta \|R\|
    \end{align}
    completing the proof.
\end{proof}

% \section{Proof of Theorem \ref{thm:clean}}

\section{Proofs of Downstream Guarantees}

% \section{Proof of Theorem \ref{thm:negative}}

% \subsection{Proof of Theorem \ref{thm:downstream}}
% \begin{proof}
% % Let $x\in S$. 
% By Theorems \ref{thm:clean} and \ref{thm:uniform}, $g$ is a bijection from clusters to vertices of $\mathcal{H}_d$.
% Pick any $x \in S_1$ and denote $v_{g,S_1} \coloneqq g(x)$.  

% \end{proof}

\subsection{Proof of Theorem \ref{thm:downstream1}}
\begin{proof}
    From Theorem \ref{thm:uniformandclean}, $g^*$ maps each cluster to a unique vertex on the $d$-dimensional hypercube. 
    % Note that ${U}g(x)$ has value $d$ in exactly one element, corresponding to the index of the cluster to which $x$ belongs to, and all other elements are integers strictly less than $d$. Thus $\tilde{g}(x)$ is a one-hot encoding of the cluster containing $x$. 
    % If $N \leq m/2$, set $\Tilde{N}=N$. 
    Let  $\psi_f \coloneqq \{\Gamma_{\mathbf{c}}: f(x)=1\; \forall x \in \Gamma_{\mathbf{c}}\}$ be the set of  clusters which $f$ labels 1, and let $N = |\psi_f|$. Similarly let $\psi_f^c \coloneqq \{\Gamma_{\mathbf{c}}: f(x)=-1\; \forall x \in \Gamma_{\mathbf{c}}\}$. 
    For all $j\in [N]$, let the $j$-th row of $\mathbf{W} \in \mathbb{R}^{m \times d}$ equal the vertex corresponding to the mapping of the $j$-th cluster in $\psi_f$ by $g^*$. For all $j \in [m-N]$, let the $j+N$-th row of $\mathbf{W}$ equal the vertex corresponding to the mapping of the $j$-th cluster in $\psi_f^c$ by $g^*$
    Then for any $x$ such that $f(x)=1$, $\mathbf{W} g^*(x)$ has exactly one element with value $d$ among the first $N$ elements, and all other elements are at most $d-2$. On the other hand, for any $x$ such that $f(x)=-1$, $\mathbf{W} g^*(x)$ has exactly one element with value $d$ among the last $m-N$ elements, and all other elements are at most $d-2$. Set $b = (d-2)\times 1_{m}$, that is, $d-2$ times the $\Tilde{N}$-dimensional vector of ones, and $a= [1_{{N}}^\top, -1_{m-{N}}^\top ]^\top$, that is, the $m$-dimensional  vector whose first $m$ elements are 1  and whose last $m-N$ elements are $-1$. Then $a^\top \text{ReLU}(\mathbf{W} g(x) - b) = f(x)$ for all $x$.
\end{proof}

 % Otherwise (if $N > m/2$) set $\Tilde{N}=m-N$, and let $\psi_f^c \coloneqq \{\Gamma_{\mathbf{c}}: f(x)=-1\; \forall x \in \Gamma_{\mathbf{c}}\}$. Analogously, let the $j$-th row of $\mathbb{W}$ equal the vertex corresponding the mapping of the $j$-th cluster in $\psi_f^c$ by $g^*$. Then for any $x$ such that $f(x)=-1$, $\mathbf{W} g^*(x)$ has exactly one element with value $d$, and all other elements are at most $d-2$. For any $x$ such that $f(x)=1$, all elements of  $\mathbf{W} g^*(x)$ are at most $d-2$. Set $b_1 = (d-2)\times 1_{\Tilde{N}}$ and $a= -1_{\Tilde{N}}$ and $b_2 = -1$. Then $a^\top \text{ReLU}(\mathbf{W} g(x) - b_1) - b_2 = f(x)$ for all $x$.
    % For any downstream task $f \in \mathcal{T}_N$, let $\mathcal{S}_c$ denote the set of clusters that belong to class  $c$. For all $c\in [r]$, construct the $c$-th row of $W$ as
    % \begin{align}
    %     W_{c,i} = \begin{cases} 1 & \text{if $i\in \mathcal{S}_c$ } \\
    %     0 & \text{otherwise} \\
    %     \end{cases}
    % \end{align}
    % where $W_{c,i}$ is the $i$-th element of the $c$-th row of $W_{c,i}$. Since $\tilde{g}(x)$ one-hot encodes the cluster containing $x$, $(W\tilde{g}(x))_c=1$ iff $x$ is in class $c$, otherwise $(W\tilde{g}(x))_c=0$. 

\subsection{Proof of Theorem \ref{thm:negative}}
\begin{proof}
Since we are in the realizable setting, ${D}$ has $m:=2^d$,  equal-size clusters, where $d>3$. 
% Partition these clusters into two sets $\Phi_1$ and $\Phi_2$, each consisting of 8 clusters. 
Moreover, since $\mathcal{F}_\star$ is arbitrarily powerful and the augmentation sets are disjoint, 
% the ``pseudo-clusters'' it induces are equivalent to augmentation sets, 
for every pair of augmentation sets $(A(x),A(\bar{x}))$, there exists an $f\in \mathcal{F}_\star$ such that $f(x^+)\neq f(\bar{x}^+) $ for all $x^+\in A(x)$ and $\bar{x}\in A(\bar{x})$, and $f$ does not intersect any other augmentation set. Further, $\mathcal{G}_\star$ can map augmentation sets to arbitrarily different vertices, even if these sets lie in the same cluster. In other words, there are clean representations in $\mathcal{G}_\star$ (meaning they are faithful to all augmentation sets) that split clusters by augmentation sets. 

Suppose the number of augmentation sets in each cluster is  $k \times m\times 2^d$ for some $k\in \mathbb{N}^+$, and all augmentation sets are of equal size $M$. Then there
exists a clean and uniform representation $g \in \mathcal{G}_\star$ such that for each vertex $v \in \mathcal{H}_d$, $kM$ of the images in the set $\{ x \in {D}_{\circ}: g(x) = v \}$ are in each cluster. In other words, for all $x  \in {D}_\circ$ and $x^+\sim A(x)$, $g(x) = g(x^+)$. Thus, 
we can apply  Theorem \ref{thm:uniformandclean} to obtain $g \in \arg\min_{g'\in\mathcal{G}_\star}\mathcal{L}(g')$ (note that in Theorem \ref{thm:uniformandclean}, cluster-preserving is equivalent to clean since we are optimizing over the restricted class $\mathcal{G}$, and the same proof can be applied exactly as is, with the word ``cluster-preserving'' replaced by ``clean'', to show that $g\in \argmin_{g'\in\mathcal{G}_\star}\mathcal{L}(g')$ if and only if $g$ is clean and uniform).

Since any head $\omega\in \mathcal{J}$ composed with $g$ must yield the same prediction for all images mapped to the same vertex on $\mathcal{H}_d$, and all vertices have the same number of images from each cluster mapped to them, the number of images with predicted label $1$ must be the same for all clusters,
 and likewise for $-1$. Thus, for any downstream binary classification task $h$ that  satisfies $h(x)=h(x')$ for all $x,x'\in \Gamma_{\mathbf{c}}$ for all $\mathbf{c}\in C$ and $\mathbb{P}_{x \sim D_\circ}[h(x)=1]=0.5$, any $\omega$ must have $\mathcal{L}_f(\omega \circ g)\geq 0.5$.
\end{proof}

\end{document}